\documentclass[11pt]{article} 

\usepackage{graphicx} 
\usepackage{wrapfig} 

\usepackage{mathpazo} 
\usepackage[T1]{fontenc} 
\usepackage{fullpage}
\usepackage{enumitem}
\usepackage[page]{appendix}
\usepackage{amssymb}
\usepackage{mathrsfs} 
\usepackage{natbib}
\usepackage{algorithm}
\usepackage{algorithmic}
\usepackage{hyperref}
\usepackage{comment}
\usepackage{amsthm}
\usepackage{mathtools}
\usepackage{epstopdf,xcolor}
\usepackage{bbm}
\usepackage[colorinlistoftodos,prependcaption]{todonotes}
\usepackage{dsfont}

\newtheorem{theorem}{Theorem}
\newtheorem{lemma}{Lemma}

\newtheorem{proposition}{Proposition}
\newtheorem{corollary}{Corollary}
\newtheorem{definition}{Definition}
\newtheorem{example}{Example}
\newtheorem{remark}{Remark}

\newtheorem{assumption}{Assumption}

\newcommand{\Pa}{\textit{Pa}}
\newcommand{\Ch}{\textit{Ch}}

\newcommand{\Size}{\textit{Size}}
\newcommand{\Counter}{\textsc{Counter}}
\newcommand{\MEC}{\textit{MEC}}
\newcommand{\I}{\mathcal{I}}
\newcommand{\E}{\mathcal{E}}
\newcommand{\tG}{\tilde{G}}
\newcommand{\tT}{\tilde{T}}

\makeatletter
\renewcommand\@biblabel[1]{\textbf{#1.}} 
\renewcommand{\@listI}{\itemsep=0pt} 

\renewcommand{\maketitle}{ 
\begin{center}
{\LARGE\@title} 
\end{center}

\vspace{50pt} 

{\large\@author} 

\vspace{40pt} 
}

\title{Interventional Experiment Design\\ for Causal Structure Learning}
       
\author{\noindent
\textsc{AmirEmad Ghassami}\\
\textit{Department of Electrical and Computer Engineering,\\ University of Illinois at Urbana-Champaign, Urbana, IL 61801, USA}\\
       \vspace{3mm}
\texttt{ghassam2@illinois.edu}\\
\textsc{Saber Salehkaleybar} \\
\textit{Department of Electrical Engineering,\\ Sharif University of Technology, Tehran, Iran}\\
       \vspace{3mm}  
\texttt{saleh@sharif.edu}\\  
\textsc{Negar Kiyavash}\\
\textit{College of Management of Technology, Management of Technology and Entrepreneurship Institute\\ \'Ecole Polytechnique F\'ed\'erale de Lausanne (EPFL), Switzerland}\\
\texttt{negar.kiyavash@epfl.ch} \\
       }

\begin{document}

\maketitle

\begin{abstract}
It is known that from purely observational data, a causal DAG is identifiable only up to its Markov equivalence class, and for many ground truth DAGs, the direction of a large portion of the edges will be remained unidentified. The golden standard for learning the causal DAG beyond Markov equivalence is to perform a sequence of interventions in the system and use the data gathered from the interventional distributions. We consider a setup in which given a budget $k$, we design $k$ interventions non-adaptively. We cast the problem of finding the best intervention target set as an optimization problem which aims to maximize the number of edges whose directions are identified due to the performed interventions. First, we consider the case that the underlying causal structure is a tree. For this case, we propose an efficient exact algorithm for the worst-case gain setup, as well as an approximate algorithm for the average gain setup. We then show that the proposed approach for the average gain setup can be extended to the case of general causal structures. In this case, besides the design of interventions, calculating the objective function is also challenging. We propose an efficient exact calculator as well as two estimators for this task. We evaluate the proposed methods using synthetic as well as real data.
\end{abstract}

\hspace*{3,6mm}\textit{Keywords:} Causal Structure Learning, Intervention Design, Directed Acyclic Graphs, Interventional Markov Equivalence. 

\vspace{30pt} 


\newpage

\section{Introduction}
\label{sec:intro}


One of the most prominent approaches for modeling and representing causal relationships among variables in a system is to use the framework of causal Bayesian networks, which consists of a directed acyclic graph (DAG), paired with a joint distribution over the variables of the system  \citep{spirtes2000causation, pearl2009causality}. In the DAG in this modeling, a directed edge from variable $X_1$ to variable $X_2$ indicates that $X_1$ is a direct cause of $X_2$. 

\begin{sloppy}

Under certain assumptions on the underlying data generating processes, such as considering linear models with non-Gaussian exogenous variables \citep{shimizu2006linear}, or assuming specific types of non-linearity on the causal modules \citep{hoyer2009nonlinear,zhang2008distinguishing}, in the population dataset, the causal DAG can be identified uniquely. 
However, such assumptions usually force a sort of additivity on the exogenous variables of the system, which in many applications may not be realistic.

\end{sloppy}

Without such extra assumptions the underlying causal DAG can be identified only up to its Markov equivalence class, which is the set of DAGs which represent the same set of conditional independencies among the variables. 
Hence, for many ground truth DAGs, the direction of a large portion of the edges will be remained unidentified. For instance, the three DAGs $X_1\rightarrow X_2\rightarrow X_3$, $X_1\leftarrow X_2\rightarrow X_3$ and $X_1\leftarrow X_2\leftarrow X_3$ indicate that $X_3$ is independent of $X_1$ conditioned on $X_2$, and they are Markov equivalent. 
In order to learn the structure beyond Markov equivalence, the golden standard is to assume that extra joint distributions generated from the perturbed causal system are available.


The main method for generating such extra joint distributions is to perform a set of interventions, each on a subset of the variables of the system, and subsequently collect data from the intervened system. This is the core idea in interventional causal structure learning.
An intervention on a variable $X$ varies the conditional distribution of $X$ given its direct causes. It can also completely make variable $X$ independent from its causes. The information obtained from an intervention depends on the type of the performed intervention, as well as the size of the intervention (i.e., the number of the target variables), and the location of the targets of the intervention in the underlying causal DAG. (We will discuss this in details in Section \ref{sec:prelim}.) An interventional experiment is comprised of a sequence of interventions with different target sets. It can be adaptive, in which each intervention in the sequence is designed based on the information obtained from previous interventions, or non-adaptive, in which all the interventions in the sequence are designed before any data is collected. There are two main questions regarding the design of interventional experiments for structure learning:
\begin{enumerate}
\item What is the smallest required number of interventions in order to fully learn the underlying causal graph?
\item For a fixed number of interventions (budget), what portion of the causal graph is learnable?
\end{enumerate}

The first problem has been addressed in the literature under different assumptions \citep{eberhardt2007causation, eberhardt2012almost, he2008active, shanmugam2015learning}.
Specifically, \cite{eberhardt2007causation} provided the worst case bounds on the number of required interventions for different types of interventions. The second question mentioned above has received less attention and we address this question herein. We consider a setup in which given a budget $k$, we design $k$ interventions non-adaptively. 
The setup in this work can be interpreted as an extension of the adaptive experiment design, in which interventions are designed in batches of size $k$, i.e., setting $k=1$, reduces the setup to the standard adaptive experiment design. 
 The main contributions of this work are summarized bellow:
\begin{itemize}
\item We cast the problem of finding the best intervention target set as an optimization problem which aims to maximize the experiment gain. The gain is defined as the number of edges whose directions are identified due to the performed interventions. We consider the optimization of the worst-case gain, as well as the average gain.
\item We start the investigation of the optimization problems by considering the case that the underlying causal structure is a tree. For this case, we present an efficient exact algorithm for the worst-case gain setup, as well as an approximate algorithm for the average gain setup. The latter is based on proving that the objective function for the average gain setup is a monotonically increasing and submodular set function.
\item We extend the approximate algorithm to the case of general causal DAGs. In this case, besides the design of interventions, calculating the objective function is also challenging. We propose an efficient exact calculator as well as an unbiased and a fast heuristic estimator for this task. Convergence analysis is provided for the unbiased estimator.
\end{itemize}


This paper is an extended version of our previous work \citep{ghassami2018budgeted,ghassami2018counting}. 
Here, we have provided an extended and more detailed presentation of the approach and algorithms. Also, we have added the study of the experiment design for the case that the underlying structure is a tree. Moreover, we have provided extended performance evaluations in the absence and presence of estimation errors.

The rest of the paper is organized as follows:
After a brief review of related works in Section \ref{sec:relwork},
we start the exposition with reviewing required concepts and classic results, as well as introducing notations and terminologies in Section \ref{sec:prelim}. A formal description of the problem setup is presented in Section \ref{sec:desc}. The proposed experiment design approach for tree causal structures and general causal structures are presented in Sections \ref{sec:tree} and \ref{sec:general}, respectively. 
A variation of the general greedy algorithm through lazy evaluations 
is presented in Section \ref{sec:algorithm}.
Using synthetic and real data, the proposed methods are evaluated in Section \ref{sec:exp}; and finally, our concluding remarks are presented in Section \ref{sec:conc}.
All the proofs are provided in the Appendix.

\section{Related Works}
\label{sec:relwork}

The main methods in the literature for learning causal DAGs from purely observational data include constraint-based methods \citep{spirtes2000causation,pearl2009causality}, score-based methods \citep{heckerman1995learning,chickering2002optimal}, and hybrid methods \citep{tsamardinos2006max}. 
As mentioned in the introduction, without any extra assumptions on the generating causal modules, Markov equivalence class of the ground truth structure is the extent of learnability, and performing interventions is needed for learning beyond Markov equivalence.

A formal definition and the details of the utilization of interventions for the task of causal discovery is provided by \cite{pearl2009causality} and \cite{spirtes2000causation}.
Especially, \cite{pearl2009causality} used the concept of atomic intervention, in which the intervened variable is forced to one particular value rather than a non-degenerate distribution, that is, $X_i=x_i$, for some value $x_i$ in the support of random variable $X_i$.
Works including \citep{eberhardt2007causation, eberhardt2012almost, he2008active, shanmugam2015learning} address the problem of finding the smallest number of interventions required for fully identifying the causal structure.
\cite{eberhardt2007causation} provided the worst case bounds on the number of required interventions for different types of interventions.
\cite{hyttinen2013experiment} drew connections between causality and known separating system constructions.
\cite{eberhardt2012almost} conjectured regarding the number of intervention with targets of unbounded size sufficient and in the worst case necessary for fully identifying a causal model. The conjecture was proved in \cite{hauser2014two} where the authors provided an algorithm that finds such a set of interventions in polynomial time. 
The problem of intervention design with interventions of unbounded size is also addressed in the case that each variable has a certain cost to intervene on \citep{kocaoglu2017cost, lindgren2018experimental}.

Note that the aforementioned works mostly assume that the cardinality of the interventions could be as large as half of the order of the graph, which may render the applicability of the results infeasible for some applications.
\cite{shanmugam2015learning} considered the problem of learning a causal graph when intervention sizes are bounded by some parameter and provided a lower bound on the number of required interventions for adaptive algorithms.
We focus on a setup with singleton interventions, i.e., interventions of size 1. As will be explained in Section 4, this setup is suitable for the applications that certain variables cannot be randomized simultaneously, and also  maximizes the gain obtained from the performed randomizations.
There are other works focused on singleton interventions as well \citep{eberhardt2006n,he2008active,hauser2014two}.
\cite{eberhardt2006n} showed that $N-1$ experiments suffice to determine the causal relations among $N>2$ variables when each experiment randomizes at most one variable.
\cite{he2008active} proposed an adaptive algorithm to minimize the uncertainty of candidate structures based on the minimax and the maximum entropy criteria.
\cite{hauser2014two} provided a greedy adaptive approach that maximizes the number of orientable edges based on a minimax optimization.  

The problem of interventional causal structure learning is also  considered in the causally insufficient systems (i.e., with latent confounders) \citep{kocaoglu2017experimental}. There also exist works that consider the problem of adaptive intervention design using a Bayesian framework, in which a distribution over possible structures and their associated parameters is maintained \citep{tong2001active,masegosa2013interactive}.

One less usual connection to the problem of interventional structure learning when we are limited to a budget of $k$ vertices to intervene on, is with the literature concerned with the influence maximization problem. 
The goal in the influence maximization problem is to find $k$ vertices (seeds) in a given network such that under a specified influence model, the expected number of vertices influenced by the seeds is maximized \citep{kempe2003maximizing,leskovec2007cost,chen2009efficient}.   
Besides the interpretative differences, an important distinction between the two problems is that in the influence maximization  problem, the goal is to spread the influence to the vertices of the graph, while in budgeted experiment design problem, the goal is to pick the initial $k$ vertices in a way that leads to discovering the orientation of as many edges as possible.  Therefore, the optimal solution to these two problems for a given graph can be quite different (see the appendix for an example).

\section{Preliminaries}
\label{sec:prelim}

In this section we briefly review concepts and classical results from the fields of graph theory, graphical models and causal structure learning, needed in the rest of the exposition. For the definitions in this section, we mainly follow \cite{pearl2009causality}, \cite{spirtes2000causation}, and   \cite{andersson1997characterization}. 

\subsection{Graphical Notation and Terminology}

A graph $G$ is a pair $G=(V(G),E(G))$, where $V(G)$ is a finite set of vertices and $E(G)$, the set of edges, is a subset of $(V\times V)\setminus\{(a,a):a\in V\}$. If for an edge $(a,b)\in E(G)$ its opposite edge, i.e., $(b,a)$, also belongs to $E(G)$ then this edges is called an \emph{undirected edge}, and we write $a-b\in G$. If for an edge $(a,b)\in E(G)$, we have $(b,a)\not\in E(G)$, then this edge is called a \emph{directed edge}, and we write $a\rightarrow  b\in G$. In this case, vertex $a$ is called a \emph{parent} of vertex $b$ and $b$ is called a \emph{child} of $a$. The set of parents and children of vertex $a$ are denoted by $\Pa(a)$ and $\Ch(a)$, respectively. 
For vertex $a$, the set of vertices $b$ such that $(a,b)\in E(G)$ or $(b,a)\in E(G)$ is called the set of \emph{neighbors} of $a$, and is denoted by $N(a)$.
A graph is called directed if all of its edges are directed, and is called undirected if all of its edges are undirected. In a graph, a vertex is called a \emph{root} vertex if it does not have any parents. An undirected graph $G^s$, for which $V(G^s)=V(G)$ and $E(G^s)=E(G)\cup\{(a,b):(b,a)\in E(G)\}$ is called the \emph{skeleton} of $G$. For a subset of vertices $A\subseteq V(G)$ the \emph{induced subgraph} of $G$ on $A$ is the graph $G[A]\coloneqq(A,E[A])$, where $E[A]\coloneqq E(G)\cap (A\times A)$.

A sequence of distinct vertices $(a_1,a_2,...,a_m)$ is called a \emph{path} from $a_1$ to $a_m$ if for $1\le i\le m-1$, $(a_i,a_{i+1})\in E(G)$, and is called a \emph{quasi-path} from $a_1$ to $a_m$ if for $1\le i\le m-1$, $(a_i,a_{i+1})\in E(G)$ or $(a_{i+1},a_i)\in E(G)$.  
A sequence of vertices $(a_1,a_2,...,a_m=a_1)$, in which all vertices except the first vertex are distinct, is called a \emph{cycle} if for $1\le i\le m-1$, $(a_i,a_{i+1})\in E(G)$. If all the edges on a path or cycle are directed, then it is called a directed path or cycle. If at least one directed and one undirected edge belongs to a path or cycle, then it is called partially directed. Vertices which have a directed path from vertex $a$ are called the \emph{descendants} of $a$, denoted by $\textit{Desc}(a)$. Any vertex is assumed to be a descendant of itself. A directed acyclic graph (DAG) is a directed graph with no directed cycles.
A chord of a cycle is an edge not in the cycle whose endpoints are in the cycle. A hole in a graph is a cycle of length at least 4 having no chords. A graph is \emph{chordal} if it has no holes.
A graph is called a \emph{chain graph} if it contains no directed or partially directed cycles. After removing all directed edges of a chain graph, the components of the remaining undirected graph are called the \emph{chain components} of the chain graph. 
 
\subsection{Causal Bayesian Networks}
 
A Bayesian network is a probabilistic graphical model representing statistical independencies among a set of variables via a DAG. This type of graphical model is of particular interest in many applications, such as pattern recognition and economics, due to its power in facilitating efficient statistical inference. A Bayesian network is formally defined as follows:

\begin{definition}[Bayesian Network]
\label{def:BN}
Let $G=(V,E)$ be a DAG on a set of random variables $V=\{X_1,...,X_p\}$, and $P_V$ be the joint distribution of $V$.\footnote{In the sequel, we will refer to variables and their corresponding vertices in the graph interchangeably.}
The pair $(G,P_V)$ is called a Bayesian network if each variable in $G$ is independent of its non-descendants given its parents according to $P_V$ (referred to as local Markov property).
\end{definition}
Based on Definition \ref{def:BN}, in a Bayesian network $(G,P_V)$, the joint distribution $P_V$ can be factorized as follows:
\[
P_V = \prod_{X_i\in V} P_{X_i |  \Pa(X_i)},
\]
where $\textit{Pa}(X)$ denotes the set of the parents of variable $X$ in $G$.
\begin{definition}[d-separation]
	In a DAG $G$ a quasi-path is said to be blocked by a subset of vertices $X_S$, $S\subseteq[p]$, if 
	\begin{enumerate}
	\item the quasi-path contains an induced subgraph of form $X_a\rightarrow X_c\rightarrow X_b$ or  $X_a\leftarrow X_c\rightarrow X_b$ such that $X_c$ is in $X_S$, or
	\item the quasi-path contains an induced subgraph of form  $X_a\rightarrow X_c\leftarrow X_b$ such that $X_c$ is not in $X_S$ and no descendant of $X_c$ is in $X_S$.
	\end{enumerate}
For any two variables $X_i$ and $X_j$ and a subset of variables $X_S$, 
we say $X_S$ d-separates $X_i$ from $X_j$,
 denoted by $(X_i~\textit{d-sep}~ X_j|X_S)$, if $X_S$ blocks every quasi-path from $X_i$ to $X_j$ on $G$.		
\end{definition}

Consider Bayesian network $(G,P_V)$. Let $\mathcal{I}(P_V)$ represent the set of all conditional independence relationships in $P_V$, and $\mathcal{I}(G)$ represent the set of all d-separations in $G$.
By definition, distribution $P_V$ satisfies the local Markov property with respect to $G$. As shown in \cite{lauritzen1996graphical}, this implies that every conditional dependency in $P_V$ is reflected in  d-separations in $G$, referred to as \emph{Global Markov property}. However, there may be conditional independencies in $P_V$ which are not reflected in $G$. If there is a one-to-one correspondence between the element of $\mathcal{I}(G)$ and $\mathcal{I}(P_V)$, then 
$G$ is called a perfect I-map for distribution $P_V$.
Therefore, the following extra condition is needed:

\begin{definition}[Faithfulness condition]
The distribution $P_V$ is \emph{faithful} to  structure $G$ if for any two variables $X_i$, $X_j$, and any subset of variables $X_S\subseteq V$, we have
\[
(X_i~\textit{d-sep}~ X_j|X_S)\in\mathcal{I}(G) \textit{ if }
(X_i\perp X_j|X_S)\in\mathcal{I}(P_V).
\]
\end{definition}
For the task of learning a Bayesian network representing a given distribution, it is common in the literature to assume the given distribution satisfies Markov and faithfulness conditions with respect to a DAG \citep{koller2009probabilistic}, as in this case, data can be used to learn a DAG reflecting precisely the conditional independencies in the data.

The directed edges in a perfect I-map does not necessarily imply causation. For instance, for a joint distribution $P_V$ on variables $V=\{X_1,X_2,X_3\}$, such that $\mathcal{I}(P_V)=\{(X_1\perp X_3|X_2)\}$, all three DAGs 
$G_1:X_1\rightarrow X_2\rightarrow X_3$,
$G_2:X_1\leftarrow X_2\rightarrow X_3$, and
$G_3:X_1\leftarrow X_2\leftarrow X_3$ are perfect I-maps.
Nevertheless, the ubiquity of DAG models in statistical applications stems primarily from their causal interpretation \citep{pearl2009causality}. The goal in the field of \emph{causal structure learning} (also known as \emph{causal discovery}) is to learn a directed graph over the variables in the system, $V$, in which a directed edge $X_i\rightarrow X_j$ implies that $X_i$ is a direct cause of $X_j$ with respect to the set $V$. We use the language of structural causal models proposed by \cite{pearl2009causality} to formalize this notion.

For a given set of \emph{endogenous} variables $V=\{X_1,...,X_p\}$, a structural causal model consists of a set of equations of the form 
\begin{equation}
\label{eq:SCM}
X_i=f_i(\Pa(X_i),N_i),\hspace{15mm}1\le i \le p,
\end{equation}
where $\Pa(X_i)\subseteq V\setminus\{X_i\}$ denotes the set of direct causes of variable $X_i$, and $N_i$ is the \emph{exogenous} variable corresponding to $X_i$, representing noise or disturbance. The equation in \eqref{eq:SCM} should be understood as a generating mechanism, and sometimes the notation $X_i\leftarrow f_i(\Pa(X_i),N_i)$ is used.
Consider the directed graph generated by drawing a directed edge from each element of $\Pa(X_i)$ to $X_i$, for all $i\in[p]$. The resulting directed graph $G$ is called the \emph{causal diagram}.
If the causal diagram is acyclic and the exogenous variables are jointly independent, then the model induces a distribution $P_V$ on the endogenous variables that satisfies the local Markov property with respect to $G$ \cite{pearl1995theory}. Therefore, the pair $(G,P_V)$ is a Bayesian network referred to as \emph{causal Bayesian network}. In this paper, we assume that the causal diagram is always a DAG.



\begin{sloppypar}

Two DAGs $G_1$ and $G_2$ are called \emph{Markov equivalent} if $\mathcal{I}(G_1)=\mathcal{I}(G_2)$. 
\cite{verma1990equivalence} proposed a graphical test for Markov equivalence among DAGs:
 Define a v-structure of graph $G$ as a triple of vertices $(a,b,c)$, with induced subgraph $a\rightarrow c\leftarrow b$. Markov equivalence can be tested as follows:
 \end{sloppypar}
 
\begin{lemma}
\label{lem:verma}
\citep{verma1990equivalence} Two DAGs are Markov equivalent if and only if they have the same skeleton and v-structures.
\end{lemma}
For a given DAG $G$, the \emph{Markov equivalence class} (MEC) of $G$ is defined as 
\[
\MEC(G)=\{G': G'\textit{ is DAG, and }\mathcal{I}(G')=\mathcal{I}(G)\}.
\]
That is, the set of all DAGs, which are Markov equivalent with $G$. 
$\MEC(G)$ can be uniquely represented by a graph $\tG=(V(\tG),E(\tG))$, called the \emph{essential graph} corresponding to $\MEC(G)$, for which $V(\tG)=V(G)$, and 
\[
E(\tG)=\bigcup_{G'\in\MEC(G)}E(G').
\]
In other words, an essential graph has the same vertices and skeleton as its members of the corresponding MEC, the directed edges are those that have the same direction in all members of the  class \citep{andersson1997characterization}.
See Figure \ref{fig:exMEC} for an example of all the elements of a MEC and the essential graph corresponding to the MEC.
With a slight abuse of notation, we denote the MEC corresponding to essential graph $\tG$ by $\MEC(\tG)$. Essential graphs are also referred to as completed partially directed acyclic graphs (CPDAGs) \citep{chickering2002optimal}, and maximally oriented graphs \citep{meek1995causal}.  \cite{andersson1997characterization} proposed a graphical criterion for characterizing an essential graph. They showed that an essential graph is a chain graph in which every chain component is chordal. 
As a corollary of Lemma \ref{lem:verma}, for an essential graph $G$, no DAG in $\MEC(G)$ can contain a v-structure in the subgraphs corresponding to chain components of $G$.
In order to obtain the essential graph from observational data, one can first learn the skeleton and v-structures of the underlying DAG using conditional independence tests, and then apply the Meek rules \citep{meek2013causal} to learn the direction of the rest of the directed edges of the essential graph in polynomial time. The Markov and faithfulness assumptions guarantee that the essential graph can be learned from the population dataset.

 \begin{figure}[t]
\begin{center}
\includegraphics[scale=0.25]{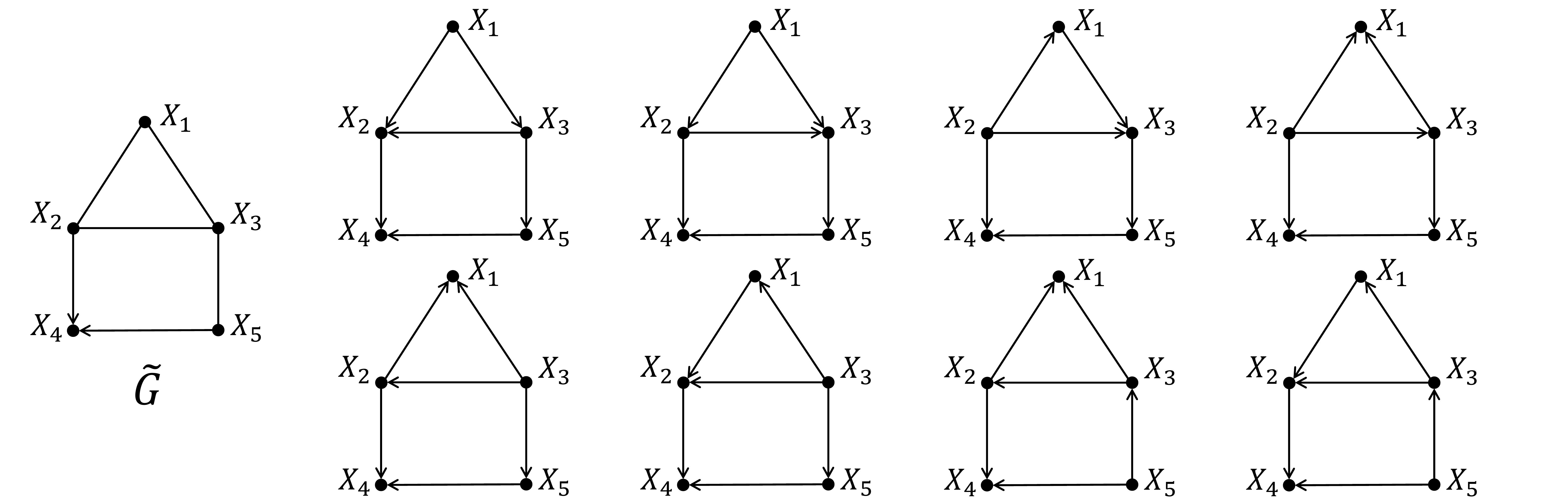}
\caption{Example of the members and the essential graph corresponding to a MEC.}
\label{fig:exMEC}
\end{center}
\end{figure}

\cite{he2008active} observed that the orientation for one chain component does not affect the orientations for other components. Therefore, each chain component can be considered as an essential graph independent of the other components. We call such an essential graph an undirected connected essential graph (UCEG). Note that a UCEG $\tG$ is  chordal and no DAGs in its corresponding equivalence class $\MEC(\tG)$ is allowed to have any v-structures. Each DAG in $\MEC(\tG)$ has exactly one root variable:
\begin{lemma}
\label{lem:root}
Any v-structure-free connected DAG has exactly one root variable.
\end{lemma}
See \citep{bernstein2017sampling} for a proof.



Suppose a joint distribution satisfying Markov and faithfulness conditions to the ground truth causal DAG $G^*$ is given.
Without any assumptions on the type of the functions or the distribution of the exogenous variables in the underlying structural causal model in \eqref{eq:SCM}, the ground truth causal DAG can be identified only up to its Markov equivalence \citep{spirtes2000causation,pearl2009causality}. Hence, the direction of all the edges in the chain components of the essential graph corresponding to $\MEC(G^*)$ will remain unresolved. 
Use interventional experiments is a common method to go beyond Markov equivalence and differentiate among the causal structures within a MEC.
We present a formal definition for an interventional experiment in the following subsection.



\subsection{Interventional Structure Learning}

As mentioned earlier, in general, from a single joint distribution over a set of variables, the ground truth causal structure can be identified up to Markov equivalence. An interventional experiment is the process of perturbing the causal system to generate extra joint distributions over the variables to enable the experimenter to improve the identifiability either merely from the new interventional distributions, or from comparing the original and interventional distributions.

Interventions are generally divided into two types of hard interventions and soft interventions.
In a hard intervention on a variable $X$, all the influences on $X$ are removed and a new value or distribution is forced on $X$, while in a soft intervention on $X$, this variable will still be influenced by its original causes after the intervention. Below, we provide a formal definition of an intervention, in which we mainly follow \cite{eberhardt2007causation}. 

Consider a causal Bayesian network $(G,P_V)$ on a set of variables $V=\{X_1,...,X_p\}$ with observational joint distribution $P_V$. Let $X_T$ be the subset of $V$ that are subject to intervention, called the intervention target set, and for $i\in T$, let $W_i$ be the \emph{intervention variable} corresponding to $X_i$. Intervention variables are jointly independent, are not influenced by any of the variables in the system, and for all $i\in T$, $W_i$ directly influences only $X_i$.
A passive observation is considered to be an intervention with empty target set.

\begin{definition}[Hard Intervention]
\label{def:hard}
 A hard intervention $I=(X_T,W_T)$ on $X_T$, for all $i\in T$
 breaks the causal influence from $\Pa(X_i)$ to $X_i$, i.e., makes $X_i$ independent of $\Pa(X_i)$, and sets the intervention variable $W_i$ as the only direct cause of $X_i$.
For all $i\in T$, $W_i$ determines the distribution of $X_i$, that is, in the factorized joint distribution, replaces the term $P_{X_i|\Pa(X_i)}$ with $P^{(I)}_{X_i}$. 
In the language of structural causal model, $I=(X_T,W_T)$ replaces $X_i=f_i(\Pa(X_i),N_i)$ with $X_i=f^{(I)}_i(W_i,N_i)$, for all $i\in T$.
Graphically, for all $i\in T$, it removes the directed edges from $\Pa(X_i)$ to $X_i$, and sets the intervention variable $W_i$ as the only parent of $X_i$ to form the interventional graph $G^{(I)}$.
\end{definition}

Intervention $I$ changes the joint distribution of $X_T$ and all variables in the system for which an element of $X_T$ is a direct or indirect cause, and results in an interventional joint distribution $P^{(I)}_{V}$. The resulting interventional joint distribution can be factorized as follows:
\[
P^{(I)}_{V}=\prod_{X_i\in X_T}P^{(I)}_{X_i}\prod_{X_i\in V\setminus X_T}P_{X_i|\Pa(X_i)}.
\]
As a specific example of a hard intervention, one can choose $W_i$ to have the same support as the support of $X_i$,
 and forces random values of $W_i$ to $X_i$ via $X_i=W_i$.
 Hard intervention or its variations are also referred to as surgical interventions \citep{pearl2009causality}, ideal interventions \citep{spirtes2000causation}, independent interventions \citep{korb2004varieties}, and structural interventions \citep{eberhardt2007causation} in the literature.

\begin{definition}[Soft Intervention]
  A soft intervention $I=(X_T,W_T)$ on $X_T$, for all $i\in T$ adds the intervention variable $W_i$ as an extra direct cause to $X_i$.
For all $i\in T$, $W_i$ directly influences the distribution of $X_i$, that is, in the factorized joint distribution, replaces the term $P_{X_i|\Pa(X_i)}$ with $P^{(I)}_{X_i|\Pa(X_i)}$, where $P_{X_i|\Pa(X_i)}\neq P^{(I)}_{X_i|\Pa(X_i)}$. 
In the language of structural causal model, $I=(X_T,W_T)$ replaces $X_i=f_i(\Pa(X_i),N_i)$ with $X_i=f^{(I)}_i(\Pa(X_i),W_i,N_i)$, for all $i\in T$.
Graphically, for all $i\in T$, it adds the intervention variable $W_i$ as a parent of $X_i$ to form the interventional graph $G^{(I)}$.
\end{definition}
The resulting interventional joint distribution can be factorized as follows:
\[
P^{(I)}_{V}=\prod_{X_i\in X_T}P^{(I)}_{X_i|\Pa(X_i)}\prod_{X_i\in V\setminus X_T}P_{X_i|\Pa(X_i)}.
\]
Soft intervention or its variations are also referred to as dependent interventions \citep{korb2004varieties}, and parametric interventions \citep{eberhardt2007causation} in the literature.

\cite{eberhardt2007causation} provided a more general definition of intervention than what we presented here. Compared to Eberhardt's definition, we do not allow the intervention variables to be confounded by the variables in the system. Also, we do not allow one intervention variable to influence more than one variable of the system, i.e., in our setup simultaneous intervention on two variables require two independent intervention variables.

Neither hard nor soft intervention can be considered as the more general notion of intervention, and either of them can be more practical depending on the application. For instance, in a medical study on the effect of alcohol on blood pressure, if the target variable is the amount of alcohol consumption, it is often feasible to assign a certain value to this variable regardless of other factors which may influence it. However, if the target is the blood pressure, it is not feasible to remove all the other causes of this target variable, yet the value of one of the known causes can be perturbed. In fact, performing a soft intervention is often more challenging \citep{eberhardt2007causation}. This is due to the fact that any change in the system may lead to removing a subset of the other causes of the target variable. 

For an intervention $I$, the cardinality of the intervention target set, i.e., $|T|$, is referred to as the size of the intervention $I$. An intervention is called \emph{singleton} if it has size equal to one. We define an \emph{experiment} of size $k$ as a sequence of $k$ interventions $\mathcal{E}=\{I_1,...,I_k\}$. 
An experiment is called \emph{adaptive} if in the sequence of interventions, the information obtained from the previous interventions is used to design the next one, otherwise it is called \emph{non-adaptive}, in which the intervention sequence is determined before any data is collected.
A non-adaptive experiment gives the experimenter the ability to perform the interventions in parallel without the need to wait for the result of one intervention to choose the next one. For example, in the study of gene regulatory networks (GRNs), when the GRN of all cells are the same, interventions can be performed simultaneously on different cells. Furthermore, as observed by \cite{eberhardt2005number}, in the worst case, no adaptive experiment design can reduce the number of interventions required for structure learning.


\cite{hauser2012characterization} and \cite{yang2018characterizing} extended the notion of Markov equivalence to the interventional case. For an experiment $\E$, DAGs $G_1$ and $G_2$ are interventional Markov equivalent if $G_1^{(I)}$ and $G_2^{(I)}$ are Markov equivalent for all $I\in\E$. Based on this notion of equivalence, interventional Markov equivalence class and interventional essential graph are defined similar to the observational case.

In the next section we formally define the problem setup and our assumptions for interventional causal structure learning.

\section{Problem Description}
\label{sec:desc}

We study the problem of causal structure learning over a set of $p$ endogenous variables $V=\{X_1,...,X_p\}$, with ground truth causal structure $G^*$ using interventions. Similar to \cite{he2008active},  \cite{shanmugam2015learning}, and \cite{kocaoglu2017cost}, we consider the case that observational data is available and hence, the interventions can be designed based on the output of an initial passive observational stage. This implies that on the population dataset, we design the interventions with side information about the MEC of the ground truth causal structure.

We consider a setup in which we are given a budget of $k$ interventions, and we design the interventions with the goal of discovering the direction of as many edges as possible in the causal graph. Interventions are designed non-adaptively, that is, each intervention is performed regardless of the information gained from the other interventions.
Note that an adaptive experiment design is a special case of our problem: In an adaptive setup, given the information deduced from the collected data, the next intervention is designed. Therefore, this setup is equivalent to ours when $k=1$. Equivalently, our setup could be considered as an extension of adaptive experiment design when the interventions are design in batches of size $k$.

\begin{sloppy}

After performing each intervention $I_i$, data is collected from interventional joint distribution $P^{(I_i)}_V$. Eventually, the observational data and the data gathered from interventions is used for the final output of the procedure. We use the GIES algorithm \citep{hauser2012characterization} for this final step.

\end{sloppy}

We assume that all the interventions should be singleton, i.e., each intervention should have size equal to one. This is beneficial since in some applications, the experimenter may not be able to randomize certain variables simultaneously.
Note that most of the literature assume that the size of each intervention is larger than one, in some cases going as high as half of the number of variables \citep{eberhardt2005number, eberhardt2012almost, hauser2014two, kocaoglu2017cost}. 
Therefore, the set of $k$ variables $\I=\{X_{I_1},...,X_{I_k}\}$ contains all the information to describe the targets in the experiment, where $X_{I_i}$ is the single variable intervened on in intervention $I_i$. We call the set $\I$ the target set of the experiment.
We denote the interventional MEC containing DAG $G$ by $\I$-$\MEC(G)$. 
Note that the passive observational experiment is contained in the experiment set, i.e., $\I$-$\MEC(G)$ contains all graphs $G'$, such that $G'$ is Markov equivalent to $G$ and $G'^{(I_i)}$ is Markov equivalent to ${G}^{(I_i)}$, for all singleton interventions $I_i$, $1\le i\le k$. 
We have the following assumptions in this work:

\begin{assumption}
\label{assumption:DAG}
The ground truth causal structure $G^*$ is a DAG and exogenous variables in the structural causal model are jointly independent.
\end{assumption}
\begin{assumption}
\label{assumption:faithful}
The observational and interventional joint distributions satisfy Markov and faithfulness conditions with respect to their corresponding observational and interventional DAGs.
\end{assumption}
\begin{assumption}
\label{assumption:oracle}
The correct essential graph $\tG^*$ can be learned from the initial observational dataset.
\end{assumption}

Under Assumptions \ref{assumption:DAG}-\ref{assumption:oracle}, we have the following result regarding the effect of a singleton intervention.

\begin{lemma}
\label{lem:neigh}
Having the observational essential graph $\tG^*$, a singleton intervention (hard or soft) on variable $X_i$  identifies the direction of all edges incident with $X_i$.
\end{lemma}

\cite{eberhardt2005number} and \cite{he2008active} provided the same result as in Lemma \ref{lem:neigh} with different proofs. Also, \cite{eberhardt2005number} observed that given the essential graph resulted from the passive observational stage, a hard intervention $I$ allows orientating the undirected edge $X_i- X_j$ if only one of $X_i$ and $X_j$ is in the target set of $I$. If both $X_i$ and $X_j$ are targeted in the intervention, this intervention is called a \emph{zero-information} intervention for the pair $\{X_i,X_j\}$. 
Our setup in which $|I_i|=1$, for all $i\in\{1,...,k\}$, avoids such zero-information experiments. Therefore, another advantage of forcing singleton interventions is that there will be no zero-information interventions in the experiment and hence, we gain the most from each randomization.
We note that a zero-information intervention does not happen for the case of soft interventions:

\begin{lemma}
\label{lem:soft}
A sequence of $k$ singleton soft interventions is equivalent to one soft intervention of size $k$ on the same targets.
\end{lemma}

Lemma \ref{lem:soft} is a corollary of Theorem 2 of \cite{eberhardt2007interventions}.
By Lemma \ref{lem:soft}, if the performed interventions are soft, they can be done simultaneously as one soft intervention of size $k$, i.e., we can have $|\E|=1$, and $|I_1|=k$. Nevertheless, as mentioned in Section \ref{sec:prelim}, soft interventions are in general more challenging to perform.

By Assumption \ref{assumption:oracle}, we assume that $\MEC(G^*)$, and hence, its corresponding essential graph $\tG^*$ is attainable from the observational data. 
Let $G_i\in\MEC(G^*)$, and for experiment with target set $\I$, denote the interventional Markov equivalence class containing $G_i$ and its corresponding interventional essential graph by $\I$-$\MEC(G_i)$ and $\tG^{(\I)}_i$, respectively.
Define $R(\I,G_i)$ as the set of edges directed in $\tG^{(\I)}_i$ but not directed in $\tG^*$, i.e., the set of edges whose directions can be learned due to the experiment with target set $\I$, if the ground truth DAG were $G_i$. Note that $R(\I,G)$ is the same for all $G\in\I$-$\MEC(G_i)$. 
$R(\I,G_i)$ can be obtained as follows: As seen in Lemma \ref{lem:neigh}, from an experiment with target set $\I$, one learns the direction of all the edges incident with the vertices in $\I$. 
Denote these directed edges by $A(\I,G_i)$. (Clearly, the orientation of these edges depends on the ground truth DAG $G_i$, and hence $G_i$ is an input argument.) 
Meek rules \citep{meek2013causal} can then be applied to $A(\I,G_i)$ to obtain extra edges oriented in $\tG^{(\I)}_i$ compared to $\tG^*$ in polynomial time.

Define the \emph{gain} of an experiment with target set $\I$ on ground truth structure $G_i$ as $D(\I,G_i)=|R(\I,G_i)|$, that is, the number of edges whose direction is discovered due to the experiment, if the ground truth DAG were $G_i$. Since the ground truth DAG is initially known only up to the elements of $\MEC(G^*)$, and since there is no preference between the members of $\MEC(G^*)$, $G^*$ is equally likely to be any of the DAGs in the class. Hence, the expected number of the edges recovered through the experiment with target set $\I$ is
\begin{equation}
\label{eq:summ}
\begin{aligned}
\mathcal{D}(\I)\coloneqq
\frac{1}{|\MEC(G^*)|}\sum_{G_i\in\MEC(G^*)}D(\mathcal{I},G_i).
\end{aligned}
\end{equation} 
We refer to $\mathcal{D}(\I)$ as the average gain of the experiment with target set $\I$. Thus, our problem of interest can be formulated as finding intervention target set $\mathcal{I}\subseteq V$ of cardinality $k$ that maximizes $\mathcal{D}(\mathcal{I})$:
\begin{equation}
\label{eq:bayes}
\max_{\mathcal{I}:\mathcal{I}\subseteq V} \mathcal{D}(\mathcal{I})~~~\text{s.t.}~~~|\mathcal{I}|= k.
\end{equation} 
We refer to \eqref{eq:bayes} as the \emph{average gain} optimization problem.
Optimization problem \eqref{eq:bayes} is challenging for two reasons: First, finding an optimal $\I$ requires a combinatorial search. Second, even for a given set $\I$, computing $\mathcal{D}(\I)$ when the value of $k$ or the cardinality of the Markov equivalence class is large, can be computationally intractable. Note that the cardinality of a MEC can be super-exponential in the number of vertices \citep{he2015counting}.

Alternatively, one can consider a minimax setup, and design the experiment for the worst-case member of the equivalence class:
\begin{equation}
\label{eq:mM}
\max_{\mathcal{I}:\mathcal{I}\subseteq V} \min_{G_i\in\MEC(G^*)}D(\mathcal{I},G_i)~~~\text{s.t.}~~~|\mathcal{I}|= k.
\end{equation} 
We refer to \eqref{eq:mM} as the \emph{worst-case gain} optimization problem.
Optimization problem \eqref{eq:mM} is studied by \cite{hauser2014two} for the case of $k=1$. Here, we consider the challenges raised when $k$ is larger than 1 and a brute force search over all subsets of $V$ of size $k$ is not computationally feasible. 
\cite{he2008active} have also considered a similar setup with singleton interventions with $k=1$. But their objective functions are different and they perform a brute force search to find the optimum target.

In Section \ref{sec:tree} we study optimization problems \eqref{eq:bayes} and \eqref{eq:mM} for the case that the underlying causal structure is a tree, and we consider the general case in Section \ref{sec:general}.

\section{Experiment Design for Tree Structures}
\label{sec:tree}

We start the investigation of optimization problems \eqref{eq:bayes} and \eqref{eq:mM} by considering the case that the underlying causal structure is a tree. 
For the obtained essential graph from the observational stage, Let $\tilde{T}_1, ..., \tilde{T}_R$ denote the induced subgraphs of the essential graph on the non-trivial chain components. Note that by definition, each $\tilde{T}_r$ is a UCEG. As mentioned in Section \ref{sec:prelim}, orientations for one chain component of an essential graph does not affect the orientations for the other components. Thus, for a given number of interventions assigned to one UCEG, the task of experiment design in that UCEG becomes independent from other UCEGs.  


Recall from Lemma \ref{lem:root} that for a given UCEG $\tG$, each DAG in $\MEC(\tG)$ has a unique root variable. Here, since the DAG is a tree and should be v-structure-free, knowing the root variable identifies the orientation of all the edges:
\begin{lemma}
	\label{lem:rootree}
	For a tree UCEG $\tT$, no two DAGs in $\MEC(\tT)$ have the same root variable, that is, the location of the root variable identifies the direction of all the edges.
\end{lemma}

For a tree UCEG $\tT_r, 1\leq r\leq R,$ and any variable $X\in V(\tT_r)$, let $T_r^{X}$ be the unique directed tree in $\MEC(\tT_r)$ with root variable $X$. Based on Lemmas \ref{lem:root} and \ref{lem:rootree}, $\MEC(\tT_r)=\{T_r^{X}:X\in V(\tT_r)\}$. Therefore, optimization problem \eqref{eq:bayes} can be written as
\begin{equation}
\label{eq:treebayes}
\begin{aligned}
\max_{\mathcal{I}:\mathcal{I}\subseteq V} \frac{1}{p_u}\sum_{r=1}^R\sum_{X\in V(\tT_r)}D(\I_r,T_r^{X}),~~~\text{s.t.}~~~\sum_{r=1}^R|\mathcal{I}_r|= k,
\end{aligned}
\end{equation}
where $p_u\coloneqq\sum_{r=1}^R |V(\tT_r)|$, and $\I_r$ is the set of intervened variables in chain component $\tT_r$, i.e., $\I_r\coloneqq \I\cap V(\tT_r)$. Furthermore, the optimization problem \eqref{eq:mM} can be written as
\begin{equation}
\label{eq:treemM}
\begin{split}
&\max_{\mathcal{I}:\mathcal{I}\subseteq V} \min_{\{X_{i_1},\cdots,X_{i_R}\}\subseteq V}\sum_{r=1}^R D(\I_r,T_r^{X_{i_r}})~~~\text{s.t.}~~~\sum_{r=1}^R|\mathcal{I}_r|= k\\
&\equiv\max_{\mathcal{I}:\mathcal{I}\subseteq V} \sum_{r=1}^R \min_{X\in V(\tT_r)}D(\I_r,T_r^{X})~~~\text{s.t.}~~~\sum_{r=1}^R|\mathcal{I}_r|= k,
\end{split}
\end{equation} 
where the two optimization problems are equivalent due to the fact that orienting edges in one UCEG does not affect orientations of the edges in other UCEGs, and hence, minimization on the root of UCEGs can be done separately.

Let $\{C_1(\I_r),...,C_{J(\I_r)}(\I_r)\}$ be the set of components of $\tT_r\setminus\I_r$, i.e., the components resulting from removing vertices $\I_r$ and edges incident to them from $\tT_r$, where $J(\I_r)$ is the number of the resulted components.
We have the following result regarding the calculation of the gain $D(\I_r,T_r^{X})$.
\begin{lemma}
	\label{lem:gain}
	For any $X\in V(\tT_r)$ and experiment target set $\I_r\subseteq V(\tT_r)$, the gain $D(\I_r,T_r^{X})$ can be calculated as follows:	
	\[
	D(\I_r,T_r^{X})=
	\begin{cases}
	|\tT_r|-1 &\quad X\in\I_r,\\
	|\tT_r|-|C_j(\I_r)| &\quad X\in C_j(\I_r),
	\end{cases}
	\]
	where $|G|$ denotes the order (number of vertices) of $G$.
\end{lemma}

Using Lemma \ref{lem:gain}, the average gain of an experiment target set $\I$ can be calculated by the following proposition:
\begin{proposition}
	\label{prop:avggain}
	The average gain of an experiment target set $\I\subseteq V$ is given as follows:
	\begin{equation}
	\label{eq:treegain}
	\mathcal{D}(\I)= \frac{1}{p_u}\sum_{r=1}^R |\tT_r|^2-\frac{k}{p_u}-\frac{1}{p_u}\sum_{r=1}^R \sum_{j=1}^{J(\I_r)} |C_j(\I_r)|^2.
	\end{equation}
\end{proposition}


Based on Lemma \ref{lem:gain} and Proposition \ref{prop:avggain}, the optimizer of the optimization problem \eqref{eq:treebayes} can be found by solving
\begin{equation}
\label{eq:fitreebayes}
\begin{aligned}
\min_{\I:\I\subseteq V} \sum_{r=1}^R \sum_{j=1}^{J(\I_r)} |C_j(I_r)|^2,~~~\text{s.t.}~~~\sum_{r=1}^R|\mathcal{I}_r|= k.
\end{aligned}
\end{equation}
Also, we have
\begin{equation*}
\label{eq:fitreemM}
\begin{split}
&\arg\max_{\mathcal{I}:\mathcal{I}\subseteq V} \sum_{r=1}^R \min_{X\in V(\tT_r)}D(\I,T_r^{X}),~~~\text{s.t.}~~~\sum_{r=1}^R|\mathcal{I}_r|= k\\
&=\arg\max_{\mathcal{I}:\mathcal{I}\subseteq V} \sum_{r=1}^R \min_{1\leq j\leq J(\mathcal{I}_r)} |\tilde{T}_r|-|C_j(\mathcal{I}_r)|,~~~\text{s.t.}~~~\sum_{r=1}^R|\mathcal{I}_r|= k\\
&=\arg\max_{\mathcal{I}:\mathcal{I}\subseteq V} \sum_{r=1}^R |\tT_r|-\max_{1\leq j\leq J(\mathcal{I}_r)} |C_j(\mathcal{I}_r)|,~~~\text{s.t.}~~~\sum_{r=1}^R|\mathcal{I}_r|= k\\
&=\arg\max_{\mathcal{I}:\mathcal{I}\subseteq V} \sum_{r=1}^R -\max_{1\leq j\leq J(\mathcal{I}_r)} |C_j(\mathcal{I}_r)|,~~~\text{s.t.}~~~\sum_{r=1}^R|\mathcal{I}_r|= k\\
&=\arg\min_{\mathcal{I}:\mathcal{I}\subseteq V} \sum_{r=1}^R\max_{1\leq j \leq J(\mathcal{I}_r)}|C_j(\mathcal{I}_r)|,~~~\text{s.t.}~~~\sum_{r=1}^R|\mathcal{I}_r|= k.
\end{split}
\end{equation*}
Hence, the optimizer of the optimization problem \eqref{eq:treemM} can be found by solving
\begin{equation}
\label{eq:fitreemM}
\min_{\mathcal{I}:\mathcal{I}\subseteq V} \sum_{r=1}^R\max_{1\leq j \leq J(\mathcal{I}_r)}|C_j(\mathcal{I}_r)|,~~~\text{s.t.}~~~\sum_{r=1}^R|\mathcal{I}_r|= k.
\end{equation}



Clearly, the optimization problems in \eqref{eq:fitreebayes} and \eqref{eq:fitreemM} can be solved via a brute-force search over all $\binom{p}{k}$ target sets, which can be  computationally intensive. In Subsections \ref{sec:algmM} and \ref{sec:alggreedy}, we will introduce efficient algorithms to address these optimization problems.

\subsection{Optimizing the Worst-Case Gain in Tree Structures}
\label{sec:algmM}

We start with the optimization problem in \eqref{eq:fitreemM}. As mentioned before, for a fixed number of intervention in UCEG $\tT_r$, the task of experiment design in that UCEG becomes independent of other UCEGs. Thus, we can formulate the optimization problem in \eqref{eq:fitreemM} as follows:
\begin{equation}
\begin{split}
&\min_{(\I_1,...,\I_R):\sum_{r=1}^R |\I_r|=k} \sum_{r=1}^R\max_{1\leq j \leq J(\mathcal{I}_r)}|C_j(\mathcal{I}_r)|\\
&\equiv\min_{(\I_1,...,\I_R): |\I_r|=k_r,\sum_{r=1}^R k_r=k} \sum_{r=1}^R\max_{1\leq j \leq J(\mathcal{I}_r)}|C_j(\mathcal{I}_r)|\\
&\equiv\min_{(k_1,...,k_R): \sum_{r=1}^R k_r=k} \sum_{r=1}^R \min_{\I_r:|\I_r|=k_r} \max_{1\leq j \leq J(\I_r)} |C_j(\I_r)|.
\end{split}
\label{eq:problem_minmax}
\end{equation}
Herein, we first propose Algorithm \ref{algTrED} that solves for the minimax problem in the summation in expression \eqref{eq:problem_minmax} for each given UCEG $\tT_r$. That is, Algorithm \ref{algTrED} finds a set $\I_r$ in  $\tT_r$ of size $k_r$ such that after removing the variables in $\I_r$, the maximum size of the remaining components is minimized. Next, we will show that how Algorithm \ref{algTrED} can be utilized to obtain an optimum solution of the problem in \eqref{eq:problem_minmax}.


Algorithm \ref{algTrED} takes a UCEG $\tT_r$ and budget of intervention $k_r$ as inputs and returns the set $\hat{\I}_{r}$ that is a solution of the following minimax problem:
\begin{equation}
\label{eq:subminimax}
\min_{\I_r:\I_r\subseteq V(\tT_r)} \max_{1\leq j \leq J(\I_r)} |C_j(\I_r)|,~~~\text{s.t.}~~~|\mathcal{I}_r|= k_r.
\end{equation}

In the main loop of Algorithm \ref{algTrED}, each variable $X_i\in V(\tT_r)$ is set as the starting point for performing Depth-First Search (DFS) on $\tT_r$. For a given threshold value $mid$, $1\le mid\le |\tT_r|$, the algorithm does the following. On the traversal of DFS, whenever all the descendants of a variable $X_j$ are visited, it decides to remove $X_j$ and adds it to the set $\mathcal{I}$ (which is the set of variables on which we will intervene), if not doing so, results in having a component with size larger than $mid$ in the subtree rooted at $X_j$ (lines 8-9). Note that after removing $X_j$, for the rest of variables in the traversal, we do not consider the disconnected vertices anymore. After checking all the variables in DFS, we see if our budget of intervention, i.e., $k_r$ is enough for performing $|\mathcal{I}|$ interventions (line 13). We update the value for $mid$ in each loop using a binary search to find the minimum threshold that can be satisfied by the budget. More specifically, if the number of interventions is less than the budget $k_r$ for a value of $mid$, we narrow down our search space to $[L,mid]$ (lines 13-15). Otherwise, we consider the region $[mid,H]$ (line 17). This procedure will be repeated for all possible choices of the starting point of DFS and we choose the best $\mathcal{I}(X)$ as the output of the algorithm (line 21).


\begin{algorithm}[t]
	\begin{algorithmic}[1]
		\STATE {\bf input:} $\tT_r$, $k_r$.
		\FOR {$X_i\in V(\tT_r)$}
		\STATE $L=1$, $H=|\tT_r|$, $T=\tT_r$
		\WHILE {$\lfloor H\rfloor\neq\lfloor L\rfloor$}
		\STATE $\I=\emptyset$,  $\textit{mid}=(L+H)/2$
		\STATE Perform DFS on $T$ starting from $X_i$.
		\FOR { $X_j\in V$, when all variables in $\textit{Desc}(X_j)$ w.r.t. $T_r^{X_i}$ are visited in DFS traversal,}
		\IF {$|\textit{Desc}(X_j)|>\textit{mid}$}
		\STATE $\I=\I\cup\{X_j\}$
		\STATE $T=T\setminus\textit{Desc}(X_j)$
		\ENDIF
		\ENDFOR
		
		\IF {$|\I|\le k_r$}
		\STATE  $\textit{mid}(X_i)=\textit{mid}$,  $\I(X_i)=\I$
		\STATE $H=\textit{mid}$			
		\ELSE 
		\STATE $L=\textit{mid}$
		\ENDIF
		\ENDWHILE
		\ENDFOR
		\STATE $\hat{\I}_{r}=\I(\arg\min_{X_i}\textit{mid}(X_i))$
		\STATE {\bf output:} $\hat{\I}_{r}$ 
		\caption{Minimax Experiment Design for a UCEG}
		\label{algTrED}
	\end{algorithmic}
\end{algorithm}


\begin{theorem}
\label{thm:algTrED}
	Algorithm \ref{algTrED} returns the optimal solution of the optimization problem in \eqref{eq:subminimax}.
\end{theorem}

Establishing an algorithm for solving the minimax problem in \eqref{eq:subminimax}, we can utilize it to solve the main optimization problem in \eqref{eq:problem_minmax}. To this end, we show that the main problem can be formulated as a multi-choice knapsack problem \citep{dudzinski1987exact}, and hence, it can be solved efficiently by existing algorithms \citep{dudzinski1987exact} proposed for the  multi-choice knapsack problem.

In order to find an optimal solution of \eqref{eq:problem_minmax}, we first using Algorithm \ref{algTrED} obtain the optimal value of objective function in \eqref{eq:subminimax} for every UCEG $\tT_r$ and any assigned budget $k_r=j$, where $0\leq j \leq k$, and denote the optimum value by $D_{r,j}$. 
Also, for each UCEG $\tT_r$ and budget $j$, we define binary indicator variable $x_{r,j}$, where $x_{r,j}=1$ if the budget assigned to $\tT_r$ is equal to $j$, otherwise, $x_{r,j}=0$.
Hence, optimization problem \eqref{eq:problem_minmax} can be reformulated as follows:
\begin{equation}
\begin{split}
\qquad\qquad\min  \sum_{r=1}^R \sum_{j=0}^k & D_{r,j}  x_{r,j} \\
\text{s.t.}\quad &\sum_{r=1}^R \sum_{j=0}^k jx_{r,j}\le k,\\
& \sum_{j=0}^k x_{r,j}=1,\\
& x_{r,j}\in \{0,1\}, \quad \mbox{for all }  1\le r\le R, \mbox{for all } 0\leq j \leq k.
\end{split}
\label{eq:knapsack}
\end{equation}
The first condition ensures that the total number of interventions performed in all UCEGs is less than or equal to budget $k$ and the second condition specifies the number of interventions assigned to each UCEG $\tT_r$. Moreover, the sum $ \sum_{j=0}^k D_{r,j} x_{r,j}$ in the objective function is equal to $D_{r,j}$ if $x_{r,j}=1$. In other words, this sum is equal to the optimal value of objective function in \eqref{eq:subminimax} if $k_r=j$. Thus, the objective function in \eqref{eq:knapsack} is equal to the one in \eqref{eq:problem_minmax}.

Regarding the time complexity of the propose approach, we first run Algorithm \ref{algTrED} on each UCEG for any budget in the range $\{0,...,k\}$. The time complexity of Algorithm \ref{algTrED} is in the order of $\mathcal{O}(p^2\log p)$. This is due to the fact that   DFS runs in time $\mathcal{O}(p)$ for a tree of order $p$ and for a fixed value of parameter $H$, the while loop in Algorithm \ref{algTrED}  will run for $\log_2(H)$ times, which can be at most $\log p$.  
Therefore, the time complexity of obtaining the optimal value of objective function in \eqref{eq:subminimax} for all $1\leq r\leq R$ and $1\leq k_r\leq k$,  is in the order of $\mathcal{O}(p^3k\log p)$. 
Moreover, the time complexity of solving the multi-choice knapsack problem is in the order of $\mathcal{O}(pk^2)$. Hence, the total time complexity of the proposed approach would be in the order of $\mathcal{O}(p^3k\log p)$.

\subsection{Optimizing the Average Gain in Tree Structures}
\label{sec:alggreedy}

We now move to the problem of experiment design on tree structures for maximizing the average gain presented in expression \eqref{eq:treegain}. Unlike the minimax case, in the case of maximizing the average gain, the objective function depends on both the maximum order of the components, as well as how uniform the order of the components are. This fact makes the design of the experiment target set more challenging in the average case. Unfortunately, we do not have an efficient exact algorithm for this case; however, we show that due to submodularity of the objective function, an efficient approximation algorithm for this case can be obtained. We start by reviewing monotonicity and submodularity properties for a set function.
\begin{definition}
A set function $f: 2^V \to \mathbb{R}$ is monotonically increasing if for all sets $\mathcal{I}_1\subseteq \mathcal{I}_2\subseteq V$, we have
\[
f(\mathcal{I}_1)\le f(\mathcal{I}_2).
\]
\end{definition}
\begin{definition}
A set function $f: 2^V \to \mathbb{R}$ is submodular if for all subsets $\mathcal{I}_1\subseteq \mathcal{I}_2\subseteq V$ and all $X \in V \setminus \mathcal{I}_2$,\footnote{If $f$ is monotonically increasing, $X \in V \setminus \mathcal{I}_2$ relaxes to $X\in V$.}
\[
f(\I_1 \cup \{ X \}) - f(\I_1) \ge f(\I_2 \cup \{ X \}) - f(\I_2).
\]
\end{definition}

\cite{nemhauser1978analysis} showed that if $f$ is a submodular and monotonically increasing set function with $f(\emptyset)=0$, then the set $\hat{\mathcal{I}}$ with $|\hat{\mathcal{I}}|=k$ found by a greedy algorithm satisfies 
\[
f(\hat{\mathcal{I}})\ge(1-\frac{1}{e})\max_{\mathcal{I}:|\mathcal{I}|=k}f(\mathcal{I}),
\]
that is, the greedy algorithm is a $(1-\frac{1}{e})$-approximation algorithm. In the following, we show that the set function $\mathcal{D}$ defined in \eqref{eq:treegain} is monotonically increasing and submodular, and hence, since $\mathcal{D}(\emptyset)=0$, the greedy algorithm is a $(1-\frac{1}{e})$-approximation algorithm for the maximization problem \eqref{eq:bayes}.

\begin{proposition}
	\label{prop:treeavg}
For tree structures, the set function $\mathcal{D}$ defined in \eqref{eq:treegain} is monotonically increasing and submodular. 
\end{proposition}

Our general greedy algorithm is presented in Algorithm \ref{algorithm:GG}. We define the \emph{marginal gain} of variable $X$ when the previous chosen set is $\I$ as 
\begin{equation}
\label{eq:marginalgain}
\Delta_X(\mathcal{I})=\mathcal{D}(\I\cup \{X\})-\mathcal{D}(\mathcal{I}). 
\end{equation}
The greedy algorithm iteratively adds a variable which has the largest marginal gain to the target set until it runs out of budget. For any input set $\I$, in order to calculate the value of $\mathcal{D}(\I)$, we use the equation in \eqref{eq:treegain}. 
Note that $\mathcal{D}(\I)$ can be  computed efficiently from \eqref{eq:treegain} as it is just needed to obtain the size of resulted components after removing variables in $\I$. To do so, we can run DFS algorithm on each component. In each DFS call, the size of a component is obtained by visiting the variables in it. Then, we will call DFS  on the next unvisited component until there is no unvisited variable in the essential graph. Therefore,   $\mathcal{D}(\I)$ can be computed in $\mathcal{O}(p)$ since the total number of edges in all components is in the order of $\mathcal{O}(p)$.
%
%

\begin{algorithm}[t]
\begin{algorithmic}
 \STATE {\bf input:} Essential graph from the observational stage, budget $k$.
 \STATE {\bf initialize:}  $\mathcal{I}_0=\emptyset$
\FOR{$i=1$ to $k$}
\STATE $X_i=\arg\max_{X\in V\setminus \I_{i-1}}\mathcal{D}(\I_{i-1}\cup \{X\})-\mathcal{D}(\I_{i-1})$
\STATE $\I_i=\I_{i-1}\cup\{X_i\}$
\ENDFOR
\STATE {\bf output:} $\hat{\I}=\I_k$
 \caption{General Greedy Algorithm}
 \label{algorithm:GG}
\end{algorithmic}
\end{algorithm}

\section{Experiment Design for General Structures}
\label{sec:general}

In this section we consider experiment design for the case of general structures, formulated in optimization problem \eqref{eq:bayes}.
We first generalize Proposition \ref{prop:treeavg} by showing that the function $\mathcal{D}$ defined in \eqref{eq:summ} is monotonically increasing and submodular. 
\begin{proposition}
	\label{prop:mono}
The set function $\mathcal{D}$ defined in \eqref{eq:summ} is monotonically increasing.
\end{proposition}

We use the following lemma in the proof of submodularity of the function $\mathcal{D}$. 
\begin{lemma}   
\label{lem:nofusion}
Let $\I_1$ and $\I_2$ be arbitrary subsets of variables of a DAG $G$. We have
\[
R(\I_1\cup \I_2,G)=R(\I_1,G)\cup R(\I_2,G).
\]
\end{lemma}
As mentioned in Section \ref{sec:desc}, from an experiment with target set $\I$, one learns the direction of all the edges incident with the vertices in $\I$, denoted by $A(\I,G)$, and then  the extra edges in the interventional essential graph can be obtained by, say, using the Meek rules starting from  $A(\I,G)$.
Lemma \ref{lem:nofusion} implies that the set of resolved edges in the essential graph starting from $A(\I_1\cup \I_2,G)$ is the same as the set of edges whose direction is resolved either in the essential graph starting from $A(\I_1,G)$ or in the essential graph starting from $A(\I_2,G)$.

\begin{theorem}
\label{thm:submodular}
The set function $\mathcal{D}$ defined in \eqref{eq:summ} is a submodular function.
\end{theorem}

Equipped with Proposition \ref{prop:mono} and Theorem \ref{thm:submodular}, we can again use Algorithm \ref{algorithm:GG}, to obtain an $(1-\frac{1}{e})$-approximation of the optimal solution of optimization problem \eqref{eq:bayes}.
However, as mentioned in Section \ref{sec:desc}, another challenge regarding solving the optimization problem \eqref{eq:bayes} is the computational aspect of calculating $\mathcal{D}(\mathcal{I})$ for a given experiment target set $\mathcal{I}$.
In Section \ref{sec:tree}, for the case of tree structures, for a given set $\I$, we calculated the value of $\mathcal{D}(\I)$ efficiently by applying DFS algorithm; yet this approach cannot be extended to the case of general structures. In the following subsections, we propose efficient methods for exact calculation and estimation of $\mathcal{D}(\I)$ for general structures.

\begin{remark}
As seen in the proof of Theorem \ref{thm:submodular}, for DAG $G$ and $\I\subseteq V(G)$, the set function $D(\I,G)$ is submodular. However, the minimum of submodular functions is not necessarily submodular. Hence, Algorithm \ref{algorithm:GG}, is not necessarily a $(1-\frac{1}{e})$-approximation algorithm for the case of worst-case gain in optimization problem \eqref{eq:mM}. Nevertheless, our experiment results in Section \ref{sec:exp} suggest that Algorithm \ref{algorithm:GG} leads to high performance in the case of worst-case gain as well.
\end{remark}

\subsection{Exact Calculation of $\mathcal{D}(\I)$}
\label{sec:D(I)calc}

In this section, we show that a method for counting the number of elements in a MEC can be used for calculating $\mathcal{D}(\I)$.
For an essential graph $\tG$, we define the size of its corresponding MEC as the number of DAGs in the class and denote it by $\Size(\tG)$. 
Let $\{\tG_1,...,\tG_R\}$ be the chain components of $\tG$. $\Size(\tG)$ can be calculated from the size of chain components using the following equation \citep{gillispie2002size,he2008active}:
\begin{equation}
\label{eq:prod}
\textit{Size}(\tG)=\prod_{r=1}^R\textit{Size}(\tG_r).
\end{equation}
Therefore, it suffices to calculate the size of UCEGs $\tG_1,...,\tG_R$.
\begin{definition}
	Let $\tG_r$ be a UCEG. The $X$-rooted subclass of $\MEC(\tG_r)$ is the set of all $X$-rooted DAGs in $\MEC(\tG_r)$. This subclass can be represented by the $X$-rooted graph $\tG_r^{X}=(V(\tG_r^X),E(\tG_r^{X}))$, called the $X$-rooted essential graph, where $V(\tG_r^X)=V(\tG_r)$, and $E(\tG_r^X)=\bigcup\{E(G):G\in X\textit{-rooted subclass of }\MEC(\tG_r)\}$.
\end{definition}
For instance, for UCEG $\tG$ in Figure \ref{fig:excount}$(a)$, $\tG^{X_1}$ and $\tG^{X_2}$ are depicted in Figures \ref{fig:excount}$(b)$ and \ref{fig:excount}$(d)$, respectively.

\begin{figure}[t]
\begin{center}
\includegraphics[scale=0.55]{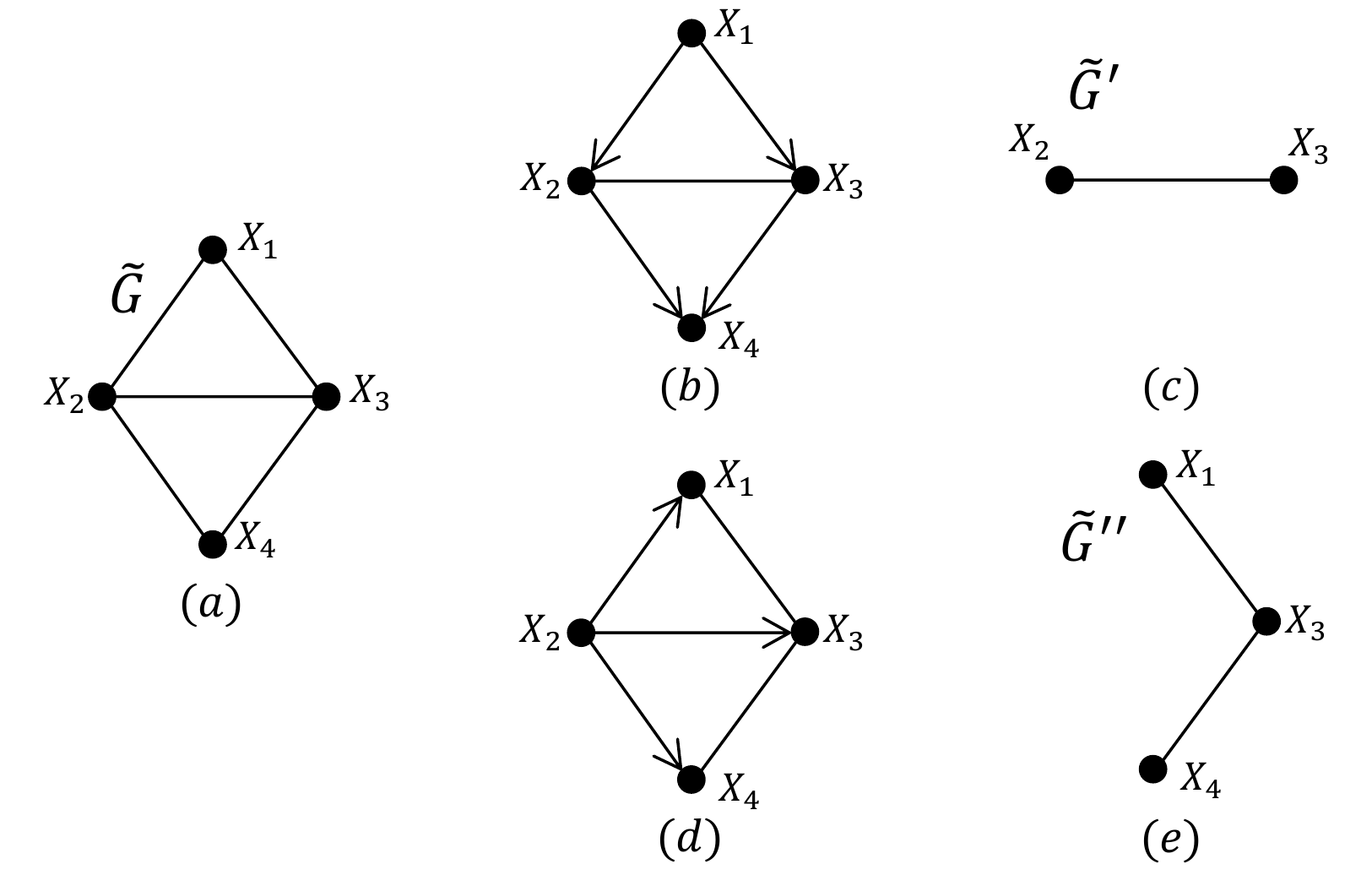}
\caption{Example of the counting and sampling approach.}
\label{fig:excount}
\end{center}
\end{figure}

\begin{lemma}
\label{lem:partition}
	\citep{he2015counting} Let $\tG_r$ be a UCEG. For any $X\in V(\tG_r)$, the $X$-rooted subclass is not empty and the set of all $X$-rooted subclasses partitions $\MEC(\tG_r)$.
\end{lemma}
From Lemma \ref{lem:partition} we have
\begin{equation}
\label{eq:sum}
\Size(\tG_r)=\sum_{X\in V(\tG_r)}\textit{Size}(\tG_r^{X}).
\end{equation}
Therefore, using equations \eqref{eq:prod} and \eqref{eq:sum}, we have
\begin{equation}
\label{eq:reccount}
\Size(\tG)=\prod_{r=1}^R\sum_{X\in V(\tG_r)}\textit{Size}(\tG_r^{X}).
\end{equation}


$\tG_r^X$ could be viewed as an essential graph, as it can be considered as an interventional essential graph with target variable $X$, for which in the underlying DAG all of edges incident to $X$ are outgoing edges.
Hence, the number of DAGs in its corresponding $X$-rooted subclass can be calculated via equation \eqref{eq:prod}. Therefore, using equation \eqref{eq:reccount}, $\Size(\tG)$ can be obtained recursively: In each chain component, each variable is set as the root variable, and in each chain component of each of the resulting rooted essential graphs, each variables is set as the root, and this procedure is repeated until the resulting essential graph is a directed graph and has no chain components. 

Note that in this procedure, after setting each variable as the root, we observe the directions that the edges in the rooted essential graph acquire. That is, it has the property that we explicitly monitor the performed orientations in the given essential graph. The approach that we present in the following for calculating and estimating the value of $\mathcal{D}(\I)$ requires this property. 
Therefore, methods for calculation of the size of the MEC which are based on explicit functions of the parameters of the structure cannot be used in our approach. For instance, \cite{he2015counting} showed that there are five types of MECs whose sizes can be formulated as functions of the number of vertices; e.g., for a tree UCEG of order $p$, the size of the MEC is $p-1$.
An efficient counting approach with our desired property of monitoring the performed orientations is proposed by \cite{ghassami2018counting}, where the counting is performed based on the clique tree representation of the essential graph.

\begin{example}
\label{ex:excount}	
Assume the UCEG in Figure \ref{fig:excount}$(a)$ is the given essential graph.

Setting vertex $X_1$ as the root of $\tG$ (by symmetry, $X_4$ is similar), in the rooted essential graph $\tG^{X_1}$, the directed edges are $X_1\rightarrow X_2$, $X_1\rightarrow X_3$, $X_2\rightarrow X_4$, and $X_3\rightarrow X_4$. This rooted essential graph is shown in Figure \ref{fig:excount}$(b)$, which has a single chain component $\tG'$, (Figure \ref{fig:excount}$(c)$). Setting vertex $X_2$ as the root of $\tG'$ (by symmetry, $X_3$ is similar), in the rooted essential graph $\tG'^{X_2}$, the directed edge is $X_2\rightarrow X_3$. This results in a directed graph, thus, $\Size(\tG'^{X_2})=1$. Similarly, $\Size(\tG'^{X_3})=1$. Therefore, using equation \eqref{eq:sum}, we have $\Size(\tG^{X_1})=\Size(\tG'^{X_2})+\Size(\tG'^{X_3})=2$. Similarly, we have $Size(\tG^{X_4})=2$.

Setting vertex $X_2$ as the root of $\tG$ (by symmetry, $X_3$ is similar), in the rooted essential graph $\tG^{X_2}$, the directed edges are $X_2\rightarrow X_1$, $X_2\rightarrow X_3$, and $X_2\rightarrow X_4$. This rooted essential graph is shown in Figure \ref{fig:excount}$(d)$, which has a single chain component $\tG''$, (Figure \ref{fig:excount}$(e)$). Setting vertex $X_1$ as the root of $\tG''$, in the rooted essential graph $\tG''^{X_1}$, the directed edges are $X_1\rightarrow X_3$ and $X_3\rightarrow X_4$. This results in a directed graph, thus, $\Size(\tG''^{X_1})=1$. Similarly, $\Size(\tG''^{X_3})=1$ and $\Size(\tG''^{X_4})=1$. Therefore, using equation \eqref{eq:sum}, we have $\Size(\tG^{X_2})=\Size(\tG''^{X_1})+\Size(\tG''^{X_3})+\Size(\tG''^{X_4})=3$. Similarly, we have $\Size(\tG^{X_3})=3$.
 
Finally, using equation \eqref{eq:sum}, we obtain that $\Size(\tG)=\sum_i \Size(\tG^{X_i})=10$.
\end{example}

Now consider the task of counting the number of elements of a $\MEC(\tG)$ in the presence of prior knowledge regarding the direction of a subset of the undirected edges of the essential graph. We present the available prior knowledge in the form of a hypothesis graph $H=(V(H),E(H))$, which is the same as $\tG$, yet the orientation of the edges corresponding to the prior knowledge are determined as well.
For essential graph $\tG$, let $\Size_H(\tG)$ denote the number of the elements of $\MEC(\tG)$, which are consistent with hypothesis $H$, i.e., $\Size_H(\tG)=|\{G:G\in \MEC(\tG),E(G)\subseteq E(H)\}|$. 
Similar to equation \eqref{eq:prod}, we have
$\textit{Size}_H(\tG)=\prod_{r=1}^R\Size_H(\tG_r)$.
Also, akin to equation \eqref{eq:sum}, for chain component of $\tG$, we have
$\Size_H(\tG_r)=\sum_{X\in V(\tG_r)}\Size_H(\tG_r^X)$.
Therefore, in order to extend the counting approach to the case of having prior knowledge, every time that a variable is chosen as the root of a UCEG, we check if the resulting oriented edges belong to $E(H)$. If this is not the case, for $X$-rooted essential graph $\tG^X$, we return $\Size(\tG^X)=0$. This guarantees that any DAG considered in the counting will be consistent with the hypothesis $H$. See Algorithm \ref{algorithm:counting} for a pseudo-code of the proposed counting approach with prior knowledge.
If $H=\tG$ it implies that we have no prior knowledge, and the algorithm outputs $\Size(\tG)$.
 Note that the ability of checking the consistency of the oriented edges with the hypothesis is the reason that we stated earlier that the property of monitoring the performed orientations in the given essential graph is required in our approach.

\begin{algorithm}[t]
\begin{algorithmic}
 \STATE {\bf input} Essential graph $\tG$, Hypothesis graph $H$.
 \STATE {\bf output} $\Counter(\tG,H)$\\
  \hrulefill
 \STATE {\bf function} $\Counter(\tG,H)$:
 \IF{$\tG$ is a directed graph} 
 \RETURN $1$.
 \ELSE
\FOR{each chain component $\tG_r$ of $\tG$}
\FOR{$X\in V(\tG_r)$}
\STATE {\bf if} $E(\tG^X_r)\subseteq E(H)$
{\bf then} $\Size(\tG^X_r)=\Counter(\tG^X_r,H)$
{\bf else} $\Size(\tG^X_r)=0$
{\bf end if}
\ENDFOR
\STATE $\Size(\tG_r)=\sum_X\Size(\tG^X_r)$
\ENDFOR
\RETURN $\prod_r\Size(\tG_r)$
\ENDIF
 \caption{Counting with Prior Knowledge}
 \label{algorithm:counting}
\end{algorithmic}
\end{algorithm}

\begin{proposition}
\label{prop:countcomp}
For a given essential graph with maximum vertex degree $\Delta$, the computational complexity of Algorithm \ref{algorithm:counting} is $\mathcal{O}(p^{\Delta+2})$.
\end{proposition}

We now demonstrate how the approach of counting with prior knowledge can be utilized for the task of calculating $\textit{D}(\I)$.
Recall that for an experiment target set $\I$ and DAG $G_i\in\MEC(G^*)$, the set $R(\I,G_i)$, i.e., the set of edges directed in $\tG^{(\I)}_i$ but not directed in $\tG^*$, only depends on the $\I$-MEC that $G_i$ belongs to. Also, recall that the $\I$-MEC that $G_i$ belongs to only depends on $A(\I,G_i)$, which is the directed edges in $G_i$ incident to vertices in $\I$.
Therefore, all DAGs $G\in\MEC(G^*)$ that have the same set $A(\I,G)$ lead to the same value for $D(\I,G)$. Therefore,  one can partition the members of $\MEC(G^*)$ with respect to their set $A(\I,G)$, and then, consider the set $A(\I,G)$ as prior knowledge and use the aforementioned counting approach to count the number of DAGs in each partition of $\MEC(G^*)$. 

Formally, let $\mathcal{H}$ be the set of hypothesis graphs, in which each element $H$ has a distinct configuration for $A(\I,G)$. If the maximum degree of the graph is $\Delta$, cardinality of $\mathcal{H}$ is at most $2^{k\Delta}$, and hence, it does not grow with $p$.
For a given hypothesis graph $H$, let $\tG_H=\{G: G\in MEC(G^*), E(G)\subseteq E(H)\}$ denote the set of members of the $\MEC(G^*)$, which are consistent with hypothesis $H$. Note that this set is in fact an interventional MEC. Using the set $\mathcal{H}$, we can write the expression of $\mathcal{D}(\I)$ as follows.
\begin{equation}
\label{eq:exact}	
\begin{aligned}	
\mathcal{D}(\I)
	&=\frac{1}{\Size(\tG^*)}\sum_{G_i\in\MEC(G^*)}D(\I,G_i)\\
	&=\frac{1}{\Size(\tG^*)}\sum_{H\in\mathcal{H}}\sum_{G_i\in \tG_H}D(\I,G_i)\\
	&=\sum_{H\in\mathcal{H}}\frac{\Size_H(\tG^*)}{\Size(\tG^*)}D(\I,G_i),
\end{aligned}
\end{equation}
where in the last summation, $G_i\in\tG_H$.
Therefore, we only need to calculate at most $2^{k\Delta}$ values instead of considering all elements of $\MEC(G^*)$, which reduces the complexity from super-exponential to constant in $p$.

Eventually, in order to design the experiment, we use the proposed calculator of $\mathcal{D}$ in a greedy algorithm. We term this approach the \textit{Greedy Intervention Design} (GrID).

\subsection{Unbiased $\mathcal{D}(\I)$ Estimator}
\label{sec:D(I)est}

The computational complexity of the approach presented in Subsection \ref{sec:D(I)calc} for exact calculation of $\mathcal{D}(\I)$ is exponential in the intervention budget $k$. Hence, it may not be computationally tractable for large values of $k$. For this scenario, we propose running Monte-Carlo simulations of the intervention model for sufficiently large number of times to obtain an accurate estimation of $\mathcal{D}(\I)$. To this end, we need a uniform sampler for generating random DAGs from $\MEC(G^*)$. We present such a sampler, which is based on the counting method presented in Subsection \ref{sec:D(I)calc}. The main idea is that in a UCEG, we choose a vertex as the root according to the portion of members of the corresponding MEC which have that vertex as the root, i.e., in UCEG $\tG$, vertex $X$ should be picked as the root with probability $\Size(\tG^X)/ \Size(\tG)$. The pseudo-code of the proposed sampler is presented in function $\textsc{UnifSamp}$ in Algorithm \ref{algorithm:unifsamp}, in which we use function $\Counter$ from Algorithm \ref{algorithm:counting}.
 
\begin{algorithm}[t]
\begin{algorithmic}
\STATE {\bf input:} Essential graph $\tG$ with chain components $\{\tG_1,..., \tG_R\}$, target set $\mathcal{I}$, and $N$.
\STATE {\bf initialize:}  $\widehat{\MEC}=\emptyset$
\FOR{$i=1$ to $N$}
\STATE Generate sample DAG $G_i=\textsc{UnifSamp}(\tG)$
\STATE $\widehat{\MEC}=\widehat{\MEC}\uplus G_i$
\ENDFOR
\STATE {\bf output:} $\hat{\mathcal{D}}(\I)=\frac{1}{N}\sum_{G_i\in\widehat{\MEC}}D(\I,G_i)$\\
\hrulefill
\STATE {\bf function} $\textsc{UnifSamp}(\tG)$

\STATE {\bf initialize:} $\mathcal{G}=\{\tG_1,..., \tG_R\}$
\WHILE{$\mathcal{G}\neq \emptyset$}
\STATE Pick an element $\tG_r\in\mathcal{G}$, and update $\mathcal{G}=\mathcal{G}\setminus \tG_r$.
\STATE Set $X\in V(\tG_r)$ as the root with probability $\frac{\Counter(\tG^X_r,\tG^X_r)}{\Counter(\tG_r,\tG_r)}$.
\STATE Add the directed edges of $\tG^X_r$ to $\tG$
\STATE $\mathcal{G}=\mathcal{G}\cup\{\text{chain components of }\tG^X_r\}$
\ENDWHILE
\RETURN $\tG$.
\caption{Unbiased $\mathcal{D}(\I)$ Estimator}
\label{algorithm:unifsamp}
\end{algorithmic}
\end{algorithm}


\begin{sloppy}
\begin{example}
For the UCEG in Figure \ref{fig:excount}$(a)$, as observed in Example \ref{ex:excount}, $\Size(\tG^{X_1})=\Size(\tG^{X_4})=2$,  $\Size(\tG^{X_2})=\Size(\tG^{X_3})=3$, and hence, $\Size(\tG)=10$. Therefore, we set vertices $X_1$, $X_2$, $X_3$, and $X_4$ as the root with probabilities $2/10$, $3/10$, $3/10$, and $2/10$, respectively. Suppose $X_2$ is chosen as the root. Then as seen in Example \ref{ex:excount}, $\Size(G''^{X_1})=\Size(G''^{X_3})=\Size(G''^{X_4})=1$. Therefore, in $G''$, we set either of the vertices as the root with equal probability to obtain the final DAG.
\end{example}
\end{sloppy}
\begin{theorem}
\label{thm:unif}
The sampler in Algorithm \ref{algorithm:unifsamp} is uniform. 
\end{theorem}

As a corollary of Proposition \ref{prop:countcomp}, for bounded degree graphs, the proposed sampler runs in polynomial time.
\begin{corollary}
\label{cor:sampcomp}
For a given essential graph with maximum vertex degree $\Delta$, the computational complexity of the uniform sampler in Algorithm \ref{algorithm:unifsamp} is $\mathcal{O}(p^{\Delta+2})$.
\end{corollary}

Equipped with the uniform sampler in Algorithm \ref{algorithm:unifsamp}, in order to estimate the value of $\mathcal{D}(\I)$, we generate $N$ DAGs from $\MEC(G^*)$. The generated DAGs are kept in a multiset $\widehat{\MEC}$, in which repetition is allowed. Finally, we calculate the estimated value $\hat{\mathcal{D}}(\mathcal{I})$ on $\widehat{\MEC}$ instead of $\MEC(G^*)$ as follows.
\[
\hat{\mathcal{D}}(\I)=\frac{1}{|\widehat{\MEC}~|}\sum_{G_i\in\widehat{\MEC}}D(\I,G_i).
\] 
The pseudo-code of our estimator is presented in Algorithm \ref{algorithm:unifsamp}. In the pseudo-code, operator $\uplus$ indicates the multiset addition.

The estimation obtained from the aforementioned approach is an unbiased estimation of $\mathcal{D}(\I)$, i.e., $\mathbb{E}[\hat{\mathcal{D}}(\mathcal{I})]=\mathcal{D}(\mathcal{I})$.
To show the unbiasedness, suppose $G_i$ is a random generated DAG in the uniform sampler. We have
\begin{align*}
\mathbb{E}[\hat{\mathcal{D}}(\I)]&=\frac{1}{N}\sum_{G_i\in\widehat{\MEC}}\mathbb{E}[D(\I,G_i)]\\
&=\frac{1}{N}\cdot N\sum_{G_i'\in\MEC(G^*)}P(G_i=G'_i)D(\I,G'_i)\\
&=\frac{1}{|\MEC(G^*)|}\sum_{G_i'\in\MEC(G^*)}D(\I,G'_i)=\mathcal{D}(\mathcal{I}).
\end{align*}

Eventually, in order to design the experiment, we use the estimator $\hat{\mathcal{D}}$ in a greedy algorithm. We term this approach the \textit{Random Greedy Intervention Design} (Ran-GrID).

We generated 100 random UCEGs of order $p\in\{10,20,30\}$, with $r\times{p \choose 2}$ edges, where parameter $0\le r\le 1$ controls the graph density. For this experiment, we picked $r=0.2$. In each graph, we selected two variables randomly to intervene on. We obtained the exact $\mathcal{D}(\I)$ using equation \eqref{eq:exact}. Furthermore, for a given sample size $N$, we estimated $\mathcal{D}(\I)$ using Algorithm \ref{algorithm:unifsamp} and obtained empirical standard deviation of the normalized error (SDNE) over all graphs with the same size, defined as $SD(|\mathcal{D}(\I)-\hat{\mathcal{D}}(\I)|/\mathcal{D}(\I))$. Figure \ref{fig:sim} depicts SDNE versus the number of samples. As can be seen, SDNE becomes fairly low for sample sizes greater than $40$.
Next, we formalize our observation regarding convergence and consider the required cardinality of the set $\widehat{\MEC}$ to obtain a desired accuracy in estimating $\mathcal{D}(\mathcal{I})$. We use Chernoff bound for this purpose. 

\begin{figure}[t]
\begin{center}
\includegraphics[scale=0.4]{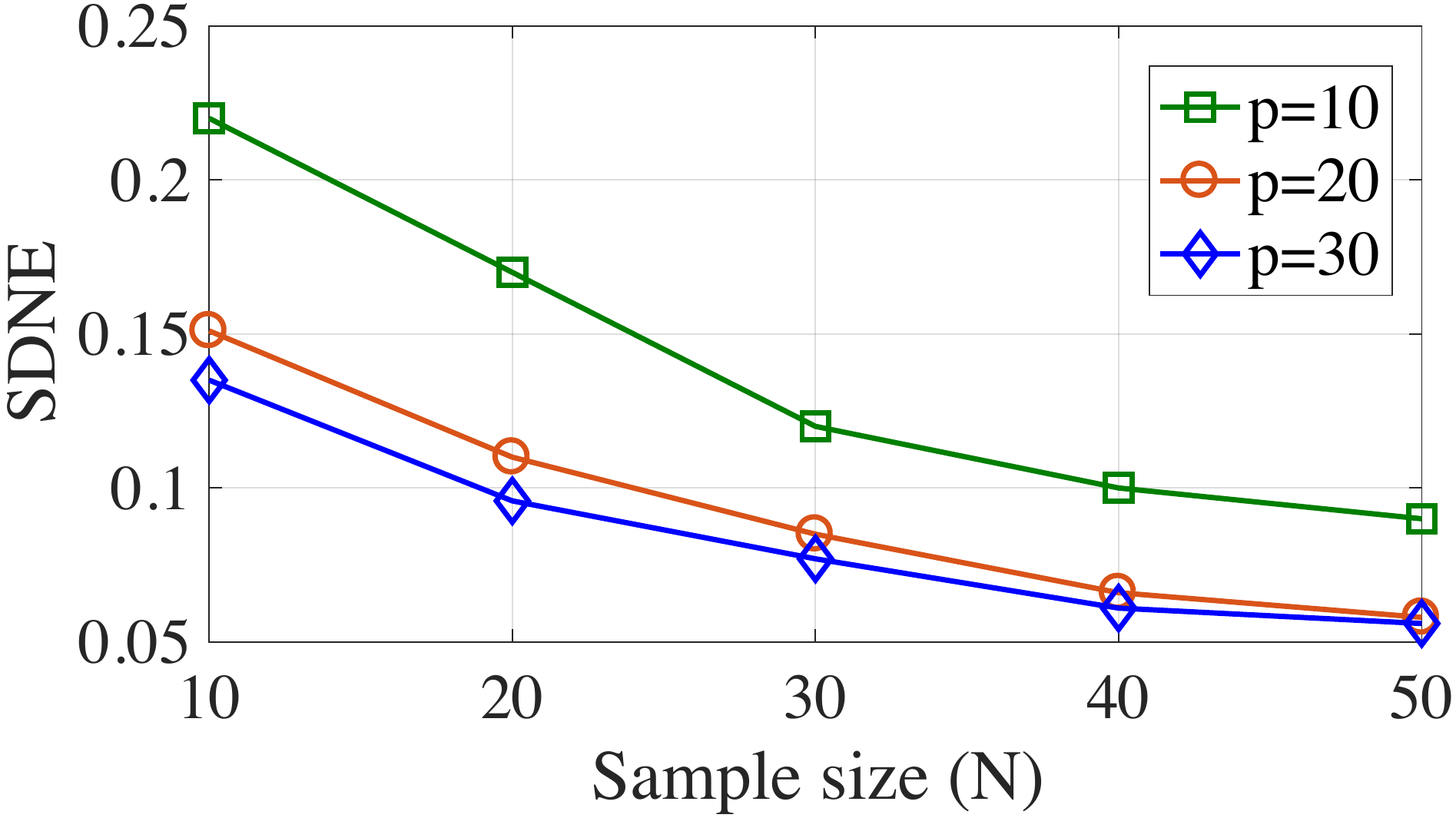}
\caption{Standard deviation of the normalized error versus the sample size.}
\label{fig:sim}
\end{center}
\end{figure}

\begin{theorem}
\label{thm:uconv}
Let $\bar{A}(\tG)$ denote the set of undirected edges of $\tG$.
For the estimator in Algorithm \ref{algorithm:unifsamp}, given experiment target set $\I$ and $\epsilon,\delta>0$, if $N=|\widehat{\MEC}~|>\frac{|\bar{A}(\tG)|(2+\epsilon)}{\epsilon^2}\ln(\frac{2}{\delta})$, then
\[
\mathcal{D}(\mathcal{I})(1-\epsilon)<\hat{\mathcal{D}}(\mathcal{I})<\mathcal{D}(\mathcal{I})(1+\epsilon),
\]
with probability larger than $1-\delta$.
\end{theorem}

For any $\epsilon'>0$, for sufficiently large sample size, the Ran-GrID method provides us with a $(1-\frac{1}{e}-\epsilon')$-approximation of the optimal value with high probability, as formalized in the following theorem.

\begin{theorem}
\label{thm:app}
For any $\epsilon',\delta'>0$, 
let $\epsilon=\frac{\epsilon'}{4k}$ and $\delta=\frac{\delta'}{4k^2}$.
If for any experiment target set $\I$, $\mathcal{D}(\I)(1-\epsilon)<\hat{\mathcal{D}}(\I)<\mathcal{D}(\I)(1+\epsilon)$ with probability larger than $1-\delta$, then Algorithm \ref{algorithm:GG} is a $(1-\frac{1}{e}-\epsilon')$-approximation algorithm with probability larger than $1-\delta'$.
\end{theorem}

\subsection{Fast $\mathcal{D}(\I)$ Estimator}
\label{subsec:fast}

\begin{algorithm}[t]
\begin{algorithmic}
\STATE {\bf input:} Essential graph $\tG$ with chain components $\{\tG_1,..., \tG_R\}$, target set $\mathcal{I}$, and $N$.
\STATE {\bf initialize:}  $\widehat{\MEC}=\emptyset$
\FOR{$i=1$ to $N$}
\STATE Generate sample DAG $G_i=\textsc{FastSamp}(\tG)$
\STATE $\widehat{\MEC}=\widehat{\MEC}\uplus G_i$
\ENDFOR
\STATE {\bf output:} $\hat{\mathcal{D}}(\I)=\frac{1}{N}\sum_{G_i\in\widehat{\MEC}}D(\I,G_i)$\\
\hrulefill
\STATE {\bf function} $\textsc{FastSamp}(\tG)$

\STATE Uniformly shuffle the order of the elements of $V(\tG)$.
\WHILE{the induced subgraph on any subset of size 3 of the variables is not directed, or a directed cycle, or a v-structure which was not in $\tG$}
\FOR{all $\{X_i,X_j,X_k\}\subseteq V(\tG)$}
\STATE Orient the undirected edges among $\{X_i,X_j,X_k\}$ independently according to $\textit{Bern}(\frac{1}{2})$ until it becomes a directed structure which is not a directed cycle or a v-structure which was not in $\tG$.
\ENDFOR
\ENDWHILE
\RETURN $\tG$.
\caption{Fast $\mathcal{D}(\I)$ Estimator}
\label{algorithm:fastsamp}
\end{algorithmic}
\end{algorithm}

Recall that the computational complexity of the unifrom sampler in Algorithm \ref{algorithm:unifsamp} is $\mathcal{O}(p^{\Delta+2})$, which will be intractable when the input graph has many vertices with large degrees. In this subsection, we propose another sampler, which is more suitable for graphs with large maximum degree. Although this sampler is not uniform, our extensive experimental results confirm that its sampling distribution is very close to uniform. We use this sampler in an estimator for $\mathcal{D}(\I)$ similar to the one in Algorithm \ref{algorithm:unifsamp}.

The pseudo-code of the proposed estimator is presented in Algorithm \ref{algorithm:fastsamp}. In this estimator, for the given essential graph $\tG$, we generate $N$ DAGs  from the MEC of $G^*$ as follows:
We consider all subsets of size 3 from $V(\tG)$ in a uniformly random order (achieved by uniformly shuffling the labels of elements of $V$). For each subset $\{X_i, X_j, X_k\}$, we orient the undirected edges among $\{X_i, X_j, X_k\}$ independently according to a Bernoulli$({1}/{2})$ distribution. If the resulting orientation on the induced subgraph on $\{X_i, X_j, X_k\}$ is a directed cycle or a new v-structure, which was not in $\tG$, we redo the orienting. We keep checking all the subsets of size 3 until the induced subgraph on all of them are directed and none of them is a new v-structure, which did not exist in $\tG$, or a directed cycle. 

\begin{proposition}
\label{prop:inMEC}
Each generated DAG $G_i$ in the sampler $\textsc{FastSamp}$ in Algorithm \ref{algorithm:fastsamp} belongs to the Markov equivalence class of $G^*$.
\end{proposition}

We generated $100$ random UCEGs of order $p\in\{20, 30,... ,60\}$ with $r\times {p \choose 2}$ edges, where parameter $0\le r\le1$ controls the graph density. 
Table \ref{table:comp} shows a comparison between the run time of the fast sampler in Algorithm \ref{algorithm:fastsamp}, denoted by $T_f$, compared to the run time of the uniform sampler in Algorithm \ref{algorithm:unifsamp}, denoted by $T_u$, for random essential graphs with different orders. 
As can be seen, 
the run time ratio $T_u/T_f$ increases as the order of the the graphs increases.

\begin{table}[t]
\begin{center}
  \begin{tabular}{ |c | c | c c c c c c |}
    \hline
 & & $p:$  & 20 & 30 & 40 & 50 & 60 \\ \hline
 & $T_u$ & &  0.50 & 2.26 & 6.65 & 19.55  & 55.59  \\
$r=0.2$ & $T_f$ & & 0.018 & 0.055  &  0.163  &   0.3 & 0.63\\ 
 & $T_u/T_f$ & & 28.41 &  41.09 &  40.67 &  65.17 &   88.24     \\ \hline
 & $T_u$ & & 0.51 & 2.27 & 7.56 & 25.46 & 59.21 \\ 
$r=0.25$ & $T_f$ & &  0.0218 & 0.06  &  0.1686   &   0.35 &  0.66\\
  &  $T_u/T_f$ & &  23.40 &   37.83 &   44.84  &   72.74  &   89.71 \\
    \hline
  \end{tabular}
  \caption{Average run time (in seconds) for the uniform sampler and the fast sampler.}
\label{table:comp}
  \end{center}
\end{table}



\section{Improved Greedy Algorithm}
\label{sec:algorithm}


 We exploit the submodularity of function $\mathcal{D}$ to implement an accelerated variant of the General Greedy Algorithm through {\it lazy} evaluations, originally proposed by \cite{minoux1978accelerated}.\footnote{There are improved versions of this algorithm in the literature \cite{mirzasoleiman2015lazier}.} 
In each round of the General Greedy Algorithm, we check the marginal gain $\Delta_X(\I)$ for all remaining vertices in $V\backslash \mathcal{I}$. Note that as a consequence of submodularity of function $\mathcal{D}$, the set function $\Delta_X$ is monotonically decreasing.
The main idea of the Improved Greedy Algorithm is to take advantage of this property to avoid checking all the variables in each round of the algorithm. More specifically, suppose for vertices $X_1$ and $X_2$, in the $i$-th round of the algorithm we have obtained marginal gains $\Delta_{X_1}(\I_i)>\Delta_{X_2}(\I_i)$. If in the $(i+1)$-th round, we calculate $\Delta_{X_1}(\I_{i+1})$ and observe that $\Delta_{X_1}(\I_{i+1})>\Delta_{X_2}(\I_i)$, from monotonic decreasing property of function $\Delta_X$, we can conclude that $\Delta_{X_1}(\I_{i+1})>\Delta_{X_2}(\I_{i+1})$, and hence, there is no need to calculate $\Delta_{X_2}(\I_{i+1})$.

Improved Greedy Algorithm is presented in Algorithm \ref{algorithm:IG}. The idea can be formalized as follows: We define a profit parameter $\textit{pro}_X$ for each variable $X$ and initialize the value for all variables with $\infty$. Moreover, we define an update flag $\textit{upd}_X$ for all variables, which will be set to \texttt{false} at the beginning of every round of the algorithm, and will be switched to \texttt{true} if we update $\textit{pro}_X$ with the value of the marginal gain of vertex $X$.
In each round, the algorithm picks vertex $X\in V\backslash \I$ with the largest profit, updates its profit with the value of the marginal gain of $X$, and sets $\textit{upd}_X$ to \texttt{true}. This process is repeated until the vertex with the largest profit is already updated, i.e., its update flag is \texttt{true}. Then we add this vertex to $\I$ and end the round.
For example, if in a round, the vertex $X$ has the highest profit and after updating the profit of this vertex, $\textit{pro}_X$ is still larger than all the other profits, we do not need to evaluate the marginal gain of any other vertex and we add $X$ to $\I$.

The correctness of the Improved Greedy Algorithm follows directly from submodularity of function $\mathcal{D}$. Theorem \ref{thm:app} holds for Algorithm \ref{algorithm:IG} as well, that is, for any $\epsilon'>0$, Improved Greedy Algorithm provides us with a $(1-\frac{1}{e}-\epsilon')$-approximation of the optimal value.
This algorithm can lead to orders of magnitude performance speedup, as shown by \cite{leskovec2007cost}.

\begin{algorithm}[t]
\begin{algorithmic}
 \STATE {\bf input:} Essential graph from the observational stage, budget $k$.
\STATE {\bf initialize:} $\I_0=\emptyset$, and $\textit{pro}_X=\infty$, $\forall X\in V$.
\FOR{$i=1$ to $k$}
\STATE $\textit{upd}_X=$ \texttt{false}, $\forall X \in V\backslash \mathcal{I}_{i-1}$
\WHILE{\texttt{true}}
\STATE $X^* =\arg\max_{X\in V\backslash \mathcal{I}_{i-1}} \textit{pro}_X$
\IF{$\textit{upd}_{X^*}$}
\STATE $\I_i=\I_{i-1}\cup \{X^*\}$
\STATE {\bf break};
\ELSE
\STATE $\textit{pro}_{X^*}= \mathcal{D}(\I_{i-1}\cup \{X^*\})-\mathcal{D}(\I_{i-1})$
\STATE $\textit{upd}_{X^*}=\texttt{true}$
\ENDIF
\ENDWHILE
\ENDFOR
\STATE {\bf output:} $\hat{\I}=\I_k$
 \caption{Improved Greedy Algorithm}
 \label{algorithm:IG}
\end{algorithmic}
\end{algorithm}


\section{Evaluation Results}
\label{sec:exp}

\subsection{Tree Structures}

We evaluated the performance of Algorithm \ref{algTrED} and the Ran-GrID approach on synthetic tree structures. As shown in Section \ref{sec:tree}, Algorithm \ref{algTrED} is optimum for the worst-case gain optimization problem. We observed that this algorithm also has a good performance on the average gain optimization problem. To see this, we generated random trees based on Barab\'asi-Albert model \citep{barabasi1999emergence,barabasi2016network}, and bounded degree model created according to Galton-Watson branching process \citep{barabasi2016network}. For both models we considered uniform and degree based distributions for the location of the root of the tree. In the degree based distribution, the probability of vertex $X$ being the root is proportional to its degree. Each generated tree was considered as a UCEG. 
 
 \begin{figure}[t]
	\begin{center}
		\centerline{\includegraphics[scale=0.62]{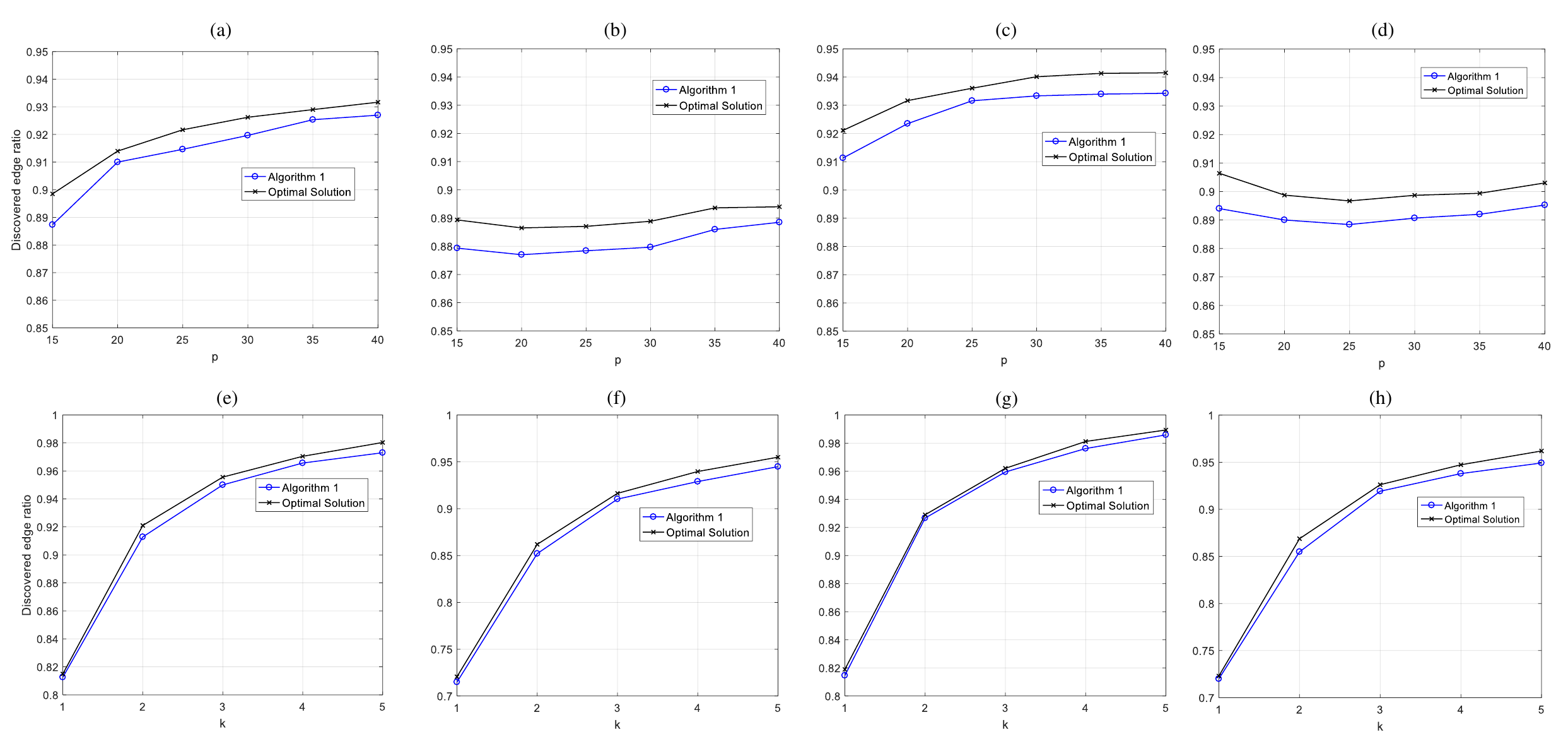}}
		\caption{The discovered edge ratio of Algorithm \ref{algTrED}, 
		and the optimal solution of average gain with respect to the order of the tree with $k=3$ (first row), and with respect to the intervention budget with $p=40$ (second row). In the first two columns, the probability of location of root has uniform distribution, while the degree based distribution is used in the simulations for the second two columns. The tree in parts (a), (c), (e) and (g) are created based on Barab\'asi-Albert model model and the bounded degree model is used in the rest.}
		\label{fig:tree}
	\end{center}
\end{figure}
 
 We considered an oracle experimental settings in evaluating the algorithm which can be seen as infinite sample case, in the absence of estimation errors. In particular, we assumed that the true essential graph is available as the input. Moreover, each intervention on a variable reveals the orientations of edges incident with that variable. As the performance measure, we consider the ratio of the number of edges whose directions are discovered as the result of interventions.
 
We generated 100 instances of random trees based on Barab\'asi-Albert model and bounded degree  model. Figure \ref{fig:tree} depicts the average discovered edge ratio of Algorithm 	\ref{algTrED}, 
and the optimal solution for the average gain case versus budget and graph order. As can be seen, in both models, the performance of the proposed algorithm
is close to the optimal solution.

\subsection{General Structures}

We evaluated the performance of the Ran-GrID algorithm for the case of general structures on synthetic and real graphs.
We compared the performance of Ran-GrID with two naive approaches: 1. Rand: Selecting experiment target set randomly, 2. MaxDeg: Sorting the list of variables based on the number of undirected edges connected to them in descending order and picking the first $k$ variables from the sorted list as the experiment target set. 
We studied the performance of the algorithms on two models of random graphs, namely, Erd\"{o}s-R\'{e}nyi graphs and random chordal graphs, described below:
\begin{itemize}
	\item Erd\"{o}s-R\'{e}nyi graphs:  In this model, we first generate the skeleton of the graph by drawing an edge between any pair of vertices with a predefined probability. Then, we construct a DAG over this skeleton based on a random permutation of vertices.
	\item Random chordal graphs: The essential graphs of DAGs constructed from Erd\"{o}s-R\'{e}nyi graphs might not have large chain components. Thus, we generate random chordal graphs and consider them as a UCEG. To do so, we use randomly chosen perfect elimination ordering (PEO)\footnote{A perfect elimination ordering $\{X_1, X_2, ...,  X_p\}$ on the vertices of an undirected chordal graph is such that for all $i$, the induced neighborhood of $X_i$ on the subgraph formed by $\{X_1, X_2, ..., X_{i-1}\}$ is a clique.} of the vertices to generate our underlying chordal graphs \citep{hauser2014two,shanmugam2015learning}. For each graph, we pick a random ordering of the vertices. Starting from the vertex $X$ with the highest order, we connect all the vertices with lower order to $X$ with probability inversely proportional to the order of $X$. Then, we connect all the parents of $X$ with directed edges, where each directed edge is oriented from the parent with the lower order to the parent with the higher order. In order to make sure that the generated graph will be connected, if vertex $X$ is not connected to any of the vertices with the lower order, we pick one of them uniformly at random and set it as the parent of $X$. 
\end{itemize}

We considered two experimental settings in evaluating the algorithms which we call \emph{oracle case} and \emph{sample case}. In the oracle case, which can be seen as infinite sample case, we execute algorithms in the absence of estimation errors. In particular, we assume that the true essential graph is available as the input. Moreover, each intervention on a variable reveals the orientations of edges incident with that variable. In the sample case, data is drawn based on a linear structural causal model with Gaussian exogenous variables. In this model, it is just needed to specify the weight of directed edges and variance of exogenous variables. Here, we drew edge weights from a uniform distribution in the range $[-1.5,-0.5]\cup [0.5,1.5]$ and exogenous variable variances from a uniform distribution in the range $[0.01,0.2]$. By intervening on a variable, we removed incoming edges to it and drew the samples of its exogenous variable from normal distribution $\mathcal{N}(2,0.2)$.

\begin{figure*}[t]
	\centering
	\begin{minipage}{.34\textwidth}
		\centering
		\includegraphics[width=.98\linewidth]{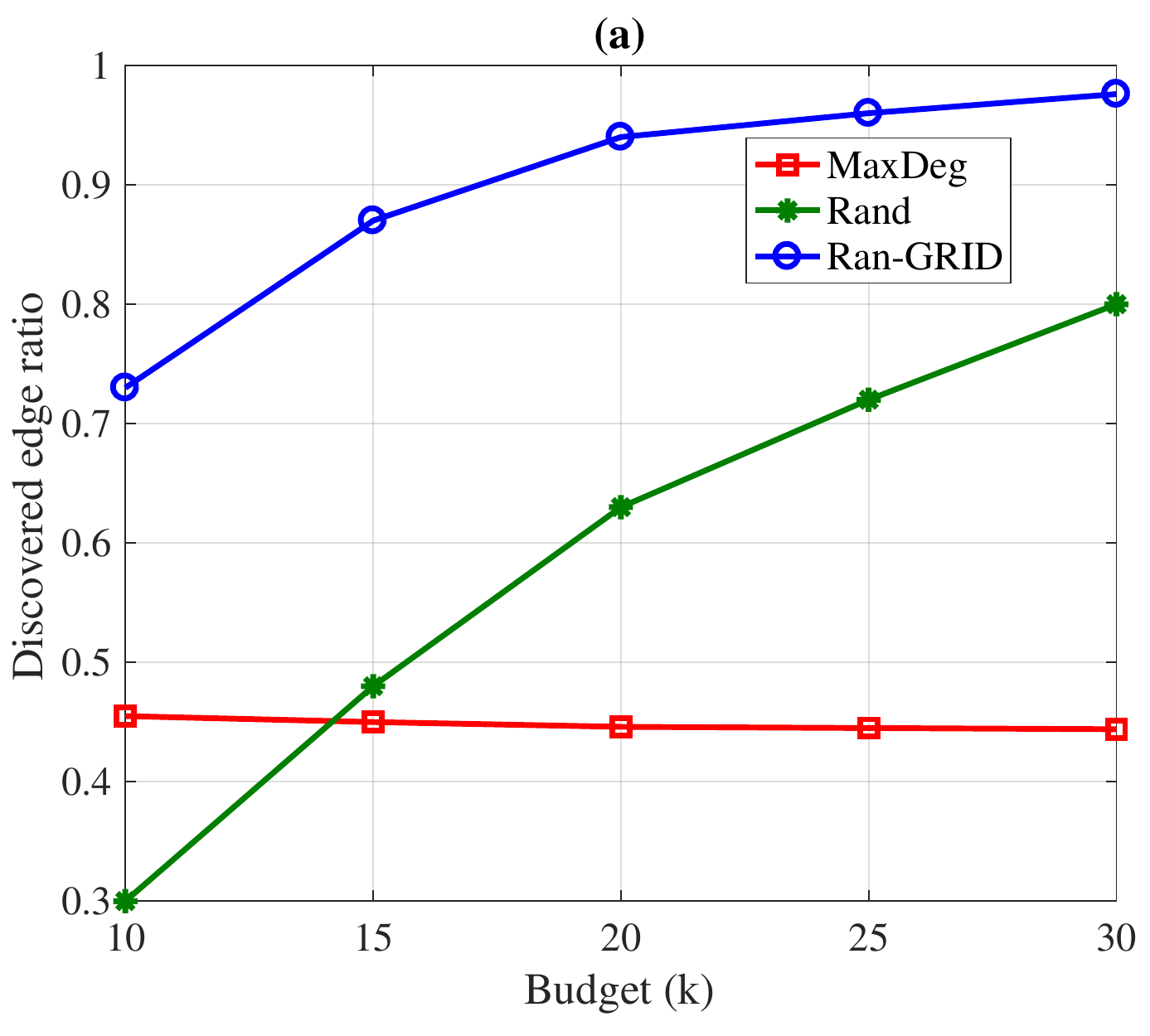}
	\end{minipage}%
	\begin{minipage}{.333\textwidth}
		\centering
		\includegraphics[width=0.95\linewidth]{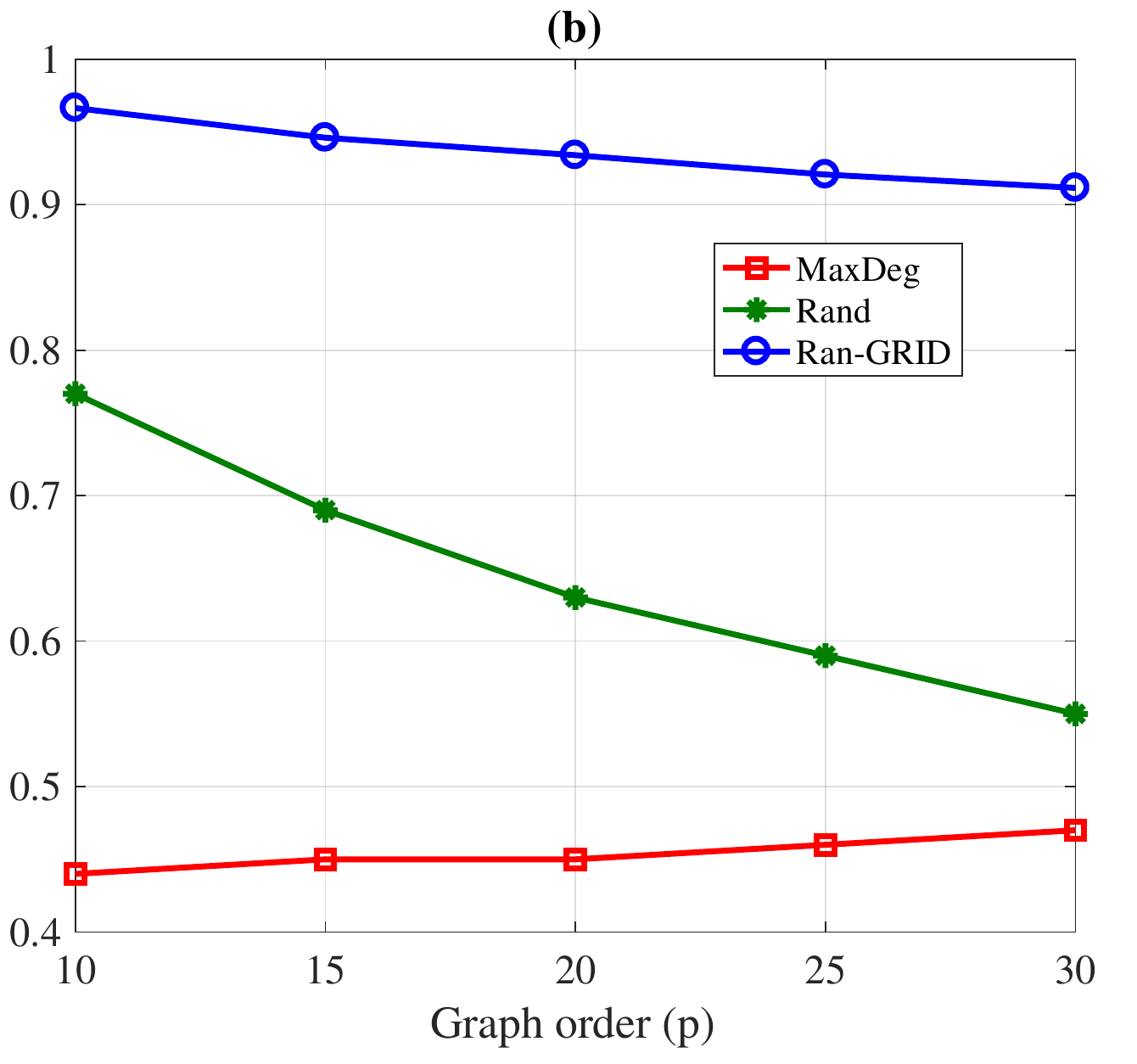}
	\end{minipage}%
	\begin{minipage}{.333\textwidth}
		\centering
		\includegraphics[width=.98\linewidth]{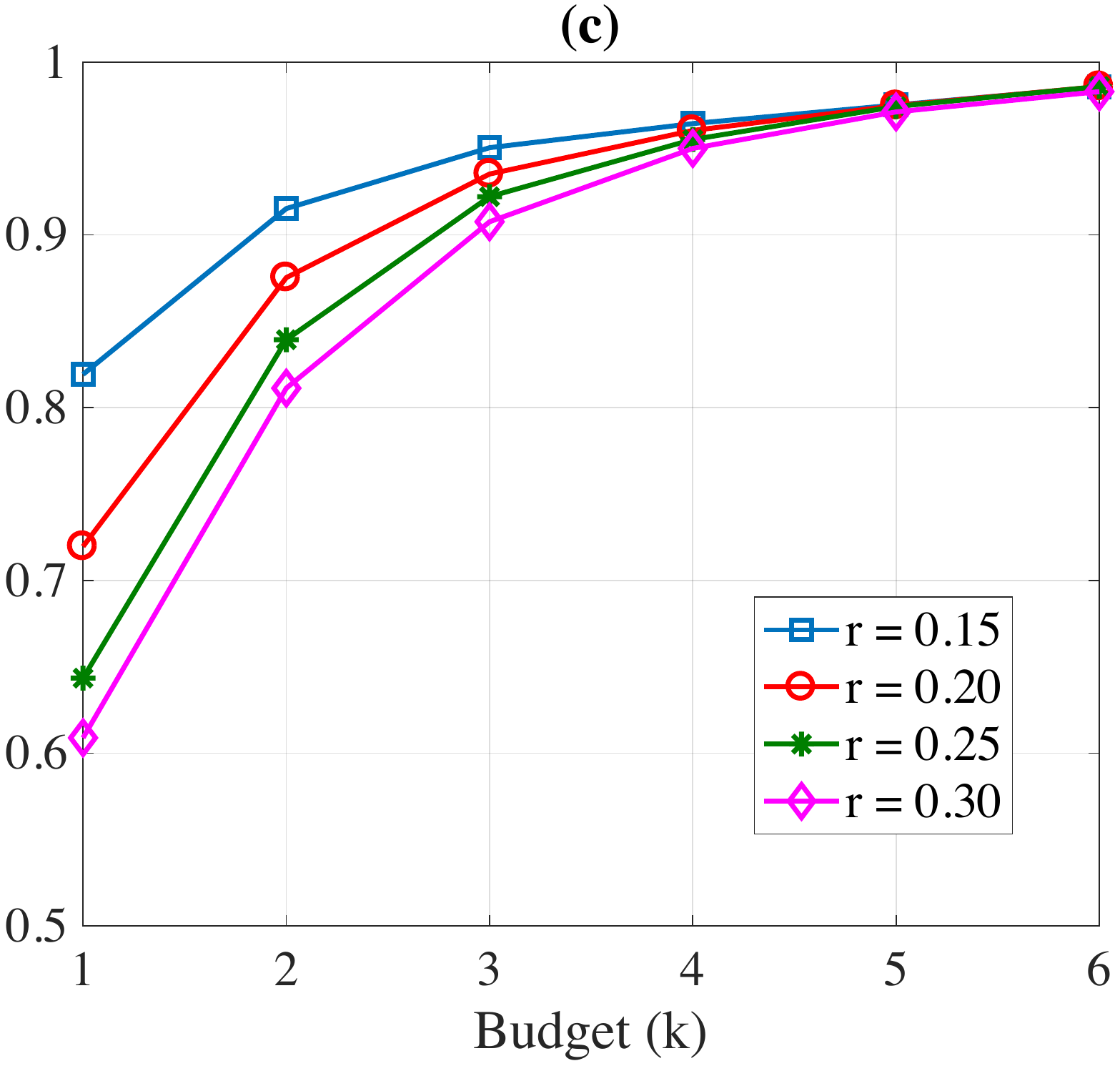}
	\end{minipage}
	\caption{Discovered edge ratio versus (a) budget for $p=20$, (b) graph orders for $k=3$, (c) budget for $p=20$ and different densities in the random chordal graphs.} 
	\vspace{-3mm}
	\label{fig:exps}
\end{figure*}

\begin{figure*}[t]
	\centering
	\begin{minipage}{.33\textwidth}
		\centering
		\includegraphics[width=.98\linewidth]{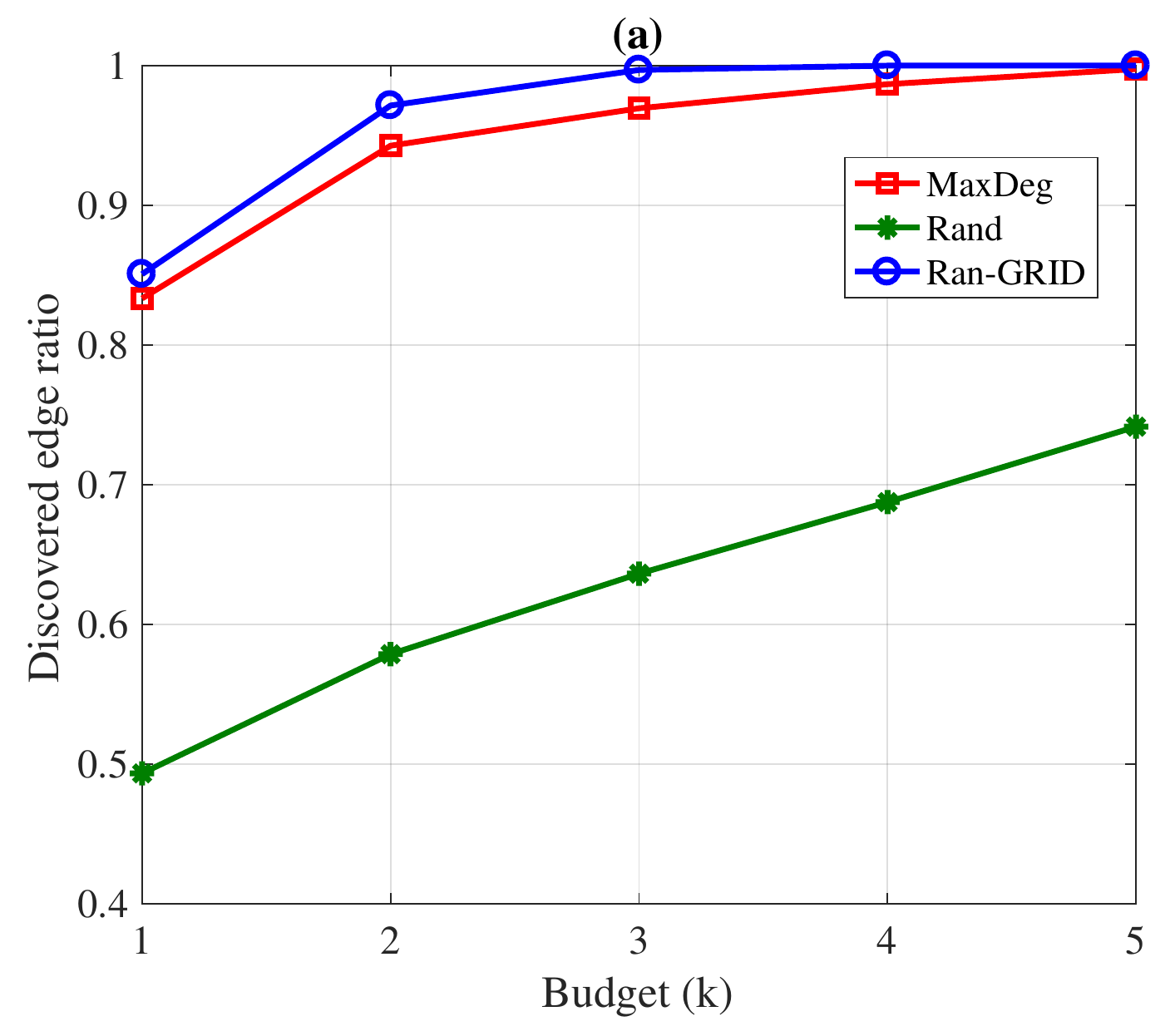}
	\end{minipage}%
	\begin{minipage}{.333\textwidth}
		\centering
		\includegraphics[width=.98\linewidth]{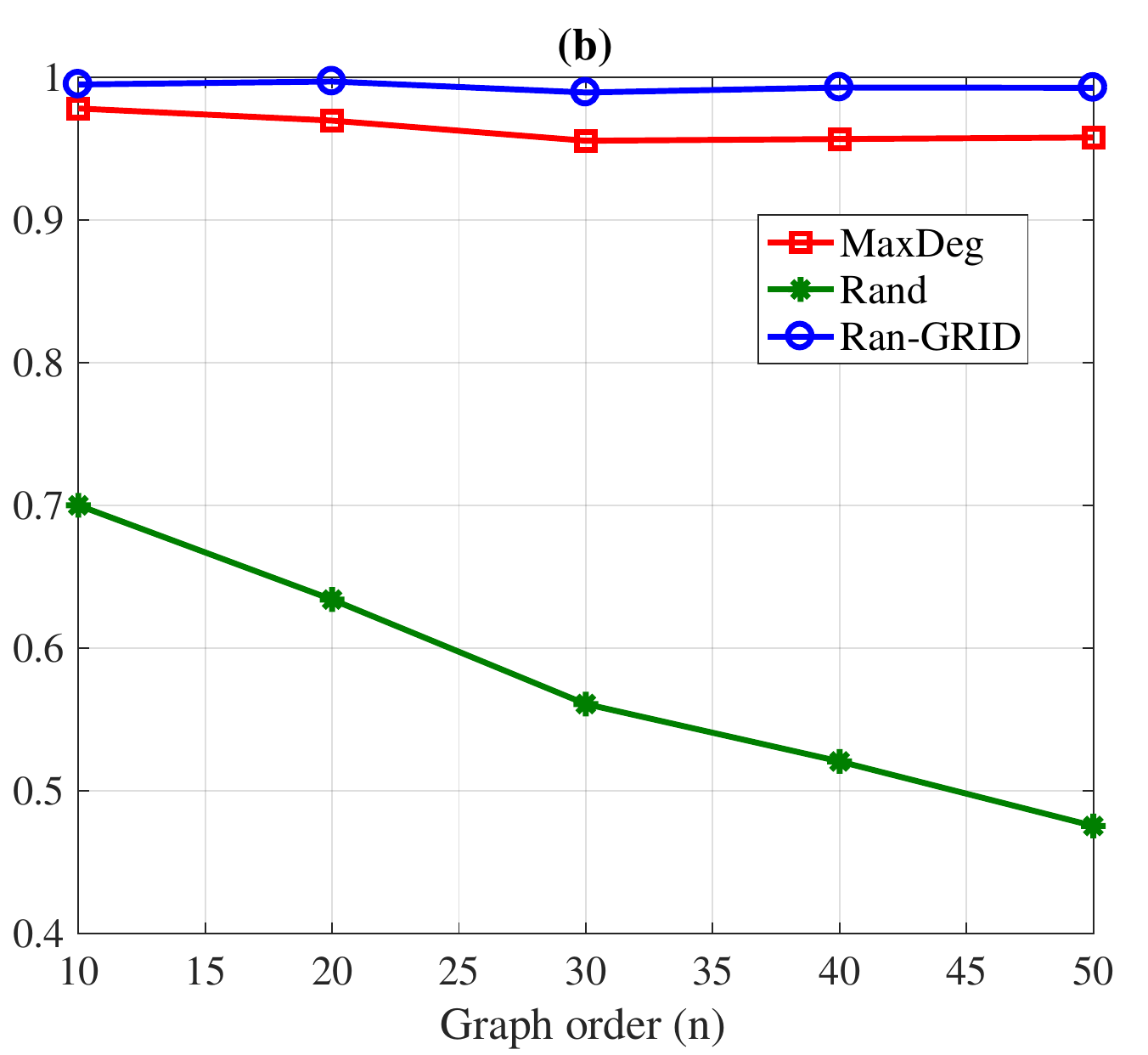}
	\end{minipage}%
	\begin{minipage}{.333\textwidth}
		\centering
		\includegraphics[width=1\linewidth]{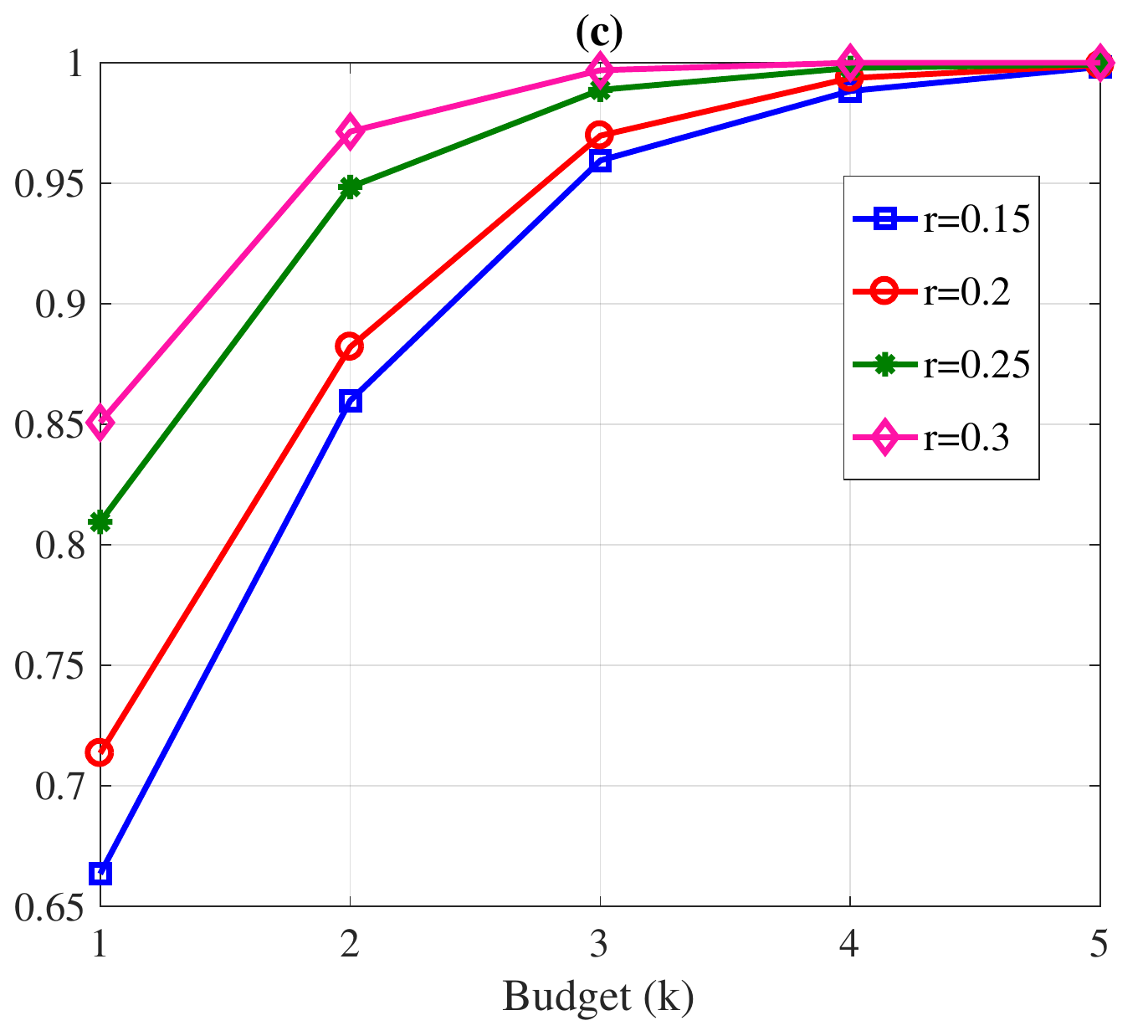}
	\end{minipage}
	\caption{Discovered edge ratio versus (a) budget for $p=20$, (b) graph orders for $k=3$, (c) budget for $p=20$ and different densities in Erd\"{o}s-R\'{e}nyi graphs.} 
	\vspace{-3mm}
	\label{fig:exps:erdos}
\end{figure*}

\begin{figure*}[t]
	\centering
	\begin{minipage}{.5\textwidth}
		\centering
		\includegraphics[width=1.02\linewidth]{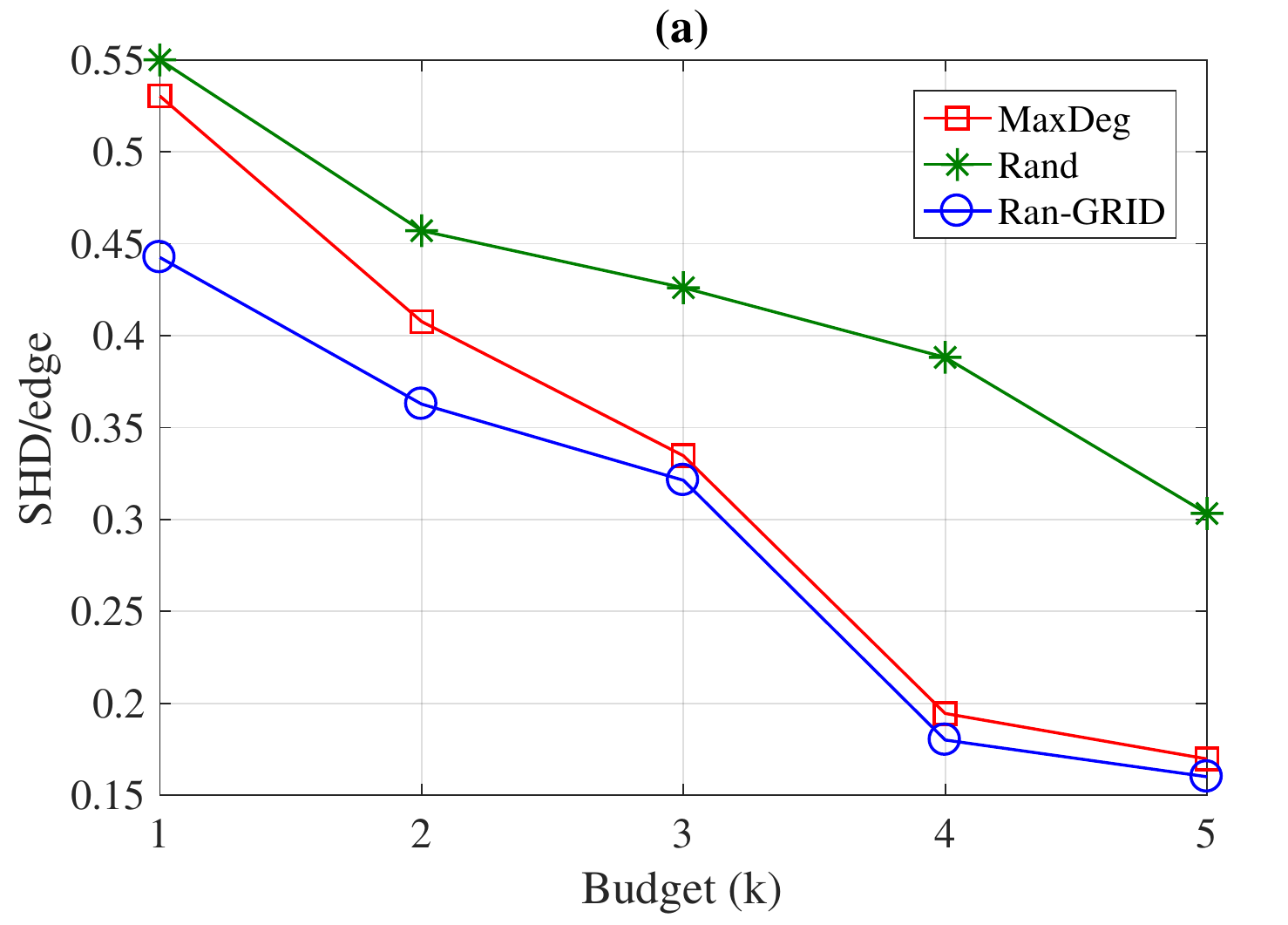}
	\end{minipage}%
	\begin{minipage}{.5\textwidth}
		\centering
		\includegraphics[width=.98\linewidth]{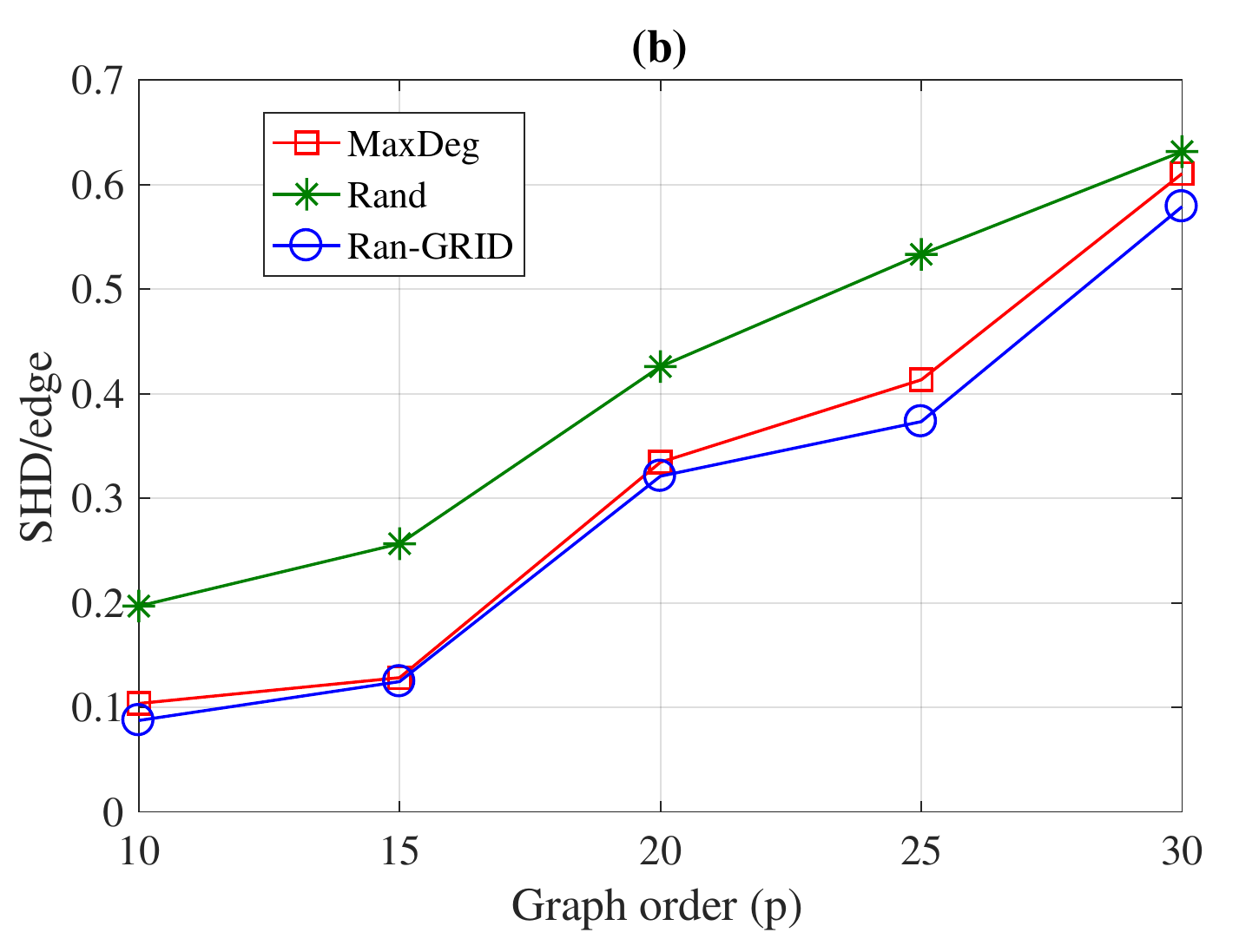}
	\end{minipage}%
	\caption{SHD per edge of true graph versus (a) budget for $p=20$ and (b) graph orders for $k=3$ in random chordal graphs.} 
	\vspace{-3mm}
	\label{fig:exps:chordal:sample}
\end{figure*}

\begin{figure*}[t]
	\centering
	\begin{minipage}{.5\textwidth}
		\centering
		\includegraphics[width=.98\linewidth]{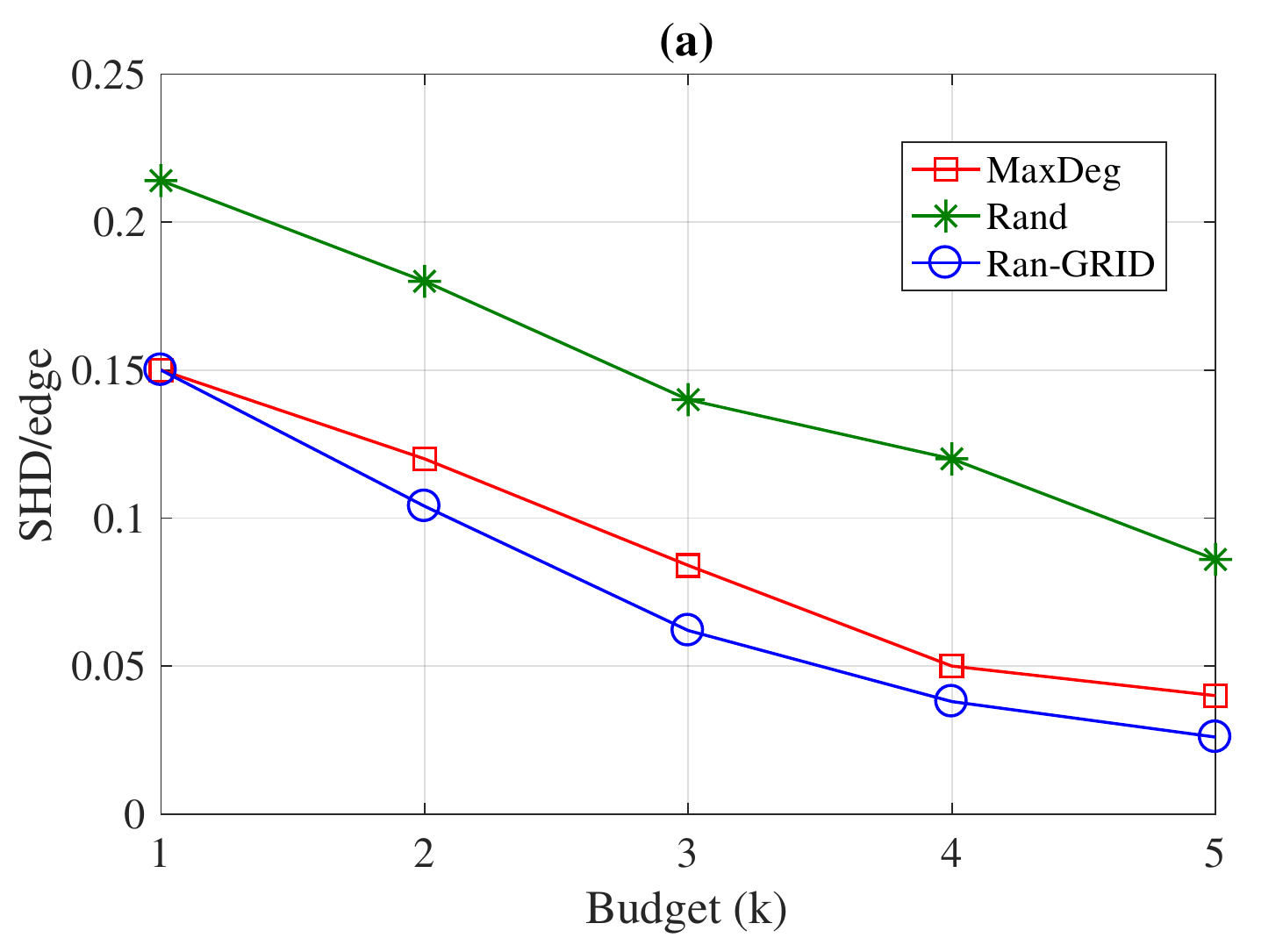}
	\end{minipage}%
	\begin{minipage}{.5\textwidth}
		\centering
		\includegraphics[width=.98\linewidth]{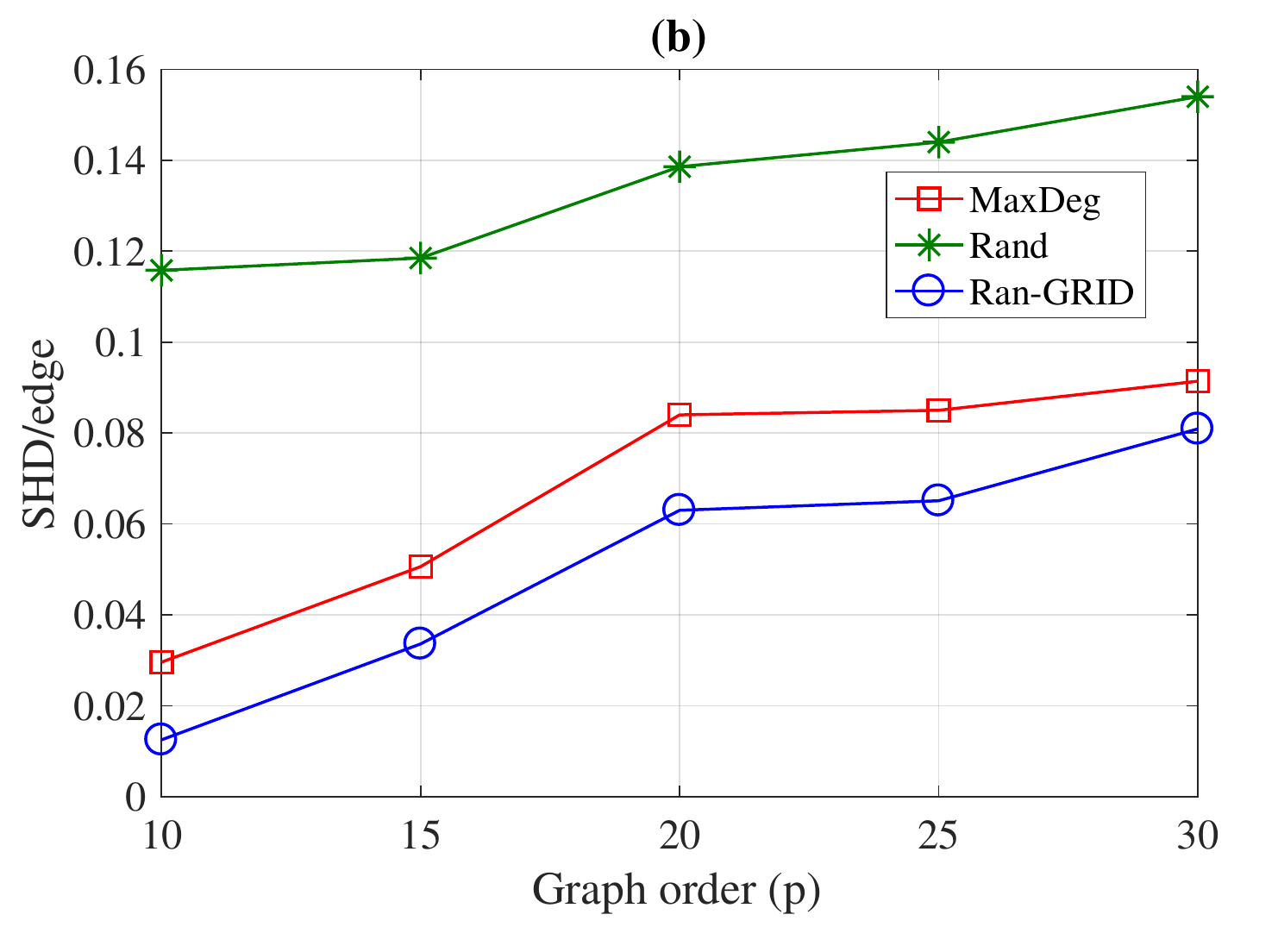}
	\end{minipage}%
	\caption{SHD per edge of true graph versus (a) budget for $p=20$ and (b) graph orders for $k=3$ in Erd\"{o}s-R\'{e}nyi graphs.} 
	\vspace{-3mm}
	\label{fig:exps:erdos:sample}
\end{figure*}

\textbf{Oracle case:} In the oracle case, as a performance measure, we consider the ratio of the number of edges whose directions are discovered merely as a result of interventions, i.e., $D(\I, G^*)$ to the number of edges whose directions were not resolved from the observational data. Note that due to our specific graph generating approach in random chordal graphs, the orientation of none of the edges is learned from the observational data.

We generated 100 instances of chordal DAGs of order $p=20$. 
Figure \ref{fig:exps}(a) depicts the discovered edge ratio with respect to the budget $k$. As seen in this figure, three interventions suffices to discover the direction of more than 90\% of the edges. 
Further, to investigate the effect of the order of the graph on the performance of the proposed algorithm and two naive approaches, we evaluated the discovered edge ratio for budget $k=3$ on graphs with order $p\in \{10,15,20,25,30\}$ in Figure \ref{fig:exps}(b). As it can be seen in the figure, the discovered edge ratio for the proposed approach is greater than $91\%$ for all orders. The performance of Rand approach degrades dramatically as $p$ increases. Moreover, MaxDeg approach has even lower performance than Rand approach. We also studied the effect of graph density on the performance of proposed algorithm. Let parameter $r$ be the ratio of average number of edges to $\binom{p}{2}$.
The discovered edge ratio for chordal DAGs of order 20 versus budget for different densities is depicted in  Figure \ref{fig:exps}(c).

Next, we generated 100 instances of Erd\"{o}s-R\'{e}nyi graphs and repeated the same experiments explained above. Note that in this case, the direction of some of the edges may be discovered in the observational essential graph. Experiment results are given in Figure \ref{fig:exps:erdos}. As can be seen, Ran-GID approach has the best performance and MaxDeg is close to it. Moreover, the discovered edge ratio is higher for denser graphs.

Furthermore, to compare the performance of the proposed algorithm with the optimal solution, we generated 100 instances of chordal DAGs of order $p=10$ and performed a brute force search to find the optimal solution for budget $k=2$. The discovered edge ratio was $0.9$ and $0.916$ for our proposed algorithm and the optimal solution, respectively. For the aforementioned setting, the running time of the proposed approach on a machine with Intel Core i7 processor and 16 GB of RAM was $216$ seconds while the one of the brute force approach was greater than $6000$ seconds.

\textbf{Sample case:} 
In this part, we first generated $10^4$ samples of observational data and fed them as the input to the GES algorithm \citep{chickering2002optimal} to obtain an estimation of the  essential graph.
It is noteworthy to mention that the essential graph might be different from the true essential graph due to finite samples. 
Then, we generated $10^4$ samples of interventional data for each experiment and gave the collection of all observational and interventional data to GIES algorithm \citep{hauser2012characterization} to get the final output. We considered Structural Hamming Distance (SHD)  as the performance metric, which measures the differences of the output graph and the true causal graph. 
Let $B$ and $\hat{B}$ be the binary adjacency matrices of the ground truth causal DAG and the output of an algorithm, respectively. SHD is defined as follows:
\[
SHD(B,\hat{B})\coloneqq\sum_{1\leq i<j\leq p} \mathds{1}[(B_{ij}\neq\hat{B}_{ij})\vee (B_{ji}\neq\hat{B}_{ji})],
\]
where $\mathds{1}[\cdot]$ is the indicator function. 
If the output of GES and the output of GIES after performing experiments are too different, one might exclude these instances in computing SHD since the essential graph obtained from observational data has too many errors.  

In Figure \ref{fig:exps:chordal:sample}(a), SHD per edges of true graph is illustrated versus the budget for $p=20$. As can be seen, Ran-GRID outperforms other methods and it can fairly learn the true causal graph after five interventions. 
In Figure \ref{fig:exps:chordal:sample}(b), SHD per edges of true graph is depicted versus the graph order for $k=3$. Again, Ran-GRID has the best performance and SHD per edge increases by increasing the graph order. 
Next, we performed the same experiment for Erd\"{o}s-R\'{e}nyi graphs where the average degree of vertices is set to 3. The results are given in Figure \ref{fig:exps:erdos:sample}. It can be seen that Ran-GRID performs better than other methods for any budget or graph order.

\subsubsection{Real Graphs}
We evaluated the performance of the proposed Improved Greedy Algorithm in gene regulatory networks (GRN). GRN is a collection of biological regulators that interact with each other. In GRN, the transcription factors are the main players to activate genes. The interactions between transcription factors and regulated genes in a species genome can be presented by a directed graph. In this graph, links are drawn whenever a transcription factor regulates a gene's expression. Moreover, some of vertices have both functions, i.e., are both transcription factor and regulated gene. 

We considered GRNs in ``DREAM 3 In Silico Network" challenge, conducted in 2008 \citep{marbach2009generating}. The networks in this challenge were extracted from known biological interaction networks. Since we know the true causal structures in these GRNs, we can obtain $Ess(G^*)$ and give it as an input to the proposed algorithm. Figure \ref{fig:sim3} depicts the discovered edge ratio in five networks extracted from GRNs of E-coli and Yeast bacteria with budget $k=5$. The order of each network is 100. As it can be seen, the discovered edge ratio is at least $0.65$ in all GRNs.

\begin{figure}[t]
	\begin{center}
		\centerline{\includegraphics[scale=0.4]{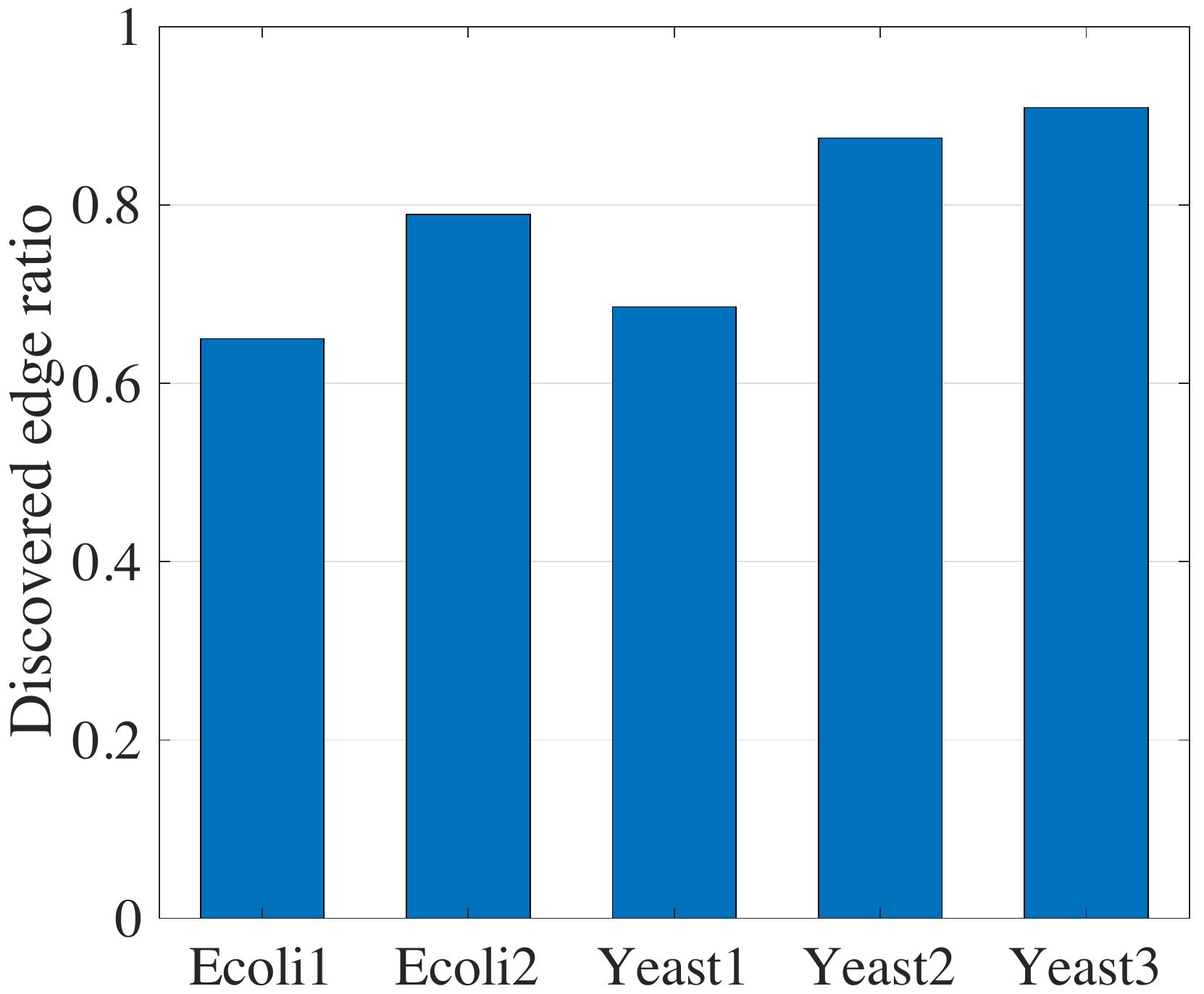}}
		\caption{Discovered edge ratio in five GRNs from DREAM 3 challenge.}
		\label{fig:sim3}
	\end{center}
\end{figure}

\section{Conclusion}
\label{sec:conc}

Without any assumptions on the causal modules, from observational data, a causal DAG can be learned only up to its Markov equivalence class, and hence, the direction of a large portion of the edges may be remained unidentified. In this case, it is common to perform interventions on a subset of the variables and use the resulting interventional distributions to improve the identifiability. Here, a natural question is that on which variables one should perform the intervention to gain the most from that intervention. We considered a setup in which the experimenter is limited to a budget $k$ for the number of interventions and the interventions should be designed non-adaptively. This setup can be considered as an extension to the customary adaptive design, in which only one intervention is designed at a time. For large values of $k$ a brute force search may not be feasible and efficient strategies for designing the interventions are required. We casted the problem as an optimization problem which aims to maximize the number of edges whose directions are identified due to the performed interventions. Here, both worst-case gain and average gain optimization can be considered. We first focused on the case that the underlying causal structure is a tree. For this case, we proposed an efficient exact algorithm for the worst-case gain setup, and an approximate algorithm for the average gain setup. The proposed approach for the average gain setup was based on our result that the objective function of the optimization in this case is monotonically increasing and submodular. In our synthetic simulations on different tree generation models, we observed that the proposed optimal algorithm for the worst-case gain also had a very high performance for the average gain. We then showed that the proposed approach for the average gain setup can be extended to the case of general causal structures. However, in this case, besides the design of interventions, calculating the objective function of the optimization problem is also challenging. This is due to the fact that the number of the members of a Markov equivalence class can potentially be super exponential in the number of the variables. We propose an efficient exact calculator for the objective function as well as two estimators. All these methods are based on a proposed method for counting and uniform sampling from the members of a Markov equivalence class. We evaluate the proposed methods using synthetic as well as real data. 

Providing an exact algorithm for the average gain setup, designing interventions for the worst-case gain setup for general causal structures, and considering the problem when the variables of the system can have latent confounders are among the directions that can be considered as future work.


\newpage

\begin{appendices}

\section{Example of Comparison with the Influence Maximization Problem}

\begin{figure}[h]
\begin{center}
\centerline{\includegraphics[scale=0.35]{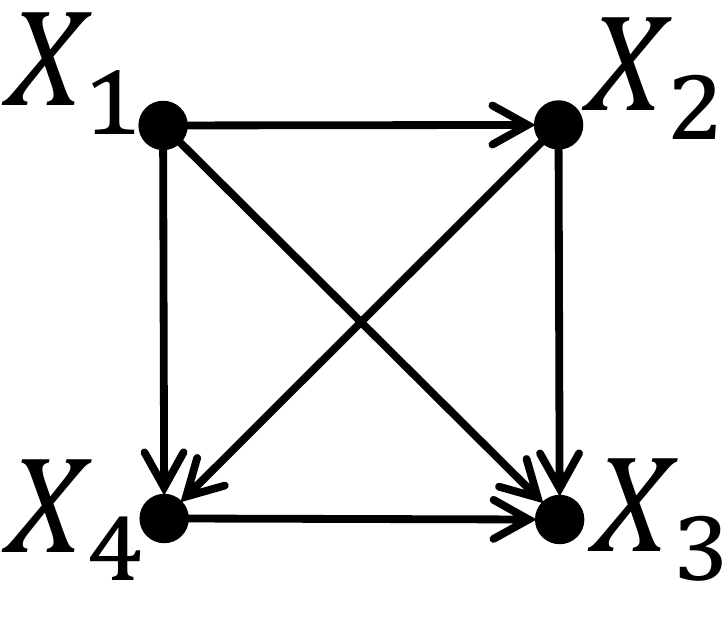}}
\caption{Example of comparison with the influence maximization problem.}

\label{fig:ex}
\end{center}
\end{figure}

Suppose $k=1$. Figure \ref{fig:ex} depicts a graph for which the optimal solution to the influence maximization problem is different from the optimal solution to the budgeted experiment design problems. Clearly, influencing vertex $X_1$ leads to influencing all the vertices in the graph, and hence, this vertex is the solution to the influence maximization problem. But, intervening on $X_1$ leads to discovering the orientation of only 3 edges, while intervening on, say $X_2$, leads to discovering the orientation of 5 edges.

\section{Proof of Lemma \ref{lem:neigh}}
\label{app:lem:neigh}

	From the passive observational stage, the set of all edges incident with $X_i$ is known. Suppose $X_j$ is adjacent with $X_i$ with unknown edge direction. If this edge in the ground truth structure has direction $X_i\rightarrow X_j$, then in the interventional distribution, there exists a subset of vertices $X_S$ containing $X_i$, for which $W_i\perp X_j|X_S$, where $W_i$ is the intervention variable corresponding to the singleton intervention on $X_i$. On the other hand, if this edge in the ground truth structure has direction $X_i\leftarrow X_j$, then in the interventional distribution, for all subsets of vertices $X_S$ containing $X_i$, we have $W_i\not\perp X_j|X_S$.
	
	The proof above works for both cases of hard and soft interventions.  \cite{eberhardt2005number} provided an alternative proof for the case of hard interventions, and \cite{he2008active} provided  alternative proofs for both cases of of soft and hard interventions.

\section{Proof of Lemma \ref{lem:rootree}}
\label{app:lem:rootree}

Suppose the root vertex is $X$. Since $\tT$ is a tree, there is a unique path from X to every other vertex.
For every vertex with path length 1 from the root, i.e., every vertex adjacent to the root, by definition, the edge is from $X$ to that vertex.
For every vertex $X_j$ with path length 2 from the root, we have the induced subgraph $X\rightarrow X_i - X_j$, and hence, since there cannot be any v-structures in the graph, the edge $X_i - X_j$ should be oriented as $X_i \rightarrow X_j$.
As the induction hypothesis, assume that for every vertex $X_i$ with path length $m$ from the root, we have the induced subgraph $X\rightarrow\cdots\rightarrow X_i$. 
Now for every vertex $X_j$ with path length $m+1$ from the root, we have the induced subgraph $X\rightarrow\cdots\rightarrow X_i - X_j$. Again, since there cannot be any v-structures in the graph, the edge $X_i - X_j$ should be oriented as $X_i \rightarrow X_j$.
Therefore, the location of the root variable identifies the direction of all the edges.

\section{Proof of Lemma \ref{lem:gain}}
\label{app:lem:gain}

We use the following lemma for the proof.
\begin{lemma}
\label{lem:desc}
For a tree UCEG $\tT$ on variable set $V$, an intervention on a variable $X_k\in V$ only determines the direction of all the edges incident to $\textit{Desc}(X_k)$, where descendants of a variable are defined with respect to the ground truth directed tree.	
\end{lemma}
\begin{proof}
By Lemma \ref{lem:neigh}, an intervention on $X_k$ identifies the direction of all edges incident to $X_k$.
Since $\tT$ is a tree, there is a unique path from X to every other vertex.
For every vertex for which the path from $X_k$ to that vertex goes through a child of $X_k$, similar to Lemma \ref{lem:rootree}, the direction of incident edges to that vertex will be identified.
Therefore, we learn the direction of all the edges incident to $\textit{Desc}(X_k)$.
Now, suppose $X_i$ is a parent of $X_k$. Therefore, for every vertex $X_j$ adjacent to $X_i$, we have the induced subgraph $X_j - X_i\rightarrow X_k$. Hence the edge $X_j - X_i$ can have either of the directions without creating a v-structure, and hence, the direction of such edge cannot be identified. Therefore, the direction of any of the edges incident to $X_j$ cannot be identified either. Consequently, we do not learn the direction of all any of the edges incident to $\textit{Non-Desc}(X_k)$.
\end{proof}

Suppose the ground truth directed tree is $T^{X}_r$. 
 By Lemma \ref{lem:desc}, after an experiment with target set $\I_r$, the edges whose directions are remained unresolved are those which are incident only to $\cap_{X_k\in\I_r}\textit{Non-Desc}(X_k)$, which are the edges of the component $C_j(\I_r)$, where $X\in C_j(\I_r)$. Noting that the size of a tree of order $p$ is $p-1$ concludes that the number of unresolved edges are $|C_j(\I_r)|-1$. 
If 	$X\in\I_r$, then $\cap_{X_k\in\I_r}\textit{Non-Desc}(X_k)=\emptyset$, i.e., the direction of all the edges are identified and the gain will be $D(\I_r,T_r^X)=|\tT_r|-1$. Otherwise the gain will be $D(\I_r,T^X_r)=|\tT_r|-1-|C_j(\I_r)|+1=|\tT_r|-|C_j(\I_r)|$.

\section{Proof of Proposition \ref{prop:avggain}}
\label{app:prop:avggain}

	We can write the average gain $\mathcal{D}(\I)$ as follows:
	\begin{align*}
	\mathcal{D}(\I)&=\frac{1}{p_u}\sum_{r=1}^R\sum_{X\in V(\tT_r)}D(\I_r,T^{X}_r) \\
	&\overset{(a)}{=}\frac{1}{p_u}\sum_{r=1}^R\sum_{X\in\I_r \cap V(\tT_r)}(|\tT_r|-1)+\frac{1}{p_u}\sum_{r=1}^R\sum_{j=1}^{J(\I_r)}\sum_{X\in C_j({\I_r})}|\tT_r|-|C_j({\I_r})|\\
	&=\frac{1}{p_u}\sum_{r=1}^R |\I_r|(|\tT_r|-1)+\frac{1}{p_u}\sum_{r=1}^R\sum_{j=1}^{J(\I_r)} |\tT_r||C_j({\I_r})|-|C_j({\I_r})|^2\\
	&\overset{(b)}{=}\frac{1}{p_u}\sum_{r=1}^R |\I_r|(|\tT_r|-1)+\frac{1}{p_u} \sum_{r=1}^R |\tT_r|(|\tT_r| -|\I_r|)-\frac{1}{p_u}\sum_{r=1}^R \sum_{j=1}^{J(\I_r)} |C_j({\I_r})|^2\\
	&=\frac{1}{p_u}\sum_{r=1}^R |\tT_r|^2-\frac{k}{p_u}-\frac{1}{p_u}\sum_{r=1}^R \sum_{j=1}^{J(\I_r)} |C_j(I_r)|^2,
	\end{align*}
	where $(a)$ is due to Lemma \ref{lem:gain} and 
	$(b)$ follows from the fact that vertices which belong to component, only exclude vertices in $\I$.


\section{Proof of Theorem \ref{thm:algTrED}}
\label{app:thm:algTrED}

We use the following lemma for the proof.
\begin{lemma}
\label{lem:leastremoval}
Among all algorithms achieving a threshold $mid$, Algorithm \ref{algTrED} uses the least number of vertex removals.
\end{lemma} 
\begin{proof}
Proof by induction. We show for each subtree, the smallest number of vertex removal is used. Since the proposed algorithm removes a vertex only if not doing so results in having a subtree with the order larger than the threshold, it delays a removal as much as possible. Now suppose for vertex $X_j$, we have used the smallest number of removals, say $l$, in subtrees rooted at the children of $X_j$. Because in each of those subtrees, the removals have been delayed the most, the order of remaining part for the subtree rooted at $X_j$ with $l$ removals is minimum. Therefore the subtree rooted at $X_j$ also contributes the least value (zero if it is chosen to intervene on) to the order of the subtree rooted at its parent.
 
\end{proof}

Now, suppose for the optimum experiment target set $\I^*_r$, that is,
\[
\I^*_r=\arg\min_{\I_r:\I_r\subseteq V(\tT_r)} \max_{1\leq j \leq J(\I_r)} |C_j(\I_r)|,
\]
with $|\I^*_r|=k_r$ we have $M^*\coloneqq\max_{1\leq j \leq J(\I^*_r)} |C_j(\I^*_r)|<\min_{X_i} mid(X_i)$. 
In this case, in the binary search in Algorithm \ref{algTrED}, when the threshold is set to $mid$ such that $M^*-1<mid\le M^*$, then by Lemma \ref{lem:leastremoval}, Algorithm \ref{algTrED} should have used less than or equal to $k_r$ vertex removals.
If it has used less than $k_r$ vertex removals, it means that it can achieve $M^*$ with $|\hat{\I}_r|<k_r$, and hence, can achieve a value less than $M^*$ with $k_r$ vertex removals, which implies that $\I^*_r$ is not optimum. Therefore, we should have 
\[
\min_{X_i} mid(X_i)\linebreak=\min_{\I_r:\I_r\subseteq V(\tT_r)} \max_{1\leq j \leq J(\I_r)} |C_j(\I_r)|.
\]

\section{Proof of Proposition \ref{prop:treeavg}}
\label{app:prop:treeavg}

{\bf Monotonicity.}
Consider $\I^1\subseteq\I^2$. Target set $\I^2$ divides some of the components of target set $\I^1$ into smaller components, or removes vertices from some of them, and keeps the rest unchanged. Suppose $C_j$ is a changed component. Therefore, corresponding to this component, for $\I^1$ we have the term $|C_j|^2$, and for $\I^2$ we have  $\sum_{l=1}^L|C_{jl}|^2$ such that $\sum_{l=1}^L|C_{jl}|<|C_j|$. Basic algebra and induction on $L$ indicates that under this condition $\sum_{l=1}^L|C_{jl}|^2$ is always less that $|C_j|^2$. Hence, $\mathcal{D}(\I^1)\le\mathcal{D}(\I^2)$.\\

\noindent
{\bf Submodularity.}
We first show that the for every root vertex $X_i$, the set function $D(\I,T^{X_i})$ is submodular. i.e., for $\I^1\subseteq\I^2$, vertex $X$,
\[
D(\I^1\cup\{X\},T^{X_i})-D(\I^1,T^{X_i})\ge D(\I^2\cup\{X\},T^{X_i})-D(\I^2,T^{X_i}).
\]
By Lemma \ref{lem:gain}, the value of the function $D(\I,T^{X_i})$ only depends on the component containing the root.
Suppose under experiment $\I^1$ the root vertex falls in component $C_{\I^1}$, and under experiment $\I^2$ the root vertex falls in component $C_{\I^2}$. 
If $C_{\I^1}=C_{\I^2}$, the result is immediate, as without intervening on $X$, $\I^1$ and $\I^2$ result in the same value for function $D$, and intervening on $X$ will also have the same result in bot experiments. Otherwise. since $\I^1\subseteq\I^2$, we have $C_{\I^2}\subseteq C_{\I^1}$. Hence, the cardinality of the set of the edges which are incident to $\textit{Desc}(X)$ in $C_{\I^1}$ is larger that the cardinality of the set of the edges which are incident to $\textit{Desc}(X)$ in $C_{\I^2}$. This implies that we have a larger gain by intervening on $X$ starting from $\I_1$ compared to $\I^2$, i.e., $D(\I^1\cup\{X\},T^{X_i})-D(\I^1,T^{X_i})\ge D(\I^2\cup\{X\},T^{X_i})-D(\I^2,T^{X_i})$.

Finally, using equality $\mathcal{D}(\I)=\frac{1}{p_u}\sum_{r=1}^R\sum_{X\in V(\tT_r)}D(\I_r,T_r^{X})$, since a non-negative linear combination of submodular functions is also submodular, the desired result is concluded.

\section{Proof of Proposition \ref{prop:mono}}
\label{app:prop:mono}

First we show that for a given directed graph $G_i\in\MEC(G^*)$ the function $D(\I,G_i)$ is a monotonically increasing function of $\I$. 
In the proposed method, intervening on elements of $\I$, we first discover the orientation of the edges in $A(\I,G_i)$, and then applying the Meek rules, we possibly learn the orientation of some extra edges. 
Having $\mathcal{I}_1\subseteq \mathcal{I}_2$ implies that $A(\I_1,G_i)\subseteq A(\I_2,G_i)$. Therefore using $\mathcal{I}_2$, we have more information about the direction of edges. Hence, in the step of applying Meek rules, by soundness and order-independence of Meek algorithm, we recover the direction of more extra edges, i.e., $R(\I_1,G_i)\subseteq R(\I_2,G_i)$, which in turn implies that $D(\mathcal{I}_1,G_i)\le D(\mathcal{I}_2,G_i)$.
Finally, from the equation $\mathcal{D}(\I)=\frac{1}{|\MEC(G^*)|}\sum_{G_i\in\MEC(G^*)}D(\mathcal{I},G_i)$, the desired result is immediate.

\section{Proof of Lemma \ref{lem:nofusion}}
The direction $R(\I_1,G^*)\cup R(\I_2,G^*)\subseteq R(\I_1\cup\I_2,G^*)$ is proved in the proof of Proposition \ref{prop:mono}. 
Define $A(\tG^*)$ as the set of directed edges in $\tG^*$, and let $R(M,G^*)$ be the set of undirected edges of $\tG^*$ whose directions can be identified by applying Meek rules starting from $A(\tG^*)\cup R(\I_1,G^*)\cup R(\I_2,G^*)$.
Again by the reasoning in the proof of Proposition \ref{prop:mono}, we have $R(\I_1\cup\I_2,G^*)\subseteq R(M,G^*)$. Therefore, in order to prove that $R(\I_1\cup\I_2,G^*)\subseteq R(\I_1,G^*)\cup R(\I_2,G^*)$, it suffices to show that $R(M,G^*)\subseteq R(\I_1,G^*)\cup R(\I_2,G^*)$, for which it suffices to show that for every directed edge $e$, if $e\not\in R(\I_1,G^*)$ and $e\not\in R(\I_2,G^*)$, then $e\not\in R(M,G^*)$.

\textit{Proof by contradiction.} Let $e\not\in R(\I_1,G^*)$ and $e\not\in R(\I_2,G^*)$, but its orientation is learned in the first iteration of applying Meek rules to $A(\tG^*)\cup R(\I_1,G^*)\cup R(\I_2,G^*)$. Then, we have learned the orientation of $e$ due to one of Meek rules \citep{verma1992algorithm}:
\begin{itemize}
\item \textbf{Rule 1.} $e=A-B$ is oriented as $A\rightarrow B$ if there exists $C$ such that $e_1=C\rightarrow A\in A(\tG^*)\cup R(\I_1,G^*)\cup R(\I_2,G^*)$, and $C-B\not\in$ skeleton of $G^*$.
\item \textbf{Rule 2.} $e=A-B$ is oriented as $A\rightarrow B$ if there exists $C$ such that $e_1=A\rightarrow C\in A(\tG^*)\cup R(\I_1,G^*)\cup R(\I_2,G^*)$, and $e_2=C\rightarrow B\in A(\tG^*)\cup R(\I_1,G^*)\cup R(\I_2,G^*)$.
\item \textbf{Rule 3.} $e=A-B$ is oriented as $A\rightarrow B$ if there exists $C$ and $D$ such that $e_1=C\rightarrow B\in A(\tG^*)\cup R(\I_1,G^*)\cup R(\I_2,G^*)$, $e_2=D\rightarrow B\in A(\tG^*)\cup R(\I_1,G^*)\cup R(\I_2,G^*)$, $A-C\in$ skeleton of $G^*$, $A-D\in$ skeleton of $G^*$, and $C-D\not\in$ skeleton of $G^*$.
\item \textbf{Rule 4.} $e=A-B$ is oriented as $A\rightarrow B$ and $e=B-C$ is oriented as $C\rightarrow B$ if there exists $D$ such that $e_1=D\rightarrow C\in A(\tG^*)\cup R(\I_1,G^*)\cup R(\I_2,G^*)$, $A-C\in$ skeleton of $G^*$, $A-D\in$ skeleton of $G^*$, and $B-D\not\in$ skeleton of $G^*$.
\end{itemize}




In what follows, we show that the orientation of $e$ cannot be learned due to any of the Meek rules unless directed edge $e$ belongs to $R(\I_1,G^*)$ or $R(\I_2,G^*)$.\\

\noindent
\textbf{Rule 1.}

Without loss of generality, assume $e_1\in A(\tG^*)\cup R(\I_1,G^*)$. Therefore, we should have the condition of rule 1 satisfied when only intervening on $\mathcal{I}_1$ as well, which implies that $e\in R(\I_1,G^*)$, which is a contradiction.\\


\noindent
\textbf{Rule 2.}

If both $e_1$ and $e_2$ belong to $A(\tG^*)\cup R(\I_1,G^*)$ (or $A(\tG^*)\cup R(\I_2,G^*)$), then we should have the condition of rule 2 satisfied when only intervening on $\mathcal{I}_1$ (or $\mathcal{I}_2$) as well, which implies that $e\in R(\I_1,G^*)$ (or $e\in R(\I_1,G^*)$), which is a contradiction. Therefore, it suffices to show that the case that $e_1$ belongs to exactly one of $A(\tG^*)\cup R(\I_1,G^*)$ or $A(\tG^*)\cup R(\I_2,G^*)$ and $e_2$ belongs only to the other one, does not happen. To this end, it suffices to show that there does not exist experiment target set $\mathcal{I}$ such that $e_1\in A(\tG^*)\cup R(\I,G^*)$, and $e,e_2\not\in A(\tG^*)\cup R(\I,G^*)$, i.e., there does not exist experiment target set $\mathcal{I}$ that has structure $S_0$, depicted in Figure \ref{fig:s0}, as a subgraph of $\tG^*$ after applying the orientations learned from $R(\I,G^*)$.

\begin{figure}[h]
\begin{center}
\centerline{\includegraphics[scale=0.23]{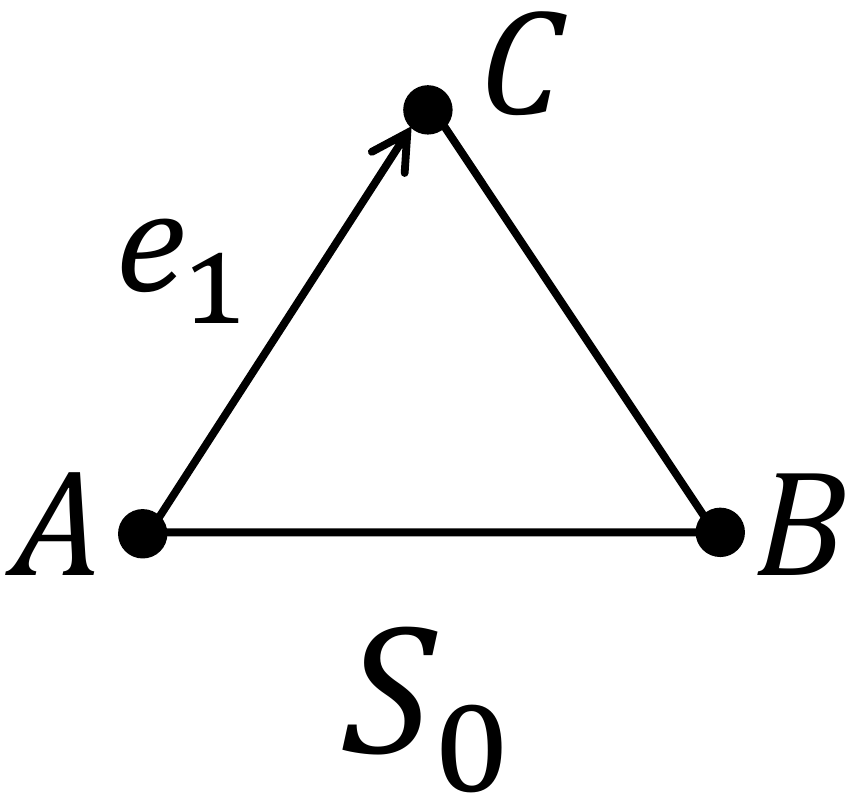}}
\caption{Structure $S_0$}
\label{fig:s0}
\end{center}
\end{figure}

If $e_1\in A(\I,G^*)$, then $A\in \I$ or $C\in \I$, which implies $e\in A(\I,G^*)$ or $e_2\in A(\I,G^*)$, respectively, and hence, $e\in R(\I,G^*)$ or $e_2\in R(\I,G^*)$, respectively. Therefore, in either case, $e\in R(\I,G^*)$, and $S_0$ will not be a subgraph.
Therefore, $e_1\not\in A(\I,G^*)$, and hence, $e_1$ was learned by applying one of the Meek rules. We consider each or the rules in the following:
\begin{itemize}
\item If we have learned the orientation of $e_1$ from rule 1, then we should have had one of the structures in Figure \ref{fig:rule1} as a subgraph of $\tG^*$ after applying the orientations learned from $R(\I,G^*)$. In case of structure $S_1$, using rule 1 on subgraph induced on vertices $\{X_1,A,B\}$, we will also learn $A\rightarrow B$. In case of structure $S_2$, using rule 4, we will also learn $B\rightarrow C$. Therefore, we cannot learn only the direction of $e_1$ and hence, $S_0$ will not be a subgraph.
\begin{figure}[h]
\begin{center}
\centerline{\includegraphics[scale=0.23]{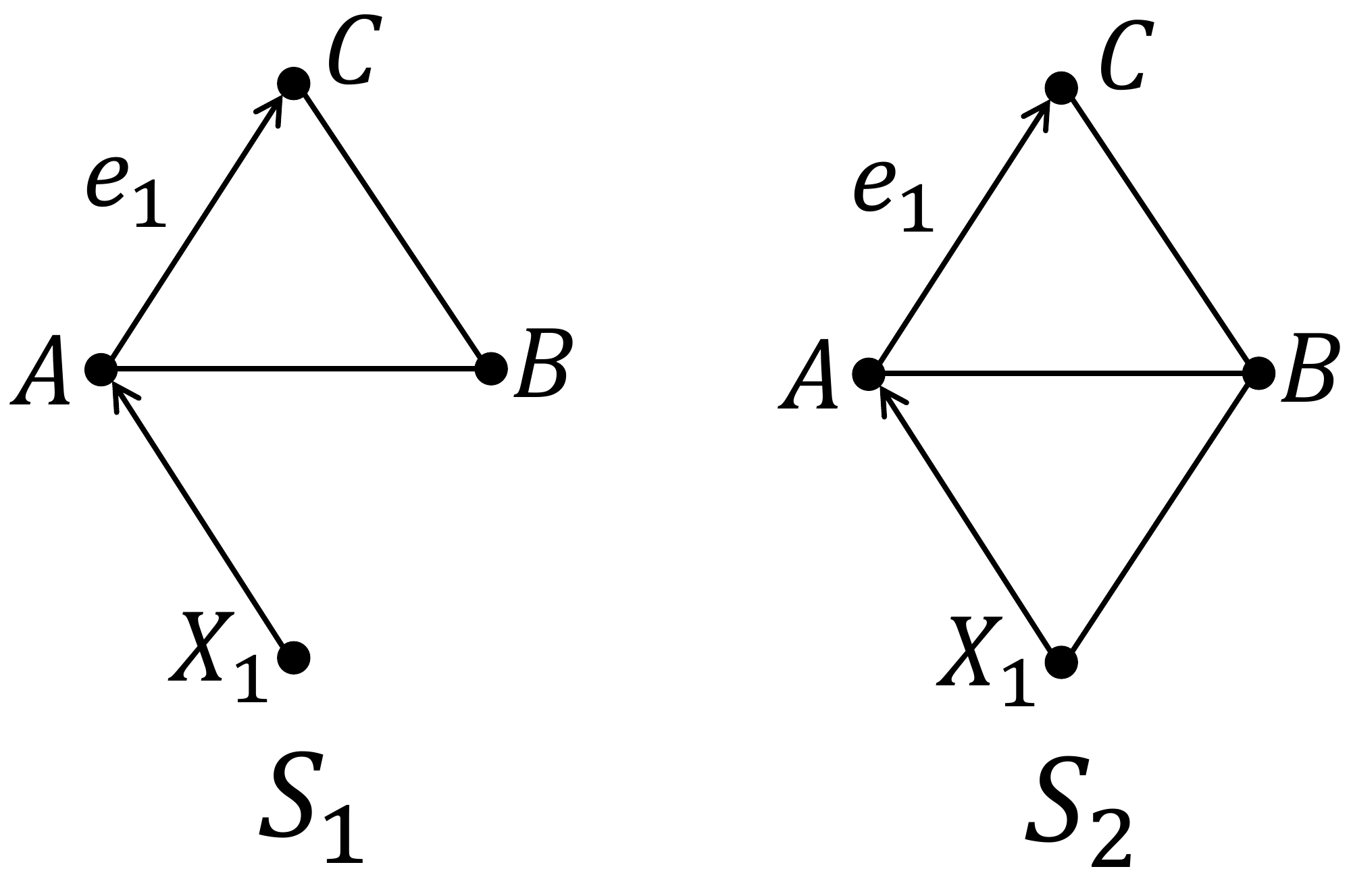}}
\caption{Rule 1}
\label{fig:rule1}
\end{center}
\end{figure}

\item If we have learned the orientation of $e_1$ from rule 3, then we have had one of the structures in Figure \ref{fig:rule3} as a subgraph of $\tG^*$ after applying the orientations learned from $R(\I,G^*)$. In case of structures $S_3$ and $S_4$, using rule 1 on subgraph induced on vertices $\{X_2,C,B\}$, we will also learn $C\rightarrow B$. In case of structure $S_5$, using rule 3 on subgraph induced on vertices $\{B,X_2,C,X_1\}$, we will also learn $B\rightarrow C$. Therefore, we cannot learn only the direction of $e_1$ and hence, $S_0$ will not be a subgraph.
\begin{figure}[h]
\begin{center}
\centerline{\includegraphics[scale=0.23]{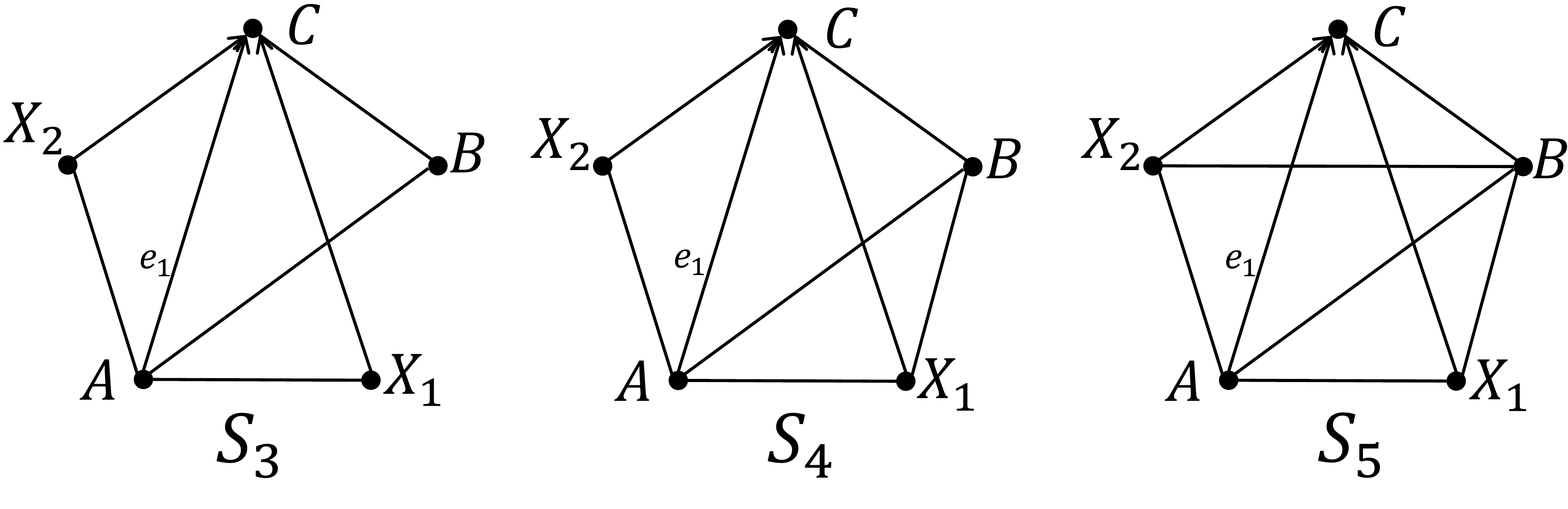}}
\caption{Rule 3}
\label{fig:rule3}
\end{center}
\end{figure}

\item If we have learned the orientation of $e_1$ from rule 4, then we have had one of the structures in Figure \ref{fig:rule4} as a subgraph of $\tG^*$ after applying the orientations learned from $R(\I,G^*)$. In case of structures $S_6$, using rule 1 on subgraph induced on vertices $\{X_1,C,B\}$, we will also learn $C\rightarrow B$. In case of structure $S_7$, using rule 1 on subgraph induced on vertices $\{X_2,X_1,B\}$, we will also learn $X_1\rightarrow B$, and then using rule 4 on subgraph induced on vertices $\{B,A,X_2,X_1\}$, we will also learn $A\rightarrow B$. In case of structures $S_{8}$, using rule 4 on subgraph induced on vertices $\{B,X_2,X_1,C\}$, we will also learn $B\rightarrow C$. Therefore, we cannot learn only the direction of $e_1$ and hence, $S_0$ will not be a subgraph.
\begin{figure}[h]
\begin{center}
\centerline{\includegraphics[scale=0.23]{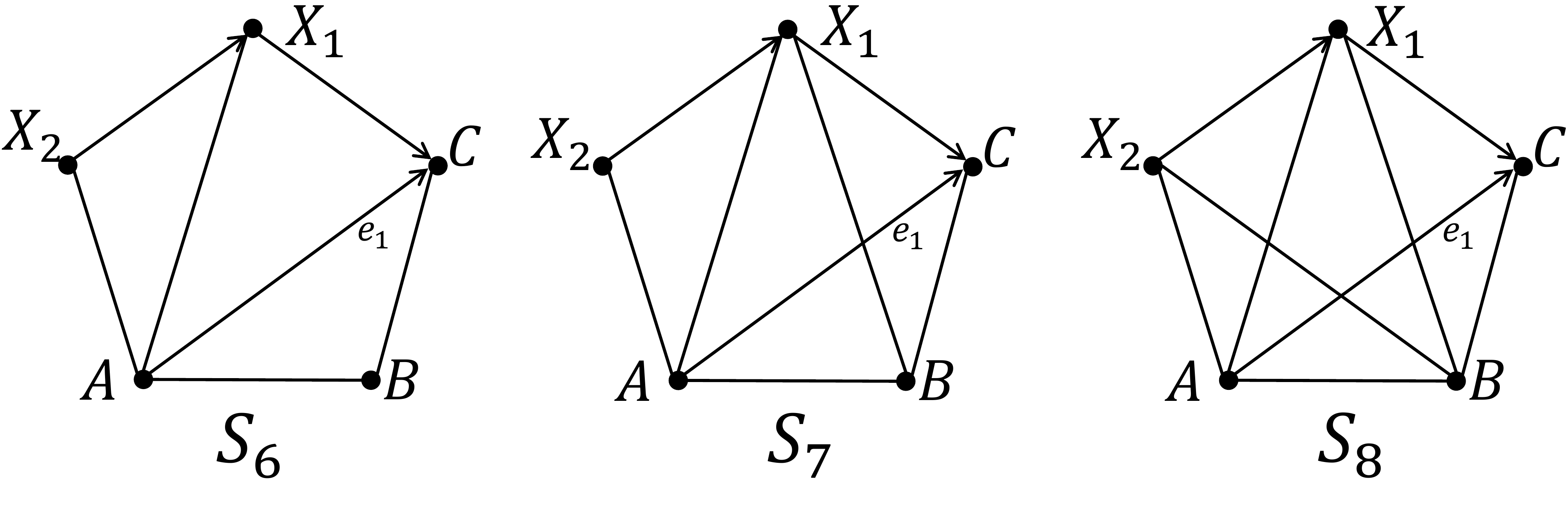}}
\caption{Rule 4}
\label{fig:rule4}
\end{center}
\end{figure}

\item If we have learned the orientation of $e_1$ from rule 2, then we should have had one of the structures in Figure \ref{fig:rule2} as a subgraph of $\tG^*$ after applying the orientations learned from $R(\I,G^*)$. In case of structure $S_9$, using rule 1 on subgraph induced on vertices $\{X_1,C,B\}$, we will also learn $C\rightarrow B$ and hence, $S_0$ will not be a subgraph.
In case of structure $S_{10}$, if $X_1\in \I$, then the direction of the edge $X_1-B$ will be also known. If the direction of this edge is $X_1\rightarrow B$, then using rule 2 on subgraph induced on vertices $\{A,X_1,B\}$, we will also learn $A\rightarrow B$; otherwise, using rule 2 on subgraph induced on vertices $\{B,X_1,C\}$, we will also learn $C\rightarrow B$. Therefore, $X_1\not\in \I$. Also, as mentioned earlier, $A\not\in \I$. Therefore, we have learned the orientation of $A\rightarrow X_1$ from applying Meek rules.

In the triangle induced on vertices $\{X_1,B,A\}$, we have learned only the orientation of one edge, which is $A\rightarrow X_1$. But as seen in structures $S_1$ to $S_9$, all of them lead to learning the orientation of at least 2 edges of a triangle. In the following, we will show that a structure of form $S_{10}$, does not lead to learning the orientation of only $A\rightarrow X_1$ and making $S_{10}$ a subgraph either.



\begin{figure}[h]
\begin{center}
\centerline{\includegraphics[scale=0.23]{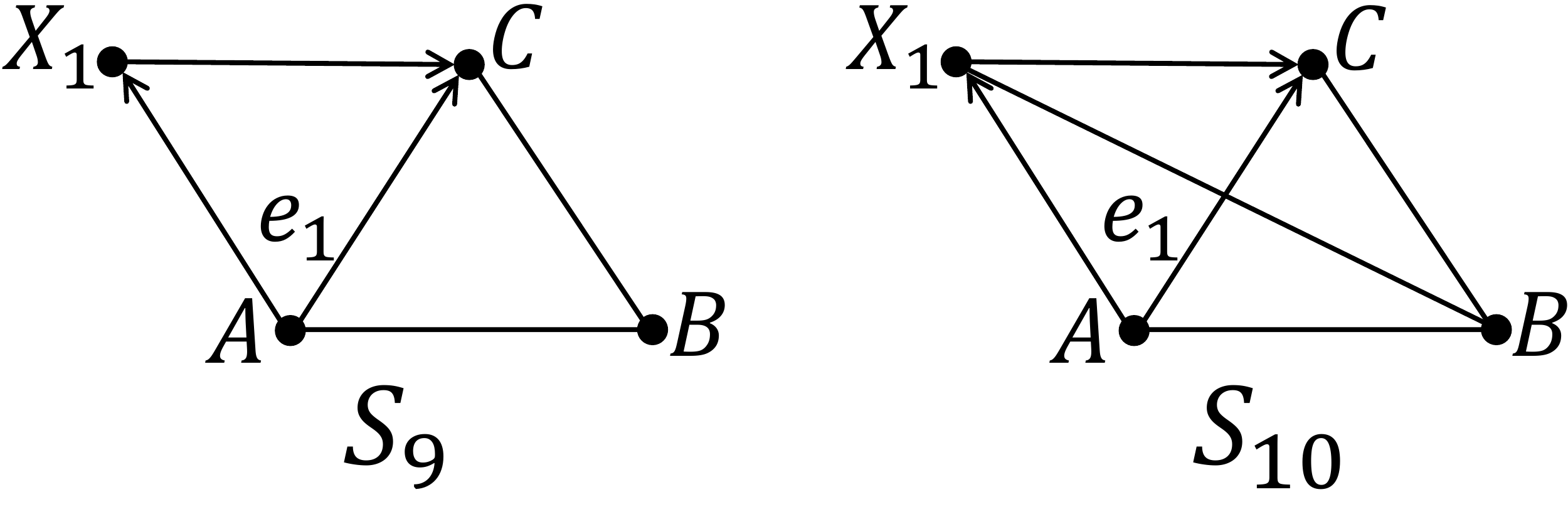}}
\caption{Rule 2}
\label{fig:rule2}
\end{center}
\end{figure}

Suppose we had learned $A\rightarrow X_1$ via a structure of form $S_{10}$, as depicted in Figure \ref{fig:s10}(a).
Using rule 4 on subgraph induced on vertices $\{X_2,X_1,C,B\}$, we will also learn $B\rightarrow C$. Therefore, we should have the edge $X_2-C$ too. Also, using rule 2 on triangle induced on vertices $\{X_2,X_1,C\}$, the orientation of this edges should be $X_2\rightarrow C$.
Therefore, in order to have $S_{10}$ as a subgraph, we need to have the structure depicted in Figure \ref{fig:s10}(b) as a subgraph.
As seen in Figure \ref{fig:s10}(b), we again have a structure similar to $S_{10}$: a complete skeleton $K_5$, which contains $X_j\rightarrow C$, $A\rightarrow X_j$, $X_j-B$, for $j\in\{1,2\}$ and $X_2\rightarrow X_1$, with a triangle on vertices $\{X_2,B,A\}$, in which we have learned only the orientation of $A\rightarrow X_2$.

\begin{figure}[h]
\begin{center}
\centerline{\includegraphics[scale=0.23]{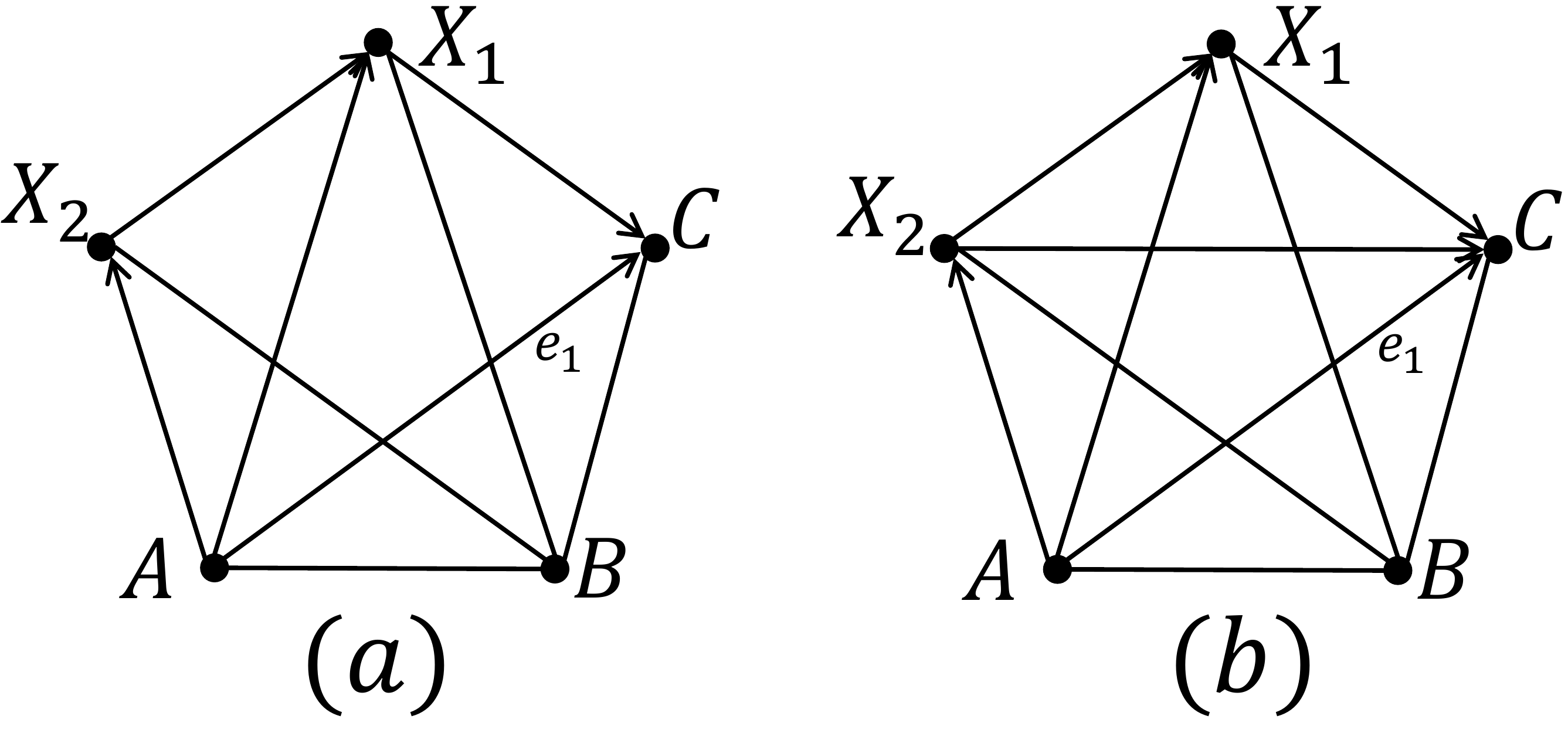}}
\caption{Step of the induction.}
\label{fig:s10}
\end{center}
\end{figure}

We claim that this procedure always repeats, i.e., at step $i$, we end up with skeleton $K_i$, which contains $X_j\rightarrow C$, $A\rightarrow X_j$, $X_j-B$, for $j\in\{1,...,i\}$ and $X_k\rightarrow X_j$, for $1\le j<k\le i$, with a triangle induced on vertices $\{X_i,B,A\}$, in which we have learned only the orientation of $A\rightarrow X_i$. We prove this claim by induction. We have already proved the base of the induction above. For the step of the induction, suppose the hypothesis is true for $i-1$. Add vertex $X_i$ to form a structure of form $S_{10}$ for $A\rightarrow X_{i-1}$. $X_i$ should be adjacent to $X_j$, for $j\in\{1,...,i-2\}$; otherwise, using rule 4 on subgraph induced on vertices $\{X_i,X_{i-1},X_j,B\}$, we will also learn $B\rightarrow X_j$. Moreover, using rule 2 on triangle induced on vertices $\{X_i,X_{i-1},X_j\}$, the direction of $X_i-X_j$ should be $X_i\rightarrow X_j$. Also, using rule 4 on subgraph induced on vertices $\{X_i,X_{i-1},C,B\}$, we will also learn $B\rightarrow C$. Therefore, we should have the edge $X_i-C$ too.

We showed that $S_0$ is a subgraph only if $S_{10}$ is a subgraph, and $S_{10}$ is a subgraph only if the structure in Figure \ref{fig:s10}(b) is a subgraph, and this chain of required subgraphs continues.
Therefore, since the order of the graph is finite, there exist a step where since we cannot add a new vertex, it is not possible to have one of the required subgraphs, and hence we conclude that $S_0$ is not a subgraph.\\
\end{itemize}

\noindent
\textbf{Rule 3.}

Since edges $e_1$ and $e_2$ form a v-structure, they should appear in $A(\tG^*)$ as well.  Therefore, we should have the condition of rule 3 satisfied when only intervening on $\mathcal{I}_1$ as well, which implies that $e\in R(\I_1,G^*)$, which is a contradiction.\\

\noindent
\textbf{Rule 4.}


Without loss of generality, assume $e_1\in R(\I_1,G^*)\cup A(\tG^*)$. Therefore, we should have the condition of rule 4 satisfied when only intervening on $\mathcal{I}_1$ as well, which implies that $e\in R(\I_1,G^*)$, which is a contradiction.\\

The argument above proves that there is no edge $e$ such that $e\not\in R(\I_1,G^*)$ and $e\not\in R(\I_2,G^*)$, but $e\in R(M,G^*)$.

\section{Proof of Theorem \ref{thm:submodular}}
\label{app:thm:submodular}

Due to Proposition \ref{prop:mono}, it suffices to show that for $\mathcal{I}_1\subseteq \mathcal{I}_2\subseteq V$, and $X_i\in V$, we have
$
\mathcal{D}(\I_1\cup \{X_i\})-\mathcal{D}(\I_1)\ge \mathcal{D}(\I_2\cup \{X_i\})-{\mathcal{D}}(\I_2)$.
First we show that for a given directed graph $G_i\in\MEC(G^*)$ the function $D(\I,G_i)$ is a submodular function of $\mathcal{I}$. 
From Lemma \ref{lem:nofusion}, we have $R(\I_1\cup \{X_i\},G_i)=R(\I_1,G_i)\cup R(\{X_i\},G_i)$. Therefore,
\begin{align*}
D(\I_1\cup \{X_i\},G_i)-D(\I_1,G_i)
&=|R(\I_1\cup \{X_i\},G_i)|-|R(\I_1,G_i)|\\
&=|R(\I_1,G_i)\cup R(\{X_i\},G_i)|-|R(\I_1,G_i)|\\
&=|R(\{X_i\},G_i)|-|R(\I_1,G_i)\cap R(\{X_i\},G_i)|.
\end{align*}
Similarly, 
\begin{align*}
D(\I_2\cup \{X_i\},G_i)-D(\I_2,G_i)
&=|R(\{X_i\},G_i)|-|R(\I_2,G_i)\cap R(\{X_i\},G_i)|.
\end{align*}
Since $\mathcal{I}_1\subseteq \mathcal{I}_2$, as seen in the proof of Proposition \ref{prop:mono}, $R(\I_1,G_i)\subseteq R(\I_2,G_i)$. Therefore, 
$-|R(\I_1,G_i)\cap R(\{X_i\},G_i)|
\ge-|R(\I_2,G_i)\cap R(\{X_i\},G_i)|$,
which implies that
\[
D(\I_1\cup \{X_i\},G_i)-D(\I_1,G_i)\ge D(\I_2\cup \{X_i\},G_i)-D(\I_2,G_i).
\] 
This together with the fact that the function $D(\I,G_i)$ is a monotonically increasing function of $\I$ (observed in the proof of Proposition \ref{prop:mono}) shows that $D(\I,G_i)$ is a submodular function of $\I$.

Finally, we have $\mathcal{D}(\I)=\frac{1}{|\MEC(G^*)|}\sum_{G_i\in\MEC(G^*)}D(\mathcal{I},G_i)$.
Since a non-negative linear combination of submodular functions is also submodular, the proof is concluded.

\section{Proof of Proposition \ref{prop:countcomp}}
\label{app:prop:countcomp}

The worst case in terms of computational complexity happens when $H=\tG$, as it requires maximum number of recursions. In function $\textsc{Counter}$, we set each vertex $X_i$ as the root and call the function $\textsc{Counter}$ for the rooted essential graph $\tG_r^{X_i}$ to compute the number of DAGs in the MEC corresponding to $\tG_r^{X_i}$. Using Meek rules, the directed edges in $\tG^{X_i}_r$ can be recovered in time $\mathcal{O}(p^3)$.

Now, we show that the degree of each vertex $X_j$ in
 $\tG^{X_i}_r$ decreases at least by one after removing directed edges. To do so, we prove that there exists a directed edge in $\tG^{X_i}_r$ that goes to vertex $X_j$.  
If $X_j$ is a neighbor of $X_i$ the proof is done, as edges are always directed from the root vertex towards its neighbors. Otherwise, consider the shortest path from $X_i$ to $X_j$ in $\tG^{X_i}_r$. This path must pass through one of the neighbors of $X_j$, say, $X_k$. Since the distance from $X_i$ to $X_k$ is less than $X_i$ to $X_j$, $X_k-X_j$ should be oriented as $X_k\rightarrow X_j$ \citep{bernstein2017sampling}. Therefore, the degree of each vertex $X_j$ in $\tG^{X_i}_r$ decreases at least by one after removing directed edges in $\tG^{X_i}_r$. 

Let $t(\Delta)$ be the computational complexity of Algorithm \ref{algorithm:counting} on a graph with maximum degree $\Delta$. Based on what we proved above, we have
\begin{equation*}
t(\Delta)\leq p t(\Delta-1)+Cp^3,
\end{equation*}  
where $C$ is a constant. The above inequality holds true since we have at most $p$ chain component in $\tG^{X_i}_r$, where the maximum degree in each of them is at most $\Delta-1$. From this inequality, it can be shown that $t(\Delta)$ is in the order of $\mathcal{O}(p^{\Delta+1})$. Since we may have at most $p$ chain components in essential graph $\tG$, the computational complexity of Algorithm \ref{algorithm:counting} is in the order of $\mathcal{O}(p^{\Delta+2})$.

\section{Proof of Theorem \ref{thm:unif}}
\label{app:thm:unif}

The objective is to show that for the input essential graph $\tG$, any DAG $G$ in $\MEC(\tG)$ is generated with probability $1/Size(\tG)$.

\emph{Proof by induction:} The function $\textsc{Counter}$ finds the size of a chain component recursively, i.e., after setting a vertex $X$ as the root and finding the orientations in $\tG^X_r$, it calls itself to obtain the size of the chain components of $\tG^X_r$. We induct on the maximum number of recursive calls required for complete orienting.\\
{\bf Induction base:}
For the base of the induction, we consider an essential graph with no required recursive call: Consider essential graph $\tG$ with chain component set $\mathcal{G}$, for which, for all $\tG_r\in\mathcal{G}$, for all $X\in V(\tG_r)$, $Size(\tG^X_r)=1$ (as an example, consider the case that $\tG_r$ is a tree). Consider $G$ in the MEC represented by $\tG$, and assume vertex $X_{\tG_r}$ is required to be set as the root in chain component $\tG_r\in\mathcal{G}$ for $G$ to be obtained. We have
\begin{align*}
P(G) &= \prod_{\tG_r\in\mathcal{G}}P(X_{\tG_r} \text{ picked})
=\prod_{\tG_r\in\mathcal{G}}\frac{\Size(\tG_r^{X_{\tG_r}})}{\Size(\tG_r)}\\
&=\prod_{\tG_r\in\mathcal{G}}\frac{1}{\Size(\tG_r)}
=\frac{1}{\prod_{\tG_r\in\mathcal{G}}\Size(\tG_r)}\\
&=\frac{1}{\Size(\tG)},
\end{align*}
where, the last equality follows from equation \eqref{eq:prod}.\\
{\bf Induction hypothesis:} 
For an essential graph $\tG$ with maximum required recursions of $l-1$, any DAG $G$ in the MEC represented by $\tG$ is generated with probability $1/Size(\tG)$.\\
{\bf Induction step:} We need to show that for an essential graph $\tG$ with maximum required recursions of $l$, any DAG $G$ in the MEC represented by $\tG$ is generated with probability $1/Size(\tG)$.
Assume vertex $X_{\tG_r}$ is required to be set as the root in chain component $\tG_r\in\mathcal{G}$, and $V_{\tG_r^{X_{\tG_r}}}$ is the set of vertices required to be set as root in the next recursions in obtained chain components in $\tG_r^{X_{\tG_r}}$ for $G$ to be obtained. We have
\begin{align*}
P(G) &= \prod_{\tG_r\in\mathcal{G}}P(X_{\tG_r} \text{ picked})P(V_{\tG_r^{X_{\tG_r}}} \text{ picked})\\
&= \prod_{\tG_r\in\mathcal{G}}\frac{Size(\tG_r^{X_{\tG_r}})}{\Size(\tG_r)}P(V_{\tG_r^{X_{\tG_r}}} \text{ picked}).
\end{align*}
By the induction hypothesis, 
\[
P(V_{\tG_r^{X_{\tG_r}}} \text{ picked})=1/\Size(\tG_r^{X_{\tG_r}}).
\]
Therefore,
\begin{align*}
P(G) &= \prod_{\tG_r\in\mathcal{G}}\frac{\Size(\tG_r^{X_{\tG_r}})}{\Size(\tG_r)}\frac{1}{\Size(\tG_r^{X_{\tG_r}})}\\
&=\frac{1}{\prod_{\tG_r\in\mathcal{G}}\Size(\tG_r)}\\
&=\frac{1}{\Size(\tG)},
\end{align*}
where, the last equality follows from equation \eqref{eq:prod}.

\section{Proof of Corollary \ref{cor:sampcomp}}
\label{app:cor:sampcomp}

For any chain component $\tG$, for calculating $\textsc{Counter}(\tG,\tG)$ we are required to calculate the size of all possible subsequent rooted classes. Therefore, we do not need to calculate the size of any rooted subclasses anymore. Hence, by Proposition \ref{prop:countcomp}, we obtain all probabilities of the from $\frac{\textsc{Counter}(\tG^X,\tG^X)}{\textsc{Counter}(\tG,\tG)}$ in $\mathcal{O}(p^{\Delta+2})$.
After selecting one of the vertices in $\tG$ as the root, say $X$, we recover all directed edges in $\tG^{X}$ in $\mathcal{O}(p^3)$ and obtain chain components of $\tG^X$. Similar to the proof of Proposition \ref{prop:countcomp}, let $t(\Delta)$ be the running time of the algorithm on a chain component in $\mathcal{G}$ with maximum degree of $\Delta$. We have
\begin{equation*}
t(\Delta)\leq pt(\Delta-1)+Cp^3,
\end{equation*}
where $C$ is a constant. It can be shown that $t(\Delta)$ is in the order of $\mathcal{O}(\Delta p^{\Delta+1})$. Since we may have at most $p$ chain components in $\mathcal{G}$, the computational complexity of uniform sampler would be in the order of $\mathcal{O}(p^{\Delta+2})$.
Therefore, the computational complexity of the approach is $\mathcal{O}(p^{\Delta+2}+p^{\Delta+2})=\mathcal{O}(p^{\Delta+2})$.

\section{Proof of Theorem \ref{thm:uconv} }

\begin{proposition}[\textbf{Chernoff Bound}]
\label{prop:Chbound}
Let $X_1,...,X_N$ be independent random variables such that for all $i$, $0 \le X_i \le 1$. Let $\mu=\mathbb{E}[\sum_{i=1}^NX_i]$. Then
\[
P(|\sum_{i=1}^NX_i - \mu |\ge\epsilon\mu) \le 2 \exp (-\frac{\epsilon^2}{2+\epsilon}\mu).
\]
\end{proposition}

\begin{proof}[Proof of Proposition \ref{prop:inMEC}]
For $i\in\{1,...,N\}$, define $X_i=\frac{D(\I,G_i)}{|\bar{A}(\tG)|}$.
We note that for the estimator in Algorithm \ref{algorithm:unifsamp}, we have $\mathbb{E}[D(\I,G_i)]=\mathcal{D}(\I)$, where $G_i$ is a random generated DAG in the sampler in Algorithm \ref{algorithm:unifsamp}. This can be proven as follows:
\begin{align*}
	\mathbb{E}[D(\I,G_i)]&=\sum_{G'_i\in\MEC(G^*)}P(G_i=G'_i)D(\I,G'_i)\\
	&=\sum_{G'_i\in\MEC(G^*)}\frac{1}{|\MEC(G^*)|}D(\I,G'_i)\\
	&=\mathcal{D}(\I).
\end{align*}
Therefore, $\mathbb{E}[X_i]=\frac{1}{|\bar{A}(\tG)|}\mathcal{D}(\mathcal{I})$. 

Using Chernoff bound we have
\begin{align*}
P(|\sum_{i=1}^NX_i-\frac{N}{|\bar{A}(\tG)|}\mathcal{D}(\I)|\ge\epsilon\frac{N}{|\bar{A}(\tG)|}\mathcal{D}(\I))
&\le 2 \exp (-\frac{N\epsilon^2}{|\bar{A}(\tG)|(2+\epsilon)}\mathcal{D}(\I))\\
&\le 2 \exp (-\frac{N\epsilon^2}{|\bar{A}(\tG)|(2+\epsilon)}).
\end{align*}
Therefore,
\[
P(|\frac{1}{N}\sum_{i=1}^ND(\I,G_i)-\mathcal{D}(\I)|\ge\epsilon\mathcal{D}(\I))\le 2 \exp (-\frac{N\epsilon^2}{|\bar{A}(\tG)|(2+\epsilon)}).
\]
Hence,
\[
P(|\hat{\mathcal{D}}(\I)-\mathcal{D}(\I)|<\epsilon\mathcal{D}(\I))
> 1-2 \exp (-\frac{N\epsilon^2}{|\bar{A}(\tG)|(2+\epsilon)}).
\]
Setting $N>\frac{|\bar{A}(\tG)|(2+\epsilon)}{\epsilon^2}\ln(\frac{2}{\delta})$, upper bounds the right hand side with $1-\delta$ and concludes the desired result.
	
\end{proof}

\section{Proof of Theorem \ref{thm:app}}

Let $\I^*=\{X_1^*,...,X_k^*\}\in\arg\max_{\I:\I\subseteq V,|\I|= k} \mathcal{D}(\I)$. We have
\begin{equation}
\begin{aligned}
\label{eq:app1}
\mathcal{D}(\mathcal{I}^*)&\overset{(a)}{\le}\mathcal{D}(\mathcal{I}^*\cup \mathcal{I}_i)=\mathcal{D}(\mathcal{I}_i)+\sum_{j=1}^k[\mathcal{D}(\mathcal{I}_i\cup\{X_1^*,...,X_j^*\})-\mathcal{D}(\mathcal{I}_i\cup\{X_1^*,...,X_{j-1}^*\})]\\
&\overset{(b)}{\le}\mathcal{D}(\mathcal{I}_i)+\sum_{j=1}^k[\mathcal{D}(\mathcal{I}_i\cup\{X_j^*\})-\mathcal{D}(\mathcal{I}_i)],
\end{aligned}
\end{equation}
where $(a)$ follows from Proposition \ref{prop:mono}, and $(b)$ follows from Theorem \ref{thm:submodular}.
Define $\hat{\mathcal{D}}_{i,X,1}$ and $\hat{\mathcal{D}}_{i,X,2}$ as the first and second calls of the estimator in $i$-th step for variable $X$, respectively. By the assumption of the theorem we have 
\begin{align*}
\mathcal{D}(\mathcal{I}_i\cup\{X_j^*\})-\epsilon\mathcal{D}(\mathcal{I}_i\cup\{X_j^*\})<\hat{\mathcal{D}}_{i,X^*_j,1}(\mathcal{I}_i\cup\{X_j^*\}),
\end{align*}
with probability larger than $1-\delta$. Therefore,
\begin{align*}
\mathcal{D}(\mathcal{I}_i\cup\{X_j^*\})<\hat{\mathcal{D}}_{i,X^*_j,1}(\mathcal{I}_i\cup\{X_j^*\})+\epsilon\mathcal{D}(\mathcal{I}^*),
\end{align*}
with probability larger than $1-\delta$.
Similarly
\begin{align*}
\hat{\mathcal{D}}_{i,X^*_j,2}(\mathcal{I}_i)<\mathcal{D}(\mathcal{I}_i)+\epsilon\mathcal{D}(\mathcal{I}_i)\hspace{1cm}&w.p.>1-\delta,\\
\Rightarrow-\mathcal{D}(\mathcal{I}_i)< -\hat{\mathcal{D}}_{i,X^*_j,2}(\mathcal{I}_i)+\epsilon\mathcal{D}(\mathcal{I}^*)\hspace{1cm}&w.p.>1-\delta,
\end{align*}
Therefore,
\begin{equation}
\label{eq:app2}
\begin{aligned}
\mathcal{D}&(\mathcal{I}_i\cup\{X_j^*\})-\mathcal{D}(\mathcal{I}_i)<\hat{\mathcal{D}}_{i,X^*_j,1}(\mathcal{I}_i\cup\{X_j^*\})\\
&-\hat{\mathcal{D}}_{i,X^*_j,2}(\mathcal{I}_i)+2\epsilon\mathcal{D}(\mathcal{I}^*)\hspace{1cm}w.p.>1-2\delta.
\end{aligned}
\end{equation}
Also, by the definition of the greedy algorithm,
\begin{equation}
\label{eq:app3}
\begin{aligned}
\hat{\mathcal{D}}_{i,X^*_j,1}&(\mathcal{I}_i\cup\{X_j^*\})-\hat{\mathcal{D}}_{i,X^*_j,2}(\mathcal{I}_i)\\
&\le\hat{\mathcal{D}}_{i,X_{i+1},1}(\mathcal{I}_i\cup\{X_{i+1}\})-\hat{\mathcal{D}}_{i,X_{i+1},2}(\mathcal{I}_i)\\
&=\hat{\mathcal{D}}_{i,X_{i+1},1}(\mathcal{I}_{i+1})-\hat{\mathcal{D}}_{i,X_{i+1},2}(\mathcal{I}_i),
\end{aligned}
\end{equation}
and similar to \eqref{eq:app2}, we have
\begin{equation}
\label{eq:app4}
\begin{aligned}
\hat{\mathcal{D}}&_{i,X_{i+1},1}(\mathcal{I}_{i+1})-\hat{\mathcal{D}}_{i,X_{i+1},2}(\mathcal{I}_i)<
\mathcal{D}(\mathcal{I}_{i+1})\\
&-\mathcal{D}(\mathcal{I}_i)+2\epsilon\mathcal{D}(\mathcal{I}^*)\hspace{1cm}w.p.>1-2\delta.
\end{aligned}
\end{equation}
Therefore, from equations \eqref{eq:app2}, \eqref{eq:app3}, and \eqref{eq:app4} we have
\begin{equation}
\label{eq:app5}
\mathcal{D}(\mathcal{I}_i\cup\{X_j^*\})-\mathcal{D}(\mathcal{I}_i)<\mathcal{D}(\mathcal{I}_{i+1})-\mathcal{D}(\mathcal{I}_i)+4\epsilon\mathcal{D}(\mathcal{I}^*),
\end{equation}
with probability larger than $1-4\delta$. Plugging \eqref{eq:app5} back in \eqref{eq:app1}, we get
\begin{align*}
\mathcal{D}(\mathcal{I}^*)&<\mathcal{D}(\mathcal{I}_i)+\sum_{j=1}^k[\mathcal{D}(\mathcal{I}_{i+1})-\mathcal{D}(\mathcal{I}_i)+4\epsilon\mathcal{D}(\mathcal{I}^*)]\\
&=\mathcal{D}(\mathcal{I}_i)+k[\mathcal{D}(\mathcal{I}_{i+1})-\mathcal{D}(\mathcal{I}_i)]+4k\epsilon\mathcal{D}(\mathcal{I}^*),
\end{align*}
with probability larger than $1-4k\delta$. Therefore,
\begin{align*}
&\mathcal{D}(\mathcal{I}^*)-\mathcal{D}(\mathcal{I}_i)\\
&< k[\mathcal{D}(\mathcal{I}^*)-\mathcal{D}(\mathcal{I}_i)]-k[\mathcal{D}(\mathcal{I}^*)-\mathcal{D}(\mathcal{I}_{i+1})]+4k\epsilon\mathcal{D}(\mathcal{I}^*),
\end{align*}
with probability larger than $1-4k\delta$. Defining $a_i\coloneqq\mathcal{D}(\mathcal{I}^*)-\mathcal{D}(\mathcal{I}_i)$, and noting that $a_0=\mathcal{D}(\mathcal{I}^*)$, by induction we have 
\begin{align*}
a_k&=\mathcal{D}(\mathcal{I}^*)-\mathcal{D}(\mathcal{I}_k)\\
&<(1-\frac{1}{k})^k\mathcal{D}(\mathcal{I}^*)+4\epsilon\mathcal{D}(\mathcal{I}^*)\sum_{j=0}^{k-1}(1-\frac{1}{k})^j\\
&<[\frac{1}{e}+4\epsilon k]\mathcal{D}(\mathcal{I}^*)\hspace{1cm}w.p.>1-4k^2\delta.
\end{align*}
It concludes that
\[
\mathcal{D}(\mathcal{I}_k)>(1-\frac{1}{e}-4\epsilon k)\mathcal{D}(\mathcal{I}^*)\hspace{1cm}w.p.>1-4k^2\delta.
\]
Therefore, for $\epsilon=\frac{\epsilon'}{4k}$ and $\delta=\frac{\delta'}{4k^2}$,  Algorithms \ref{algorithm:GG} is a $(1-\frac{1}{e}-\epsilon')$-approximation algorithm with probability larger than $1-\delta'$.

%
%
%

\section{Proof of Proposition \ref{prop:inMEC}}

We require the following lemma for the proof.
\begin{lemma}
\label{claim:tri}
If a directed chordal graph has a directed cycle then it has a directed cycle of size 3.
\end{lemma}
\begin{proof}
If the directed cycle is of size 3 itself, the claim is trivial. Suppose the directed cycle $C_n$ is of size $n>3$. Relabel the vertices of $C_n$ to have $C_n=(X_1,...,X_n,X_1)$. Since the graph is chordal, $C_n$ has a chord and hence we have a triangle induced on vertices $\{X_i,X_{i+1},X_{i+2}\}$ for some $i$. If the direction of $X_i-X_{i+2}$ is $X_{i+2}\rightarrow X_i$, we have the directed cycle $(X_i,X_{i+1},X_{i+2},X_i)$ which is of size 3. Otherwise, we have the directed cycle $C_{n-1}=(X_1,...,X_i,X_{i+2},..,X_n,X_1)$ on $n-1$ vertices. Relabeling the vertices from $1$ to $n-1$ and repeating the above reasoning concludes the lemma.

\end{proof}
\begin{proof}[Proof of Proposition \ref{prop:inMEC}]
All the components in the undirected subgraph of $\tG$ are chordal \citep{hauser2012characterization}. Therefore, by Lemma \ref{claim:tri}, to insure that a generated directed graph is a DAG, it suffices to make sure that it does not have any directed cycles of length 3, which is one of the checks that we do in the proposed procedure. For checking if the generated DAG is in the same Markov equivalence class as $G^*$, since they have the same skeleton, it suffices to check if they have the same set of v-structures \citep{judea1991equivalence}, which is the other check that we do in the sampler in Algorithm \ref{algorithm:fastsamp}.

\end{proof}

\end{appendices}

~\newpage

\vskip 0.2in
\bibliographystyle{plainnat}
\bibliography{Refs}

\begin{thebibliography}{49}
\providecommand{\natexlab}[1]{#1}
\providecommand{\url}[1]{\texttt{#1}}
\expandafter\ifx\csname urlstyle\endcsname\relax
  \providecommand{\doi}[1]{doi: #1}\else
  \providecommand{\doi}{doi: \begingroup \urlstyle{rm}\Url}\fi

\bibitem[Andersson et~al.(1997)Andersson, Madigan, and
  Perlman]{andersson1997characterization}
Steen~A Andersson, David Madigan, and Michael~D Perlman.
\newblock A characterization of {M}arkov equivalence classes for acyclic
  digraphs.
\newblock \emph{The Annals of Statistics}, 25\penalty0 (2):\penalty0 505--541,
  1997.

\bibitem[Barab{\'a}si(2016)]{barabasi2016network}
Albert-L{\'a}szl{\'o} Barab{\'a}si.
\newblock \emph{Network science}.
\newblock Cambridge University Press, 2016.

\bibitem[Barab{\'a}si and Albert(1999)]{barabasi1999emergence}
Albert-L{\'a}szl{\'o} Barab{\'a}si and R{\'e}ka Albert.
\newblock Emergence of scaling in random networks.
\newblock \emph{science}, 286\penalty0 (5439):\penalty0 509--512, 1999.

\bibitem[Bernstein and Tetali(2017)]{bernstein2017sampling}
Megan Bernstein and Prasad Tetali.
\newblock On sampling graphical {M}arkov models.
\newblock \emph{arXiv preprint arXiv:1705.09717}, 2017.

\bibitem[Chen et~al.(2009)Chen, Wang, and Yang]{chen2009efficient}
Wei Chen, Yajun Wang, and Siyu Yang.
\newblock Efficient influence maximization in social networks.
\newblock In \emph{Proceedings of the 15th ACM SIGKDD international conference
  on Knowledge discovery and data mining}, pages 199--208. ACM, 2009.

\bibitem[Chickering(2002)]{chickering2002optimal}
David~Maxwell Chickering.
\newblock Optimal structure identification with greedy search.
\newblock \emph{Journal of machine learning research}, 3\penalty0
  (Nov):\penalty0 507--554, 2002.

\bibitem[Dudzi{\'n}ski and Walukiewicz(1987)]{dudzinski1987exact}
Krzysztof Dudzi{\'n}ski and Stanis{\l}aw Walukiewicz.
\newblock Exact methods for the knapsack problem and its generalizations.
\newblock \emph{European Journal of Operational Research}, 28\penalty0
  (1):\penalty0 3--21, 1987.

\bibitem[Eberhardt(2007)]{eberhardt2007causation}
Frederick Eberhardt.
\newblock Causation and intervention.
\newblock \emph{Unpublished doctoral dissertation, Carnegie Mellon University},
  2007.

\bibitem[Eberhardt(2012)]{eberhardt2012almost}
Frederick Eberhardt.
\newblock Almost optimal intervention sets for causal discovery.
\newblock \emph{arXiv preprint arXiv:1206.3250}, 2012.

\bibitem[Eberhardt and Scheines(2007)]{eberhardt2007interventions}
Frederick Eberhardt and Richard Scheines.
\newblock Interventions and causal inference.
\newblock \emph{Philosophy of Science}, 74\penalty0 (5):\penalty0 981--995,
  2007.

\bibitem[Eberhardt et~al.(2005)Eberhardt, Glymour, and
  Scheines]{eberhardt2005number}
Frederick Eberhardt, Clark Glymour, and Richard Scheines.
\newblock On the number of experiments sufficient and in the worst case
  necessary to identify all causal relations among n variables.
\newblock In \emph{Proceedings of the 21st Conference on Uncertainty and
  Artificial Intelligence (UAI-05)}, pages 178--184, 2005.

\bibitem[Eberhardt et~al.(2006)Eberhardt, Glymour, and
  Scheines]{eberhardt2006n}
Frederick Eberhardt, Clark Glymour, and Richard Scheines.
\newblock N-1 experiments suffice to determine the causal relations among n
  variables.
\newblock In \emph{Innovations in machine learning}, pages 97--112. Springer,
  2006.

\bibitem[Ghassami et~al.(2018{\natexlab{a}})Ghassami, Salehkaleybar, Kiyavash,
  and Bareinboim]{ghassami2018budgeted}
AmirEmad Ghassami, Saber Salehkaleybar, Negar Kiyavash, and Elias Bareinboim.
\newblock Budgeted experiment design for causal structure learning.
\newblock In \emph{International Conference on Machine Learning}, pages
  1719--1728, 2018{\natexlab{a}}.

\bibitem[Ghassami et~al.(2018{\natexlab{b}})Ghassami, Salehkaleybar, Kiyavash,
  and Zhang]{ghassami2018counting}
AmirEmad Ghassami, Saber Salehkaleybar, Negar Kiyavash, and Kun Zhang.
\newblock Counting and sampling from {M}arkov equivalent {DAG}s using clique
  trees.
\newblock \emph{arXiv preprint arXiv:1802.01239}, 2018{\natexlab{b}}.

\bibitem[Gillispie and Perlman(2002)]{gillispie2002size}
Steven~B Gillispie and Michael~D Perlman.
\newblock The size distribution for {M}arkov equivalence classes of acyclic
  digraph models.
\newblock \emph{Artificial Intelligence}, 141\penalty0 (1-2):\penalty0
  137--155, 2002.

\bibitem[Hauser and B{\"u}hlmann(2012)]{hauser2012characterization}
Alain Hauser and Peter B{\"u}hlmann.
\newblock Characterization and greedy learning of interventional {M}arkov
  equivalence classes of directed acyclic graphs.
\newblock \emph{Journal of Machine Learning Research}, 13\penalty0
  (Aug):\penalty0 2409--2464, 2012.

\bibitem[Hauser and B{\"u}hlmann(2014)]{hauser2014two}
Alain Hauser and Peter B{\"u}hlmann.
\newblock Two optimal strategies for active learning of causal models from
  interventional data.
\newblock \emph{International Journal of Approximate Reasoning}, 55\penalty0
  (4):\penalty0 926--939, 2014.

\bibitem[He and Geng(2008)]{he2008active}
Yang-Bo He and Zhi Geng.
\newblock Active learning of causal networks with intervention experiments and
  optimal designs.
\newblock \emph{Journal of Machine Learning Research}, 9\penalty0
  (Nov):\penalty0 2523--2547, 2008.

\bibitem[He et~al.(2015)He, Jia, and Yu]{he2015counting}
Yangbo He, Jinzhu Jia, and Bin Yu.
\newblock Counting and exploring sizes of {M}arkov equivalence classes of
  directed acyclic graphs.
\newblock \emph{Journal of Machine Learning Research}, 16\penalty0
  (1):\penalty0 2589--2609, 2015.

\bibitem[Heckerman et~al.(1995)Heckerman, Geiger, and
  Chickering]{heckerman1995learning}
David Heckerman, Dan Geiger, and David~M Chickering.
\newblock Learning bayesian networks: The combination of knowledge and
  statistical data.
\newblock \emph{Machine learning}, 20\penalty0 (3):\penalty0 197--243, 1995.

\bibitem[Hoyer et~al.(2009)Hoyer, Janzing, Mooij, Peters, and
  Sch{\"o}lkopf]{hoyer2009nonlinear}
Patrik~O Hoyer, Dominik Janzing, Joris~M Mooij, Jonas Peters, and Bernhard
  Sch{\"o}lkopf.
\newblock Nonlinear causal discovery with additive noise models.
\newblock In \emph{Advances in neural information processing systems}, pages
  689--696, 2009.

\bibitem[Hyttinen et~al.(2013)Hyttinen, Eberhardt, and
  Hoyer]{hyttinen2013experiment}
Antti Hyttinen, Frederick Eberhardt, and Patrik~O Hoyer.
\newblock Experiment selection for causal discovery.
\newblock \emph{Journal of Machine Learning Research}, 14\penalty0
  (1):\penalty0 3041--3071, 2013.

\bibitem[Kempe et~al.(2003)Kempe, Kleinberg, and Tardos]{kempe2003maximizing}
David Kempe, Jon Kleinberg, and {\'E}va Tardos.
\newblock Maximizing the spread of influence through a social network.
\newblock In \emph{Proceedings of the ninth ACM SIGKDD international conference
  on Knowledge discovery and data mining}, pages 137--146. ACM, 2003.

\bibitem[Kocaoglu et~al.(2017{\natexlab{a}})Kocaoglu, Dimakis, and
  Vishwanath]{kocaoglu2017cost}
Murat Kocaoglu, Alex Dimakis, and Sriram Vishwanath.
\newblock Cost-optimal learning of causal graphs.
\newblock In \emph{Proceedings of the 34th International Conference on Machine
  Learning-Volume 70}, pages 1875--1884. JMLR. org, 2017{\natexlab{a}}.

\bibitem[Kocaoglu et~al.(2017{\natexlab{b}})Kocaoglu, Shanmugam, and
  Bareinboim]{kocaoglu2017experimental}
Murat Kocaoglu, Karthikeyan Shanmugam, and Elias Bareinboim.
\newblock Experimental design for learning causal graphs with latent variables.
\newblock In \emph{Advances in Neural Information Processing Systems}, pages
  7021--7031, 2017{\natexlab{b}}.

\bibitem[Koller and Friedman(2009)]{koller2009probabilistic}
Daphne Koller and Nir Friedman.
\newblock \emph{Probabilistic graphical models: principles and techniques}.
\newblock MIT press, 2009.

\bibitem[Korb et~al.(2004)Korb, Hope, Nicholson, and Axnick]{korb2004varieties}
Kevin~B Korb, Lucas~R Hope, Ann~E Nicholson, and Karl Axnick.
\newblock Varieties of causal intervention.
\newblock In \emph{Pacific Rim International Conference on Artificial
  Intelligence}, pages 322--331. Springer, 2004.

\bibitem[Lauritzen(1996)]{lauritzen1996graphical}
Steffen~L Lauritzen.
\newblock \emph{Graphical models}, volume~17.
\newblock Clarendon Press, 1996.

\bibitem[Leskovec et~al.(2007)Leskovec, Krause, Guestrin, Faloutsos,
  VanBriesen, and Glance]{leskovec2007cost}
Jure Leskovec, Andreas Krause, Carlos Guestrin, Christos Faloutsos, Jeanne
  VanBriesen, and Natalie Glance.
\newblock Cost-effective outbreak detection in networks.
\newblock In \emph{Proceedings of the 13th ACM SIGKDD international conference
  on Knowledge discovery and data mining}, pages 420--429. ACM, 2007.

\bibitem[Lindgren et~al.(2018)Lindgren, Kocaoglu, Dimakis, and
  Vishwanath]{lindgren2018experimental}
Erik Lindgren, Murat Kocaoglu, Alexandros~G Dimakis, and Sriram Vishwanath.
\newblock Experimental design for cost-aware learning of causal graphs.
\newblock In \emph{Advances in Neural Information Processing Systems}, pages
  5279--5289, 2018.

\bibitem[Marbach et~al.(2009)Marbach, Schaffter, Mattiussi, and
  Floreano]{marbach2009generating}
Daniel Marbach, Thomas Schaffter, Claudio Mattiussi, and Dario Floreano.
\newblock Generating realistic in silico gene networks for performance
  assessment of reverse engineering methods.
\newblock \emph{Journal of computational biology}, 16\penalty0 (2):\penalty0
  229--239, 2009.

\bibitem[Masegosa and Moral(2013)]{masegosa2013interactive}
Andr{\'e}s~R Masegosa and Seraf{\'\i}n Moral.
\newblock An interactive approach for bayesian network learning using
  domain/expert knowledge.
\newblock \emph{International Journal of Approximate Reasoning}, 54\penalty0
  (8):\penalty0 1168--1181, 2013.

\bibitem[Meek(1995)]{meek1995causal}
Christopher Meek.
\newblock Causal inference and causal explanation with background knowledge.
\newblock In \emph{UAI 1995}, pages 403--410, 1995.

\bibitem[Meek(2013)]{meek2013causal}
Christopher Meek.
\newblock Causal inference and causal explanation with background knowledge.
\newblock \emph{arXiv preprint arXiv:1302.4972}, 2013.

\bibitem[Minoux(1978)]{minoux1978accelerated}
Michel Minoux.
\newblock Accelerated greedy algorithms for maximizing submodular set
  functions.
\newblock \emph{Optimization Techniques}, pages 234--243, 1978.

\bibitem[Mirzasoleiman et~al.(2015)Mirzasoleiman, Badanidiyuru, Karbasi,
  Vondr{\'a}k, and Krause]{mirzasoleiman2015lazier}
Baharan Mirzasoleiman, Ashwinkumar Badanidiyuru, Amin Karbasi, Jan Vondr{\'a}k,
  and Andreas Krause.
\newblock Lazier than lazy greedy.
\newblock In \emph{AAAI}, pages 1812--1818, 2015.

\bibitem[Nemhauser et~al.(1978)Nemhauser, Wolsey, and
  Fisher]{nemhauser1978analysis}
George~L Nemhauser, Laurence~A Wolsey, and Marshall~L Fisher.
\newblock An analysis of approximations for maximizing submodular set
  functions?i.
\newblock \emph{Mathematical Programming}, 14\penalty0 (1):\penalty0 265--294,
  1978.

\bibitem[Pearl(2009)]{pearl2009causality}
Judea Pearl.
\newblock \emph{Causality}.
\newblock Cambridge university press, 2009.

\bibitem[Pearl and Verma(1995)]{pearl1995theory}
Judea Pearl and Thomas~S Verma.
\newblock A theory of inferred causation.
\newblock In \emph{Studies in Logic and the Foundations of Mathematics}, volume
  134, pages 789--811. Elsevier, 1995.

\bibitem[Pearl(1991)]{judea1991equivalence}
TS~Verma~Judea Pearl.
\newblock Equivalence and synthesis of causal models.
\newblock In \emph{Proceedings of Sixth Conference on Uncertainty in Artificial
  Intelligence}, pages 220--227, 1991.

\bibitem[Shanmugam et~al.(2015)Shanmugam, Kocaoglu, Dimakis, and
  Vishwanath]{shanmugam2015learning}
Karthikeyan Shanmugam, Murat Kocaoglu, Alexandros~G Dimakis, and Sriram
  Vishwanath.
\newblock Learning causal graphs with small interventions.
\newblock In \emph{Advances in Neural Information Processing Systems}, pages
  3195--3203, 2015.

\bibitem[Shimizu et~al.(2006)Shimizu, Hoyer, Hyv{\"a}rinen, and
  Kerminen]{shimizu2006linear}
Shohei Shimizu, Patrik~O Hoyer, Aapo Hyv{\"a}rinen, and Antti Kerminen.
\newblock A linear non-gaussian acyclic model for causal discovery.
\newblock \emph{Journal of Machine Learning Research}, 7\penalty0
  (Oct):\penalty0 2003--2030, 2006.

\bibitem[Spirtes et~al.(2000)Spirtes, Glymour, and
  Scheines]{spirtes2000causation}
Peter Spirtes, Clark~N Glymour, and Richard Scheines.
\newblock \emph{Causation, prediction, and search}.
\newblock MIT press, 2000.

\bibitem[Tong and Koller(2001)]{tong2001active}
Simon Tong and Daphne Koller.
\newblock Active learning for structure in bayesian networks.
\newblock In \emph{International joint conference on artificial intelligence},
  volume~17, pages 863--869. Citeseer, 2001.

\bibitem[Tsamardinos et~al.(2006)Tsamardinos, Brown, and
  Aliferis]{tsamardinos2006max}
Ioannis Tsamardinos, Laura~E Brown, and Constantin~F Aliferis.
\newblock The max-min hill-climbing bayesian network structure learning
  algorithm.
\newblock \emph{Machine learning}, 65\penalty0 (1):\penalty0 31--78, 2006.

\bibitem[Verma and Pearl(1990)]{verma1990equivalence}
T~Verma and Judea Pearl.
\newblock Equivalence and synthesis of causal models.
\newblock In \emph{UAI 1990}, pages 220--227, 1990.

\bibitem[Verma and Pearl(1992)]{verma1992algorithm}
Thomas Verma and Judea Pearl.
\newblock An algorithm for deciding if a set of observed independencies has a
  causal explanation.
\newblock In \emph{Proceedings of the Eighth international conference on
  uncertainty in artificial intelligence}, pages 323--330. Morgan Kaufmann
  Publishers Inc., 1992.

\bibitem[Yang et~al.(2018)Yang, Katcoff, and Uhler]{yang2018characterizing}
Karren~D Yang, Abigail Katcoff, and Caroline Uhler.
\newblock Characterizing and learning equivalence classes of causal dags under
  interventions.
\newblock \emph{arXiv preprint arXiv:1802.06310}, 2018.

\bibitem[Zhang and Hyv{\"a}rinen(2009)]{zhang2008distinguishing}
Kun Zhang and Aapo Hyv{\"a}rinen.
\newblock On the identifiability of the post-nonlinear causal model.
\newblock In \emph{Proc. 25th Conference on Uncertainty in Artificial
  Intelligence (UAI 2009), Montreal, Canada}, 2009.

\end{thebibliography}

\end{document}